\documentclass{article} 
\usepackage{iclr2024_conference,times}

\usepackage{amsmath,amsfonts,bm}

\def\eqref#1{equation~\ref{#1}}

\def\floor#1{\lfloor #1 \rfloor}
\def\1{\bm{1}}

\def\eps{{\epsilon}}

\def\mF{{\bm{F}}}

\DeclareMathAlphabet{\mathsfit}{\encodingdefault}{\sfdefault}{m}{sl}
\SetMathAlphabet{\mathsfit}{bold}{\encodingdefault}{\sfdefault}{bx}{n}

\newcommand{\E}{\mathbb{E}}

\newcommand{\R}{\mathbb{R}}

\DeclareMathOperator*{\argmax}{arg\,max}

\usepackage{bbm}

\usepackage{todonotes}
\setuptodonotes{color=blue!30}

\usepackage{amssymb}

\usepackage{graphicx}
\usepackage{caption}
\usepackage{subcaption}
\usepackage{adjustbox}
\usepackage{multirow}

\usepackage{makecell}
\usepackage{tablefootnote} 
\newcolumntype{x}[1]{>{\centering\let\newline\\\arraybackslash\hspace{0pt}}p{#1}} 

\usepackage{xspace}
\makeatletter
\DeclareRobustCommand\onedot{\futurelet\@let@token\@onedot}
\def\@onedot{\ifx\@let@token.\else.\null\fi\xspace}
\def\ie{\emph{i.e}\onedot}  
\def\eg{\emph{e.g}\onedot}  
\makeatother

\usepackage{tikz}

\newcommand{\vertiii}[1]{{\left\vert\kern-0.25ex\left\vert\kern-0.25ex\left\vert #1
    \right\vert\kern-0.25ex\right\vert\kern-0.25ex\right\vert}} 
\def\param{{\bm{\theta}}}
\def\path{{p}}
\newcommand\inner[2]{\left\langle #1, #2 \right\rangle} 

\DeclareMathOperator*{\Rad}{\textrm{Rad}}
\def\Param{\bm{\Theta}}
\newcommand{\paramto}[1]{{\param^{\to #1}}}
\newcommand{\paramfrom}[1]{{\param^{#1 \to }}}
\newcommand{\paramfromto}[2]{{\param^{#1 \to #2}}}

\newcommand{\paths}{\mathcal P}
\newcommand{\pathsto}[1]{\paths^{\to #1}}
\newcommand{\Phito}[1]{\Phi^{\to #1}}

\def\activation{{\bm{A}}}
\newcommand{\ActTo}[1]{\activation^{\to #1}}

\DeclareMathOperator*{\id}{id}
\def\eps{{\boldsymbol{\varepsilon}}}%
\def\dout{{d_{\textrm{out}}}}%
\def\din{{d_{\textrm{in}}}}%
\newcommand{\NeuronSet}{N}
\newcommand{\NeuronsSubScript}[1]{\NeuronSet_{#1}}

\def\NeuronsIn{\NeuronsSubScript{\textrm{in}}}
\def\NeuronsOut{\NeuronsSubScript{\textrm{out}}}
\def\biasNeuron{{v_{\texttt{bias}}}}
\def\BiasAndNeuronsIn{\NeuronsIn\cup\{\biasNeuron\}}

\DeclareMathOperator{\ant}{ant}
\DeclareMathOperator{\suc}{suc}
\DeclareMathOperator{\realization}{R}
\DeclareMathOperator{\supp}{support}
\DeclareMathOperator{\relu}{ReLU}
\newcommand{\kpool}{k\textrm{-}\mathtt{pool}}
\newcommand{\pool}{*\textrm{-}\mathtt{pool}}
\def\gPeeling{g}
\newcommand{\pathend}[1]{{#1}_{\mathtt{end}}}
\newcommand{\length}[1]{\mathtt{length}(#1)}

\newcommand{\pathProduct}[3]{\Pi(#3,#1,#2)}

\def\bias{b}

\def\biascompleted{\gamma}
\def\P{{\mathbb{P}}} 

\def\pre{{\textrm{pre}}}
\def\phi{{\varphi}}

\def\E{{\mathbb{E}}} 
\def\R{{\mathbb R}} 
\def\Rp{{\mathbb R}_{\geq 0}} 
\def\N{{\mathbb N}} 
\def\Ns{{\mathbb N}_{>0}} 

\renewcommand{\leq}{\leqslant}
\renewcommand{\geq}{\geqslant}

\DeclareMathOperator*{\sgn}{sgn}%
\def\hatparam{{\hat{\param}}} 

\newcommand{\rv}[1]{{\mathbf #1}}%

\usepackage{amsthm}

\newtheorem{theorem}{Theorem}[section]%
\newtheorem{lemma}{Lemma}[section]%
\newtheorem{definition}{Definition}[section]%
\newtheorem{remark}{Remark}[section]%
\usepackage[colorlinks,urlcolor=blue,citecolor=blue,linkcolor=blue]{hyperref}
\usepackage{url}
\usepackage{cleveref}

\usepackage{xcolor}

\newcommand\textred[1]{{\color{red}#1}}
\newcommand\mathred[1]{{\color{red}{#1}}}
\newcommand\red[1]{\ifmmode\mathred{#1}\else\textred{#1}\fi}

\newcommand\textblue[1]{{\color{blue}\bfseries#1}}
\newcommand\mathblue[1]{{\color{blue}\boldsymbol{#1}}}
\newcommand\textorange[1]{{\color{orange}\bfseries#1}}
\newcommand\mathorange[1]{{\color{orange}\boldsymbol{#1}}}
\newcommand\blue[1]
{\ifmmode\mathblue{#1}\else\textblue{#1}\fi}
\newcommand\mblue[1]{\marginpar{\blue}}
\newcommand\orange[1]
{\ifmmode\mathorange{#1}\else\textorange{#1}\fi}
\usepackage{listings}
\definecolor{codegreen}{rgb}{0,0.6,0}
\definecolor{codegray}{rgb}{0.5,0.5,0.5}
\definecolor{codepurple}{rgb}{0.58,0,0.82}
\definecolor{backcolour}{rgb}{0.95,0.95,0.92}

\lstdefinestyle{mystyle}{
    backgroundcolor=\color{backcolour},   
    commentstyle=\color{codegreen},
    keywordstyle=\color{magenta},
    numberstyle=\tiny\color{codegray},
    stringstyle=\color{codepurple},
    basicstyle=\ttfamily\footnotesize,
    breakatwhitespace=false,         
    breaklines=true,                 
    captionpos=b,                    
    keepspaces=true,                 
    numbers=left,                    
    numbersep=5pt,                  
    showspaces=false,                
    showstringspaces=false,
    showtabs=false,                  
    tabsize=2
}

\usepackage{algorithm}
\usepackage[noend]{algpseudocode}

\title{A path-norm toolkit for modern networks: consequences, promises and challenges}

\author{Antoine Gonon, Nicolas Brisebarre, Elisa Riccietti \& Rémi Gribonval\\
Univ Lyon, EnsL, UCBL, CNRS, Inria,  LIP, F-69342, LYON Cedex 07, France
}

\iclrfinalcopy 
\begin{document}

\maketitle

\begin{abstract}
This work introduces the first toolkit around path-norms that fully encompasses general DAG ReLU networks with biases, skip connections and any operation based on the extraction of order statistics: max pooling, GroupSort etc.
This toolkit notably allows us to establish generalization bounds for modern neural networks that are not only the most widely applicable path-norm based ones, but also recover or beat the sharpest known bounds of this type. 
These extended path-norms further enjoy the usual benefits of path-norms: ease of computation,  invariance under the symmetries of the network, and improved sharpness on layered fully-connected networks compared to the product of operator norms, another complexity measure most commonly used.

The versatility of the toolkit and its ease of implementation allow us to challenge the concrete promises of path-norm-based generalization bounds, by numerically evaluating the sharpest known bounds for ResNets on ImageNet.
\end{abstract}
\section{Introduction}
Developing a thorough understanding of theoretical properties of neural networks is key to achieve central objectives such as efficient and trustworthy training, robustness to adversarial attacks (e.g. via Lipschitz bounds), or statistical soundness guarantees (via so-called generalization bounds).

The so-called path-norms and path-lifting are promising concepts to theoretically analyze neural networks: the $L^1$ path-norm has been used to derive generalization guarantees \citep{Neyshabur15NormBasedControls, Barron19V}, and the path-lifting has led for example to identifiability guarantees \citep{BonaPellissier22NipsIdentifiability, Stock22Embedding} and characterizations of properties of the dynamics of training algorithms \citep{Marcotte23GradientFlows}.

Yet, current definitions of path-norms and of path-lifting are severely limited: they only cover simple models unable to combine in a single framework pooling layers, skip connections, biases, or even multi-dimensional output  \citep{Neyshabur15NormBasedControls, Kawaguchi17GeneralizationInDL, BonaPellissier22NipsIdentifiability, Stock22Embedding}. Thus, the promises of existing theoretical guarantees based on these tools are currently out of reach: they cannot even be tested on standard modern networks. 

Due to the lack of versatility of these tools, known results have only been tested on toy examples. This prevents us from both understanding the reach of these tools and from diagnosing their strengths and weaknesses, which is necessary to either improve them in order to make them actually operational, if possible, or to identify without concession the gap between theory and practice, in particular for generalization bounds. 

\emph{This work adresses the challenge of making these tools fully compatible with modern networks, and to concretely assess them on standard real-world examples}. First, \textbf{it formalizes a definition of path-lifting (and path-norms) adapted to very generic ReLU networks}, covering any DAG architecture (in particular with skip connections), including in the presence of max/average-pooling (and even more generally $k$-max-pooling, which extracts the $k$-th largest coordinate, recovering max-pooling for $k=1$) and/or biases. This covers a wide variety of modern networks (notably ResNets, VGGs, U-nets, ReLU MobileNets, Inception nets, Alexnet)\footnote{The conclusion discusses networks not covered by the framework and adaptations needed to cover them.}, and recovers previously known  definitions of these tools in simpler settings such as layered fully-connected networks.

The immediate interests of these tools are: 1) \emph{path-norms are easy to compute}\footnote{Code for reproducibility is in \citet{Gonon23PathNormCode}.} on modern networks via a single forward-pass; 2) \emph{path-norms are invariant under neuron permutations and parameter rescalings that leave the network invariant}; and 3) \emph{the path-norms yield Lipschitz bounds}. These properties were known (but scattered in the literature) in the restricted case of layered fully-connected ReLU networks primarily without biases (and without average/$k$-max-pooling nor skip connections) \citep{Neyshabur15NormBasedControls, Neyshabur17ImplictRegularizationInDL, PhDFurusho20Lipschitz, Jiang20FantasticGeneralizationMeasures, GKDziugaite20SearchRobustMeasuresGeneralization, BonaPellissier22NipsIdentifiability, Stock22Embedding}. They are generalized here for generic DAG ReLU networks with most of the standard ingredients of modern networks (pooling, skip connections...), with the notable exception of the attention mechanism.

Moreover, {\em path-norms tightly lower bound products of operator norms}, another complexity measure that does not enjoy the same invariances as path-norms, despite being widely used for Lipschitz bounds (\eg, to control adversarial robusteness) \citep{Neyshabur18PACBayesSpectrallyNormalizedMarginBounds, Gonon22Quantization} or generalization bounds \citep{Neyshabur15NormBasedControls, Bartlett17SpectralGeneralizationBound, Golowich18GeneralizationBound}. This bound, which was only known for {\em scalar-valued} layered fully-connected ReLU networks without biases \citep{Neyshabur15NormBasedControls}, is again generalized here to generic {\em vector-valued} DAG ReLU networks. This requires introducing so-called {\em mixed} path-norms and extending products of operator norms.

Second, \textbf{this work also establishes a new generalization bound for modern ReLU networks based on their corresponding $L^1$ path-norm}. This bound covers arbitrary output dimension (while previous work focused on scalar dimension, see \Cref{tab:GeneralizationBounds}), generic DAG ReLU network architectures with average/$k$-max-pooling, skip connections and biases. The achieved generalization bound recovers or beats the sharpest known ones of this type, that were so far only available in simpler restricted settings, see \Cref{tab:GeneralizationBounds} for an overview. Among the technical ingredients used in the proof of this generalization bound, \emph{the new contraction lemmas and the new peeling argument are among the main theoretical contributions of this work}. \emph{The first new contraction lemma extends the classical ones with scalar} $t_i\in\R$, and contractions $f_i$ of the form $\E_{\eps} \gPeeling\left(\sup_{t\in T} \sum_{i\in I} \eps_i f_i(t_i)\right) \leq \E_\eps \gPeeling\left(\sup_{t\in T} \sum_{i\in I} \eps_i t_i\right)$ \cite[Theorem 4.12]{LedouxTalagrand91ProbaBanachSpaces}, with convex non-decreasing $\gPeeling$, \emph{to situations where there are multiples independent copies} indexed by $z\in Z$ of the latter:
$\E_{\eps} \max_{z\in Z} \gPeeling\left(\sup_{t\in T^z} \sum_{i\in I} \eps_{i,z} f_{i,z}(t_i)\right) \leq \E_{\eps} \max_{z\in Z}\gPeeling\left(\sup_{t\in T^z} \sum_{i\in I} \eps_{i,z} t_{i}\right)$. \emph{The second new contraction lemma deals with vector-valued} $t_i\in\R^W$, and functions $f_i$ that compute \emph{the $k$-th largest input's coordinate, to cope with $k$-max-pooling neurons}, and it also handles \emph{multiple independent copies} indexed by $z\in Z$. The most closely related lemma we could find is the vector-valued one in \citet{Maurer16VectorContraction} established with a different technique, and that holds only for $\gPeeling=\id$ with a single copy ($|Z|=1$). {\em The peeling argument} reduces the Rademacher complexity of the whole model to that of the inputs by peeling off the neurons of the model one by one. This is inspired by the peeling argument of \citet{Golowich18GeneralizationBound}, which is however specific to layered fully-connected ReLU networks with layer-wise constraints on the weights. Substantial additional ingredients are developed to handle \emph{arbitrary DAG ReLU networks} (as there is no longer such a thing as a {\em layer} to peel), \emph{with not only ReLU but also $k$-max-pooling and identity neurons} (where \citet{Golowich18GeneralizationBound} has only ReLU neurons), and leveraging \emph{only a global constraint through the path-norm} (instead of layerwise constraints on operator norms). The analysis notably makes use of the rescaling invariance of the proposed generalized path-lifting.

\begin{table}[htbp]
    \centering
    \caption{Generalization bounds (up to universal multiplicative constants) for a ReLU network estimator with parameters restricted to belong to some subset $\Param \subseteq \R^G$ (\Cref{def:NN}), learned from $n$ iid training points when 1) the loss $\hat{y}\in(\R^\dout,\|\cdot\|_2)\mapsto \ell(\hat{y}, y)\in\R$ is $L$-Lipschitz for every $y$, and 2) inputs are bounded in $L^\infty$-norm by $B\geq 1$. Here, $\din/\dout$ are the input/output dimensions, $K=\max_{v\in\NeuronSet_{\pool}} |\ant(v)|$ is the maximum kernel size (\Cref{def:NN}) of the $*$-max-pooling neurons, $M_d$ is the matrix of layer $d$ for a {\em layered fully-connected} network (LFCN) without bias $R_\param(x) = M_D\relu(M_{D-1}\dots \relu(M_1x))$, $D$ is the depth. Note that having $r$ in the bound is more desirable than having $R$ since $r\leq R$ and $R$ can be arbitrarily large even when $r=0$ (\Cref{thm:Phi=MinRescalings} and \Cref{fig:PhiNormVsProdNorm} in appendix). This is because $R$ decouples the layers without taking into account rescaling invariances.}    \label{tab:GeneralizationBounds}
    \begin{adjustbox}{max width=1.1\textwidth,center}
    \begin{tabular}{x{0.28\textwidth}|x{0.4\textwidth}|x{0.18\textwidth}|>{\raggedright\arraybackslash}x{0.25\textwidth}}
         & {\footnotesize Architecture} & {\footnotesize Parameter set $\Param$} & {\footnotesize Generalization bound} \\
        \hline
        {\footnotesize \cite[Eq. (5)]{Kakade08GeneralizationLinearModels}   \cite[Sec. 4.5.3]{BachBook}} & {\footnotesize LFCN with depth $D=1$, no bias, $\dout=1$ \hspace{1cm} (linear regression)} & {\footnotesize $\|\param\|_1=\|\Phi(\param)\|_1\leq r$} & $\frac{LB}{\sqrt{n}} \times r\sqrt{\ln(\din)}$\\
        \hline
        {\footnotesize \cite[Thm. 6]{E22Barron} \quad \cite[Proposition 7]{Bach17ConvexNN}} & {\footnotesize LFCN with $D=2$, no bias, $\dout=1$ (two-layer network)} & {\footnotesize $\|\Phi(\param)\|_1\leq r$} & $\frac{LB}{\sqrt{n}} \times r\sqrt{\ln(\din)}$ \\
        \hline
        {\footnotesize \cite[Corollary 7]{Neyshabur15NormBasedControls}} & {\footnotesize DAG, no bias, $\dout=1$} & {\footnotesize $\|\Phi(\param)\|_1\leq r$} & $\frac{LB}{\sqrt{n}} \times 2^D r\sqrt{\ln(\din)}$ \\
        \hline
        {\footnotesize \cite[Theorem 3.2]{Golowich18GeneralizationBound}} & {\footnotesize LFCN with arbitrary $D$, no bias, $\dout=1$} & {\footnotesize $\prod\limits_{d=1}^D \|M_d\|_{1,\infty} \leq R$} & $\frac{LB}{\sqrt{n}} \times R \sqrt{D + \ln(\din)}$\\
        \hline
        {\footnotesize \cite[Corollary 2]{Barron19V}} & {\footnotesize LFCN with arbitrary $D$, no bias, $\dout=1$} & {\footnotesize $\|\Phi(\param)\|_1\leq r$} & $\frac{LB}{\sqrt{n}} \times r \sqrt{D + \ln(\din)}$\\
        \hline
        \hline
        {Here, \Cref{thm:GeneralizationBound}} & {\footnotesize {DAG, {\em with biases}, arbitrary $\dout$, with ReLU, identity and $k$-max-pooling neurons for $k\in\{k_1,\dots,k_P\}\subset\{1,\dots,K\}$}} & {\footnotesize ${\|\Phi(\param)\|_1\leq r}$}  & ${\frac{LB}{\sqrt{n}} \times r\sqrt{D\ln(PK) + \ln(\din\dout)}}$
        \\
    \end{tabular}
    \end{adjustbox}
\end{table}
The versatility of the proposed tools
\textbf{enables us to compute for the first time generalization bounds based on the $L^1$ path-norm on networks really used in practice}. This is the opportunity to assess the current state of the gap between theory and practice, and to diagnose possible room for improvements. As a concrete example, we demonstrate that on ResNet18 trained on ImageNet: 1) the proposed generalization bound can be numerically computed; 2) for a (dense) ResNet18 trained with standard tools, roughly {\em 30 orders of magnitude} would need to be gained for this path-norm based bound to match practically observed generalization error; 3) the same bound evaluated on a {\em sparse} ResNet18 (trained with standard sparsification techniques) is decreased by up to $13$ orders of magnitude. We conclude the paper by discussing promising leads to reduce this gap.

\textbf{Paper structure.} \Cref{sec:NN} introduces the ReLU networks being considered, and generalizes to this model the central definitions and results related to the path-lifting, the path-activations and path-norms. \Cref{sec:GeneralizationBound} state a versatile generalization bound for such networks based on path-norm, and sketches its proof. \Cref{sec:exp} reports numerical experiments on ImageNet and ResNets. Related works are discussed along the way, and we refer to \Cref{sec:RelatedWorks} for more details.
\section{ReLU model and path-lifting}\label{sec:NN}
\vspace{-0.2cm}
\Cref{subsec:model} defines a general DAG ReLU model that covers modern architectures. \Cref{subsec:pathnorm} then introduces the so-called path-norms and extends related known results to this general model.
\vspace{-0.2cm}
\subsection{ReLU model that covers modern networks}\label{subsec:model}
\vspace{-0.1cm}
\Cref{def:NN} introduces the model considered here (extending, \eg, \cite{Devore21NNApprox}).

\begin{definition}\label{def:ReluKpool}
The ReLU function is defined as $\relu(x):=x\mathbbm{1}_{x\geq 0}$ for $x\in\R$. The $k$-max-pooling function $\kpool(x) := x_{(k)}$ returns  the $k$-th largest coordinate of $x\in\R^d$.
\end{definition}

\begin{definition}
\label{def:NN} Consider a Directed Acyclic Graph (DAG) $G=(\NeuronSet, E)$ with edges $E$, and vertices $\NeuronSet$ called neurons. For a neuron $v$, the sets $\ant(v),\suc(v)$ of antecedents and successors of $v$ are $ant(v) := \{u\in\NeuronSet, u\to v \in E\}, \suc(v) := \{u\in\NeuronSet, v\to u \in E\}$. Neurons with no antecedents (resp. no successors) are called input (resp. output) neurons, and their set is denoted $\NeuronsIn$ (resp. $\NeuronsOut$). Input and output dimensions are respectively $\din := |\NeuronsIn|$ and $\dout := |\NeuronsOut|$. 

$\bullet$ A {\bf ReLU neural network architecture} is a tuple $(G, (\rho_v)_{v\in\NeuronSet\setminus\NeuronsIn})$ composed of a DAG $G=(\NeuronSet,E)$ with attributes $\rho_v\in \{\id,\relu\}\cup \{\kpool, k\in\Ns\}$ for $v\in\NeuronSet\setminus(\NeuronsOut\cup\NeuronsIn)$ and $\rho_v=\id$ for $v\in\NeuronsOut$. We will again denote the tuple $(G, (\rho_v)_{v\in\NeuronSet\setminus\NeuronsIn})$ by $G$, and it will be clear from context whether the results depend only on $G=(\NeuronSet,E)$ or also on its attributes. Define $\NeuronSet_{\rho}:=\{v\in \NeuronSet, \rho_v=\rho\}$ for an activation $\rho$, and $\NeuronSet_{\pool} :=\cup_{k\in\Ns} \NeuronSet_{\kpool}$. A neuron in $\NeuronSet_{\pool}$ is called a $*$-max-pooling neuron. For $v\in\NeuronSet_{\pool}$, its kernel size is defined as being $|\ant(v)|$.

$\bullet$ {\bf Parameters} associated with this architecture are vectors\footnote{For an index set $I$, denote $\R^I = \{(\param_i)_{i\in I}, \param_i\in\R\}$.} $\param\in\R^G:=\R^{E\cup \NeuronSet\setminus\NeuronsIn}$. We call bias $\bias_v:=\param_v$ the coordinate associated with a neuron $v$ (input neurons have no bias), and denote $\paramfromto{u}{v}$ the weight associated with an edge $u\to v\in E$. We will often denote $\paramto{v}:=(\paramfromto{u}{v})_{u\in\ant(v)}$ and $\paramfrom{v}:=(\paramfromto{u}{v})_{u\in\suc(v)}$.

$\bullet$ The {\bf realization} of a neural network with parameters $\param\in\R^G$ is the function
$R^{G}_\param:\R^{\NeuronsIn}\to\R^{\NeuronsOut}$ (simply denoted $R_\param$ when $G$ is clear from the context) defined for every input $x\in\R^{\NeuronsIn}$ as
\[
    R_\param(x) := (v(\param, x))_{v\in\NeuronsOut},
\]
where we use the same symbol $v$ to denote a neuron $v\in\NeuronSet$ and the associated function $v(\param,x)$, defined as $v(\param, x) := x_v$ for an input neuron $v$, and defined by induction otherwise
\begin{equation}\label{def:Neurons}
    v(\param, x) := \left\{\begin{array}{cc}
        \rho_v(\bias_v + \sum_{u\in\ant(v)} u(\param, x)\paramfromto{u}{v}) & \textrm{if }\rho_v=\relu\textrm{ or }\rho_v=\id,\\
        \kpool\left((\bias_v + u(\param, x)\paramfromto{u}{v})_{u\in\ant(v)} \right) & \textrm{if }\rho_v=\kpool.
    \end{array}\right.
\end{equation}
\end{definition}
Such a model indeed encompasses modern networks via the following implementations:\\
$\bullet$ \emph{Max-pooling:} set $\rho_v=\kpool$ for $k=1$, $\bias_v=0$ and $\paramfromto{u}{v}=1$ for every $u\in\ant(v)$.

$\bullet$ \emph{Average-pooling:} set $\rho_v=\id$, $\bias_v=0$ and $\paramfromto{u}{v}=1/|\ant(v)|$ for every $u\in\ant(v)$.

$\bullet$ \emph{GroupSort/Top-k operator:} use the DAG structure and $*$-max-pooling neurons \citep{Anil19GroupSort, Sander23DifferentiableTopK}.

$\bullet$ \emph{Batch normalization:} set $\rho_v=\id$ and weights accordingly. Batch normalization layers only differ from standard affine layers by the way their parameters are updated during training.

$\bullet$ \emph{Skip connections:} via the DAG structure, the outputs of any past layers can be added to the pre-activation of any neuron by adding connections from these layers to the given neuron.

$\bullet$ \emph{Convolutional layers:} consider them as (doubly) circulant/Toeplitz fully connected layers.

\subsection{Path-lifting, path-activations and path-norms}\label{subsec:pathnorm}
Given a general DAG ReLU network $G$ as in \Cref{def:NN}, it is possible to define a set of paths $\paths^G$, a path-lifting $\Phi^{G}$ and path-activations $\activation^G$, see \Cref{def:PhiActivation} in the supplementary. The $L^q$ path-norm is simply $\|\Phi^G(\param)\|_q$. Some bounds we later establish also exploit so-called mixed path-norms $\|\Phi^G(\param)\|_{q,r}$ ($=\|\Phi^G(\param)\|_q$ when $q=r$) introduced in \Cref{def:MixedPathNorm}. Superscript $G$ is omitted when obvious from the context. The interest is that (mixed) path-norms, which are easy to compute, can be interpreted as Lipschitz bounds of the network, smaller than another Lipschitz bound based on products of operator norms. We give here a high-level overview of the definitions and properties, and refer to \Cref{app:ModelsBasics} for formal definitions, proofs and technical details. We highlight that the definitions and properties coincide with previously known ones in classical simpler settings. 

\textbf{Path-lifting and path-activations: fundamental properties.} The path-lifting $\Phi$ and the path-activations $\activation$ are defined to ensure the next fundamental properties: 1) for each parameter $\param$, the path-lifting $\Phi(\param)\in\R^{\paths}$ is independent of the inputs $x$, and polynomial in the parameters $\param$ in a way that it is invariant under all neuron-wise rescaling symmetries\footnote{Because of positive-homogeneity of the considered activations functions, the realized function is preserved \citep{Stock22Embedding} when the incoming weights and the bias of a neuron are multiplied by $\lambda>0$, while its outgoing weights are divided by $\lambda$. Path-norms inherit such symmetries and are further invariant to certain neuron permutations, typically within each layer in the case of layered fully-connected} networks.; 2) $\activation(\param,x) \in \R^{\paths \times (\din+1)}$ takes a finite number of values and is piece-wise constant as a function of $(\param,x)$; and 3) denoting $\Phito{v}$ and $\ActTo{v}$ to be the same objects but associated with the graph deduced from $G$ by keeping only the largest subgraph of $G$ with the same inputs as $G$ and with single output $v$, the output of every neuron $v$ can be written as
\begin{equation}\label{eq:ForwardWithPhi}
v(\param,x) = \inner{\Phito{v}(\param)}{\ActTo{v}(\param,x) \left(\begin{array}{c}
     x \\
     1 
\end{array}\right)}.
\end{equation}
Compared to previous definitions given in simpler models (no $k$-max-pooling even for a single given $k$, no skip connections, no biases, one-dimensional output and/or layered network) \citep{Kawaguchi17GeneralizationInDL, BonaPellissier22NipsIdentifiability, Stock22Embedding}, the main novelty is essentially to properly define the path-activations $A(\param,x)$ in the presence of $*$-max-pooling neurons: when going through a $k$-max-pooling neuron, a path stays active only if the previous neuron of the path is the first in lexicographic order to be the $k$-th largest input of this pooling neuron.

\textbf{Path-norms are easy to compute.} It is mentioned in \citet[Appendix C.6.5]{GKDziugaite20SearchRobustMeasuresGeneralization} and \citet[Equation (44)]{Jiang20FantasticGeneralizationMeasures} (without proof) that for layered fully-connected ReLU networks without biases, the $L^2$ path-norm can be computed in a single forward pass with the formula: $\|\Phi(\param)\|_2^2 = \|R_{|\param|^2}(\bm{1})\|_1$, where $|\alpha|^q$ is the vector $\alpha$ with $x\mapsto |x|^q$ applied coordinate-wise and where $\bm{1}$ is the constant input equal to one. This can be proved in a straightforward way, using \Cref{eq:ForwardWithPhi}, and extended to mixed path-norm: $\|\Phi(\param)\|_{q,r} = \|
\ |
R_{|\param|^q}(\bm{1})|^{1/q}\ \|_r$. 
However, \Cref{app:ModelsBasics} shows that this formula is \emph{false} as soon as there is at least one $*$-max-pooling neuron, and that easy computation remains possible by first replacing the activation function of $*$-max-pooling neurons with the identity before doing the forward pass. Average-pooling neurons also need to be explicitly modeled as described after \Cref{def:NN}, to apply $x\mapsto |x|^q$ to their weights.

\textbf{The mixed $L^{1,r}$ path-norm is a Lipschitz bound.} \Cref{eq:ForwardWithPhi} is fundamental to understand the role of path-norms. It shows that $\Phi$ contains information about the slopes of the function realized by the network on each region where $\activation$ is constant. A formal result that goes in this direction is the Lipschitz bound $\|R_\param(x)-R_{\param}(x')\|_{r}\leq \|\Phi(\param)\|_{1,r}\|x-x'\|_\infty$, previously only known in the specific case $r=1$, for {\em scalar-valued} layered fully-connected ReLU networks \citep[before Section 3.4]{Neyshabur17ImplictRegularizationInDL}
\citep[Theorem 5]{PhDFurusho20Lipschitz}, which does not require mixed path-norms since $\|\Phi(\param)\|_{q,r}=\|\Phi(\param)\|_q$ for {\em scalar-valued} networks. This is generalized to the more general case of \Cref{def:NN} in \Cref{lem:LipsX}. This allows for leveraging generic generalization bounds that apply to the set of all $L$-Lipschitz functions $f:[0,1]^\din\to[0,1]$, however these bounds suffer from the curse of dimensionality \citep{Luxburg04LipschitzGeneralization}, unlike the bounds established in \Cref{sec:GeneralizationBound}.

\textbf{Path-norms tightly lower bound products of operator norms.} For \emph{layered fully-connected} ReLU networks (LFCN) with no bias, \ie, $R_\param(x) = M_D\relu(M_{D-1}\dots \relu(M_1x))$, with matrices $M_1,\dots, M_D$, 
another known Lipschitz bound is \citep{Neyshabur18PACBayesSpectrallyNormalizedMarginBounds, Gonon22Quantization} $\prod_{d=1}^D \|M_d\|_{q,\infty}$ with $q=1$, where $\|M\|_{q,\infty}$ is the maximum $L^q$ norm of a row of matrix $M$. This product is also used to derive generalization guarantees \citep{Neyshabur15NormBasedControls, Bartlett17SpectralGeneralizationBound, Golowich18GeneralizationBound}. So which one of path-norm and products of operator norms should be used? There are at least three reasons to consider the path-norm. \textbf{First}, it holds $\|\Phi(\param)\|_{q,\infty} \leq \prod_{d=1}^D \|M_d\|_{q,\infty}$ (\Cref{thm:Phi=MinRescalings}), with equality if the parameters are properly rescaled. This is known for {\em scalar-valued} LFCNs without biases \citep[Theorem 5]{Neyshabur15NormBasedControls}. \Cref{app:ModelsBasics} generalizes it to DAGs as in \Cref{def:NN}. The difficulty is to define the equivalent of the product of operator norms with an arbitrary DAG and in the presence of biases. Apart from that, the proof is essentially the same as in  \citet{Neyshabur15NormBasedControls}, with a similar rescaling that leads to equality of both measures, see \Cref{alg:NormalizeAlgo}. \textbf{Second}, there are cases where the product of operator norms is arbitrarily large while the path-norm is zero (see \Cref{fig:PhiNormVsProdNorm} in \Cref{app:PhiVsProducts}). Thus, it is \emph{not desirable} to have a generalization bound that depends on this product of operator norms since, compared to the path-norm, it fails to capture the complexity of the network end-to-end by \emph{decoupling the layers of neurons} one from each other. \textbf{Third}, it has been empirically observed that products of operator norms \emph{negatively} correlate with the empirical generalization error while the path-norm \emph{positively} correlates \citep[Table 2]{Jiang20FantasticGeneralizationMeasures}\citep[Figure 1]{GKDziugaite20SearchRobustMeasuresGeneralization}.
\vspace{-0.2cm}
\section{Generalization bound}\label{sec:GeneralizationBound}
\vspace{-0.2cm}
The generalization bound of this section is based on path-norm for general DAG ReLU network. It encompasses modern networks, recovers or beats the sharpest known bounds of this type, and applies to the cross-entropy loss. The top-one accuracy loss is not directly covered, but can be controlled via a bound on the margin-loss, as detailed at the end of this section.
\vspace{-0.2cm}
\subsection{Main result}\label{subsec:GeneralizationBound}
\vspace{-0.1cm}

To state the main result let us recall the definition of the generalization error.
\begin{definition}\label{def:GeneralizationError}(Generalization error)
   Consider an architecture $G$ (\Cref{def:NN}) with input and output dimensions $\din$ and $\dout$, and a so-called loss function $\ell:\R^\dout\times\R^\dout \to \R$. The \emph{$\ell$-generalization error} of parameters $\param$ on a collection $Z$ of $n\in\Ns$ pairs of input/output $z_i = (x_i,y_i)\in\R^{\din}\times\R^\dout$ and with respect to a probability measure $\mu$ on $\R^{\din}\times\R^\dout$ is:
    \begin{equation*}
        \begin{aligned}
            \textrm{$\ell$-generalization error}(\param,Z,\mu) := & \underbrace{\E_{(\rv{X}_0,\rv{Y}_0)\sim \mu}\left(\ell\left(R_{\param}(\rv{X}_0), \rv{Y}_0\right)\right)}_{\textrm{test error}} - \underbrace{\frac{1}{n}\sum_{i=1}^n \ell\left(R_{\param}(x_i), y_i\right)}_{\textrm{training error when trained on $Z$}}.
        \end{aligned}
    \end{equation*}
\end{definition}
While all the other results hold for arbitrary biases, the next theorem holds {\bf only if the $*$-max-pooling neurons have null biases}. For simplicity, we state the result when {\em all} the biases are null. We then explain how to extend the result when neurons $u\notin\NeuronSet_{\pool}$ may have $\bias_u\neq 0$ .
\begin{theorem}\label{thm:GeneralizationBound}
Consider a ReLU neural network architecture $G$ (\Cref{def:NN}) {\em with null biases}, with input/output dimensions $\din$/$\dout$. Denote $D$ its depth (the maximal length of a path from an input to an output), $P:=|\{k\in\Ns, \exists u\in\NeuronSet_{\kpool}\}|$ the number of distinct types of $*$-max-pooling neurons in $G$, 
$K:=\max_{u\in\NeuronSet_{\pool}} |\ant(u)|$ its maximal kernel size ($K:=1$ if $P=0$). Assume $\NeuronsIn\cap\NeuronsOut=\emptyset$. Consider a so-called loss function $\ell:\R^\dout\times\R^\dout \to \R$ and $L>0$ such that
\begin{equation}\label{hyp:LossLips}
        \ell(\hat{y}_1,y) - \ell(\hat{y}_2,y) \leq L \|\hat{y}_1 - \hat{y}_2\|_2, \quad \forall y, \hat{y}_1, \hat{y}_2\in\supp(\rv{Y}_1).
\end{equation}
Consider $n+1$ iid random variables $\rv{Z}_i=(\rv{X}_i,\rv{Y}_i)\sim \mu$, $0 \leq i \leq n$, with $\mu$ a probability measure on input/output pairs in $\R^{\din}\times\R^{\dout}$, and denote $\rv{Z}=(\rv{Z}_i)_{i=1,\dots,n}$. Define $\sigma:=\left(\E_{\rv{X}}\max(n,\sum_{i=1}^n \|\rv{X}_i\|_\infty^2)\right)^{1/2}$.\\
For any set of parameters $\Param$ and any estimator $\hatparam: \rv{Z} \mapsto \hatparam(\rv{Z})\in\Param$ it holds\footnote{The definition of the generalization error (\Cref{def:GeneralizationError}) has been given for deterministic $\param$ to keep things simple. The careful reader will have noted that the term corresponding to the test error in \Cref{def:GeneralizationError} has to be modified when $\param$ is a function of $\rv{Z}$. Indeed, the expectation has to be taken on an iid copy $\rv{Z}_0$ conditionally on each $\rv{Z}_i$, $i=1,\dots,n$: the test error should be defined as $\E_{\rv{Z}_0\sim\mu}\left(\ell(R_{\hatparam(\rv{Z})}(\rv{X}_0), \rv{Y}_0)|\rv{Z}\right)$. This correct definition will be used in the proof, but it has no importance here to understand the statement of the theorem.}\footnote{Classical concentration results \citep{Boucheron13ConcentrationInequalities} can be used to deduce a bound that holds with high probability under additional mild assumptions on the loss.}: 
   \[
   \E_{\rv{Z}} \textrm{$\ell$-generalization error}(\hatparam(\rv{Z}), \rv{Z}, \mu) \leq \frac{4\sigma}{n}  LC\sup_{\param\in\Param} \|\Phi(\param)\|_1
  \]
    with ($\log$ being the natural logarithm)
    \[
    C := 
    \left(D\log((3+2P)K) + \log\left(\frac{3+2P}{1+P}\din\dout\right)\right)^{1/2}.
    \]
\end{theorem}
Any neural network with nonzero biases can be transformed into an equivalent network (with same path-norms and same $R_\param$) with null biases for every $v\notin\NeuronSet_{\pool}$: add an input neuron $\biasNeuron$ with constant input equal to one, add edges between this input neuron and every neuron $v\notin\NeuronSet_{\pool}$ with parameter $\paramfromto{\biasNeuron}{v}:=\bias_v$, and set $\bias_v=0$. Thus the same result holds with $\bias_v\neq 0$ for $v\notin\NeuronSet_{\pool}$, with $\din$ replaced by $\din+1$ in the definition of $C$, and with an additional constant input coordinate equal to one in the definition of $\sigma$ so that $\sigma=\left(\E_{\rv{X}}\max\left(n, \max_{u=1,\dots,\din} \sum_{i=1}^n (\rv{X}_i)_u^2\right)\right)^{1/2}\geq \sqrt{n}$.
The proof in \Cref{app:DetailsProofGeneralizationBound} is directly given for networks with nonzero biases (except $*$-max-pooling neurons), using this construction.

\Cref{thm:GeneralizationBound} \emph{applies to the cross-entropy loss with $L=\sqrt{2}$} (see \Cref{app:CrossEntropy}) if the labels $y$ are one-hot encodings\footnote{A vector $y$ is a one-hot encoding of a class $c$ if $y=(\mathbbm{1}_{c'=c})_{c'\in\{1,\dots,\dout\}}$.}. A final softmax layer can be incorporated for free to the model by putting it in the loss. This does not change the bound since it is $1$-Lipschitz with respect to the $L^2$-norm (this is a simple consequence of the computations made in \Cref{app:CrossEntropy}).

On ImageNet, it holds $1/\sqrt{n}\leq \sigma/n \leq 2.6/\sqrt{n}$ (\Cref{sec:exp}). This yields a bounds that decays in $\mathcal{O}(n^{-1/2})$ which is better than the generic $\mathcal{O}(n^{-1/\din})$ generalization bound for Lipschitz functions \citep[Thm. 18]{Luxburg04LipschitzGeneralization} that suffer from the curse of dimensionality. Besides its wider range of applicability, this bounds also recovers or beats the sharpest known ones based on path-norm, see \Cref{tab:GeneralizationBounds}.

Finally, note that \Cref{thm:GeneralizationBound} can be tightened: the same bound holds without counting the identity neurons when computing $D$. Indeed, for any neural network and parameters $\param$, it is possible to remove all the neurons $v\in\NeuronSet_{\id}$ by adding a new edge $u\to w$ for any $u\in\ant(v), w\in\suc(v)$ with new parameter $\paramfromto{u}{v}\paramfromto{v}{w}$ (if this edge already exists, just add the latter to its already existing parameter). This still realizes the same function, with the same path-norm, but with less neurons, and thus with $D$ possibly decreased. The proof technique would also yield a tighter bound but not by much: the occurences of $3$ in $C$ would be replaced by $2$.
\begin{proof}[Sketch of proof for \Cref{thm:GeneralizationBound}] The proof idea is explained below. Details are in \Cref{app:DetailsProofGeneralizationBound}. 

\textbf{Already known ingredients.} Classical arguments \citep[Theorem 26.3]{ShalevShwartz14UnderstandingML}\citep{Maurer16VectorContraction}, that are valid for any model, bound the expected generalization error by the Rademacher complexity of the model. It remains to bound the latter, and this gets specific to neural networks. In the case of a layered fully-connected ReLU neural network with no biases and scalar output (and no skip connections nor $k$-max-pooling even for a single given $k$), \citet{Golowich18GeneralizationBound} proved that it is possible to bound this Rademacher complexity with no exponential factor in the depth, by peeling, one by one, each layer off the Rademacher complexity. To get more specific, for a class of functions $\mF$ and a function $\Psi:\R\to\R$, denote $\Rad\circ \Psi(\mF)=\E_\eps \Psi(\sup_{f\in\mF} \sum_{i=1}^n \eps_i f(x_i))$ the Rademacher complexity of $\mF$ associated with $n$ inputs $x_i$ and $\Psi$, where the $\eps_i$ are iid Rademacher variables ($\eps_i=1$ or $-1$ with equal probability). The goal for a generalization bound is to bound this in the case $\Psi(x)=\id(x)=x$. In the specific case where $\mF_D$ is the class of functions that correspond to layered fully-connected ReLU networks with depth $D$, assuming that some operator norm of each layer $d$ is bounded by $r_d$, \citet{Golowich18GeneralizationBound} basically guarantees $\Rad\circ\Psi_\lambda(\mF_D) \leq 2\Rad\circ\Psi_{\lambda r_D}(\mF_{D-1})$ for every $\lambda>0$, where $\Psi_\lambda(x)=\exp(\lambda x)$. Compared to previous works of \citet{Golowich18GeneralizationBound} that were directly working with $\Psi=\id$ instead of $\Psi_\lambda$, the important point is that working with $\Psi_\lambda$ gets the $2$ outside of the exponential. Iterating over the depth $D$, optimizing over $\lambda$, and taking a logarithm at the end yields (by Jensen's inequality) a bound on $\Rad\circ\id (\mF_D)$ with a dependence on $D$ that grows as $\sqrt{D\log(2)}$ instead of $2^D$ for previous approaches.

\textbf{Novelties for general DAG ReLU networks.} Compared to the setup of \citet{Golowich18GeneralizationBound}, there are at least three difficulties to do something similar here. First, \emph{the neurons are not organized in layers} as the model can be an arbitrary DAG. So what should be peeled off one by one? Second, \emph{the neurons are not necessarily ReLU neurons} as their activation function might be the identity (average-pooling) or $*$-max-pooling. Finally, \citet{Golowich18GeneralizationBound} has a constraint on the weights of each layer, which makes it possible to pop out the constant $r_d$ when layer $d$ is peeled off. \emph{Here, the only constraint is global}, since it constrains the paths of the network through $\|\Phi(\param)\|_1\leq r$. In particular, due to rescalings, the weights of a given neuron could be arbitrarily large or small under this constraint. 

The first difficulty is primarily addressed using a new peeling lemma (\Cref{app:Peeling}) that exploits a new contraction lemma (\Cref{app:Contraction}). The second difficulty is resolved by splitting the ReLU, $k$-max-pooling and identity neurons in different groups before each peeling step. This makes the $\log(2)$ term in \citet{Golowich18GeneralizationBound} a $\log(3+2P)$ term here ($P$ being the number of different $k$'s for which $k$-max-pooling neurons are considered). Finally, the third obstacle is overcome by \emph{rescaling the parameters} to normalize the vector of incoming weights of each neuron. This type of rescaling has also been used in \citet{Neyshabur15NormBasedControls, Barron19V}.
\end{proof}
\begin{remark}[Improved bound with assumptions on $*$-max-pooling neurons]\label{rmk:ImprovedBoundMaxPool} In the specific case where there is a single type of $k$-max-pooling neurons ($P=1$), assuming that these $k$-max-pooling neurons are grouped in layers, and that there are no skip connections going over these $k$-max-pooling layers (satisfied by ResNets, not satisfied by U-nets), a sharpened peeling argument can yield the same bound but with $C$ replaced by $C_{\textrm{sharpened}}=\left(D\log(3) + M\log(K) + \log((\din+1)\dout)\right)^{1/2}$ with $M$ being the number of $k$-max-pooling layers (cf.~   \Cref{app:Peeling}). The details are tedious so we only mention this result without proof. This basically improves $\sqrt{D\log(5K)}$ into $\sqrt{D\log(3) + M\log(K)}$. For Resnet152, $K=9$, $D=152$ and $M=1$, $\sqrt{D\log(5K)}\simeq 24$ while $\sqrt{D\log(3) + M\log(K)}\simeq 13$.
\end{remark}
\vspace{-0.2cm}
\subsection{How to deal with the top-1 accuracy loss?}\label{subsec:BoundForLosses}
\vspace{-0.2cm}
\Cref{thm:GeneralizationBound} \emph{does not apply to the top-1 accuracy loss} as \Cref{hyp:LossLips} cannot be satisfied for any finite $L>0$ in general (see \Cref{app:Top1Acc}). It is still possible to bound the expected (test) top-1 accuracy by the so-called {\em margin loss} achieved at training \citep[Lemma A.4]{Bartlett17SpectralGeneralizationBound}. 
The margin-loss is a relaxed definition of the top-1 accuracy loss. A corollary of \Cref{thm:GeneralizationBound} is the next result proved in \Cref{app:MarginLoss}.
\begin{theorem}[Bound on the probability of misclassification]\label{thm:MarginLoss}
    Consider the setting of \Cref{thm:GeneralizationBound}. 
    Assume that the labels are indices $y \in \{1,\ldots,\dout\}$. For any $\gamma>0$, it holds  
    \begin{equation}\label{eq:GeneBoundMargin}
        \begin{aligned}
            \P\left(\argmax_{c} R_\param(\rv{X}_1)_c \neq \rv{Y}_1\right) \leq  &\, \frac{1}{n}\sum_{i=1}^n \mathbbm{1}_{(R_{\hatparam(\rv{Z})}(\rv{X}_i))_{\rv{Y}_i} \leq \gamma + \max_{c \neq \rv{Y}_i}(R_{\hatparam(\rv{Z})}(\rv{X}_i))_{c}}\\
            &  + \frac{8\sigma}{n}C\frac{\sup_{\param}\|\Phi(\param)\|_1}{\gamma}.
        \end{aligned}
    \end{equation}
\end{theorem}
Note that the result is homogeneous: scaling both the outputs of the model and $\gamma$ by the same scalar leaves the classifier {\em and the corresponding bound} unchanged. 
\section{Experiments}\label{sec:exp}
\Cref{thm:GeneralizationBound} gives the first path-norm generalization bound that can be applied to modern networks (with average/$*$-max-pooling, skip connections etc.). This bound is also the sharpest known bound of this type (\Cref{tab:GeneralizationBounds}). Since this bound is also easy to compute, the goal of this section is to numerically challenge for the first time the sharpest generalization bounds based on path-norm  on modern networks. Note also that path-norms tightly lower bound products of operator norms (\Cref{app:PhiVsProducts}) so that this also challenges the latter.

\textbf{When would the bound be informative?} For ResNets trained on ImageNet, the training error associated with cross-entropy is typically between $1$ and $2$, and the top-1 training error is typically less than $0.30$. The same orders of magnitude apply to the empirical generalization error. To ensure that the test error (either for cross-entropy or top-1 accuracy) is of the same order as the training error, \emph{the bound should basically be of order $1$}.

For parameters $\param$ learned from training data, \Cref{thm:GeneralizationBound}  and \Cref{thm:MarginLoss} allow to bound the expected loss in terms of a performance measure (that depends on a free choice of $\gamma>0$ for the top-1 accuracy) on training data plus a term bounded by $\frac{4\sigma}{n}C \times L \times \|\Phi(\param)\|_1$. The Lipschitz constant $L$ is $\sqrt{2}$ for cross-entropy, and $2/\gamma$ for the top-1 accuracy.

\textbf{Evaluation of $\frac{4\sigma}{n}C$ for ResNets on ImageNet.} We further bound $\sigma/n$ by $B/\sqrt{n}$, where $B \simeq 2.6$ is the maximum $L^\infty$-norm of the images of ImageNet normalized for inference. We at most lose a factor $B$ compared to the bound directly involving $\sigma$ since it also holds $\sigma/n\geq 1/\sqrt{n}$ by definition of $\sigma$. We train on $99\%$ of ImageNet so that $n=1268355$. Moreover, recall that $C=(D\log((3+2P)K) + \log(\frac{3+2P}{1+P}(\din+1)\dout))^{1/2}$. For ResNets, $P=1$ (as there are only classical max-pooling neurons, corresponding to $k$-max-pooling with $k=1$), the kernel size is $K=9$, $\din=224\times 224\times 3$, $\dout=1000$, and the depth is $D=2+\textrm{\# basic blocks}\times \textrm{\# conv per basic block}$, with the different values available in \Cref{app:expes}. The values for $4BC/\sqrt{n}$ are reported in \Cref{tab:ValueBoundResNets}. Given these results and the values of the Lipschitz constant $L$, on ResNet18, \emph{the bound would be informative only when $\|\Phi(\param)\|_1\lesssim 10$ or $\|\Phi(\param)\|_1/\gamma \lesssim 10$} respectively for the cross-entropy and the top-1 accuracy.

\begin{table}[t]
    \centering
    \caption{Numerical evaluations on ResNets and ImageNet1k with 2 significant digits. Multiplying by the Lipschitz constant $L$ of the loss and the path-norm gives the bound in \Cref{thm:GeneralizationBound}. The second line reports the values when the analysis is sharpened for max-pooling neurons, see \Cref{rmk:ImprovedBoundMaxPool}.}  \label{tab:ValueBoundResNets}
    \begin{tabular}{x{0.18\textwidth}|x{0.1\textwidth}|x{0.1\textwidth}|x{0.1\textwidth}|x{0.1\textwidth}|x{0.1\textwidth}}
         ResNet& 18 & 34 & 50 & 101 & 152 \\
        \hline
        $\frac{4}{\sqrt{n}}CB=$ & $0.088$ & $0.11$ & $0.14$ & $0.19$ & $0.23$ \\
        $\frac{4}{\sqrt{n}}C_{\textrm{sharpened}}B=$ & $0.060$ & $0.072$ & $0.082$ & $0.11$ & $0.13$
    \end{tabular}
\end{table}
We now compute the path-norms of trained ResNets, both dense and sparse, using the simple formula proved in \Cref{thm:ComputePathNorm} in appendix.

\textbf{$L^1$ path-norm of pretrained ResNets are $30$ orders of magnitude too large.} \Cref{tab:IncreasingDepth} shows that the $L^1$ path-norm is 30 orders of magnitude too large to make the bound informative for the cross-entropy loss. The choice of $\gamma$ is discussed in \Cref{app:expes}, where we observe that there is no possible choice that leads to an informative bound for top-1 accuracy in this situation.

\begin{table}[t]
    \centering
    \caption{Path-norms of pretrained ResNets available on PyTorch, computed in float32.}\label{tab:IncreasingDepth}
    \begin{tabular}{x{0.12\textwidth}|x{0.11\textwidth}|x{0.11\textwidth}|x{0.11\textwidth}|x{0.11\textwidth}|x{0.11\textwidth}}
         ResNet& 18 & 34 & 50 & 101 & 152 \\
        \hline
        $\|\Phi(\param)\|_1$ & $1.3\times 10^{30}$ & overflow & overflow & overflow & overflow \\
        \hline
        $\|\Phi(\param)\|_2$ & $2.5\times 10^2$ & $1.1\times 10^2$ & $2.0\times 10^8$ & $2.9\times 10^{9}$ & $8.9\times 10^{10}$ \\
        \hline
        $\|\Phi(\param)\|_4$ & $7.2\times 10^{-6}$ & $4.9\times 10^{-6}$ & $6.7\times 10^{-4}$ & $3.0\times 10^{-4}$ & $1.5\times 10^{-4}$
    \end{tabular}
\end{table}

\textbf{Sparse ResNets can decrease the bounds by 13 orders of magnitude.} We just saw that pretrained ResNets have very large $L^1$ path-norm. Does every network with a good test top-1 accuracy have such a large $L^1$ path-norm? Since any zero in the parameters $\param$ leads to many zero coordinates in $\Phi(\param)$, we now investigate whether sparse versions of ResNet18 trained on ImageNet have a smaller path-norm. Sparse networks are obtained with iterative magnitude pruning plus rewinding, with hyperparameters similar to the one in \citet[Appendix A.3]{Frankle21MissingTheMark}. Results show that the $L^1$ path-norm decreases from $\simeq10^{30}$ for the dense network to $\simeq 10^{17}$ after 19 pruning iterations, basically losing between a half and one order of magnitude per pruning iteration. Moreover, the test top-1 accuracy is better than with the dense network for the first 11 pruning iterations, and after 19 iterations, the test top-1 accuracy is still way better than what would be obtained by guessing at random, so this is still a non-trivial matter to bound the generalization error for the last iteration. Details are in \Cref{app:expes}. This shows that there are indeed practically trainable networks with much smaller $L^1$ path-norm that perform well. It remains open whether alternative training techniques, possibly with path-norm regularization, could lead to networks combining good performance and informative generalization bounds.

\textbf{Additional observations: increased depth and train size.} In practice, increasing the size of the network (\ie the number of parameters) or the number of training samples can improve generalization. We can, again, assess for the first time whether the bounds based on path-norms follows the same trend for standard modern networks. \Cref{tab:IncreasingDepth} shows that path-norms of pretrained ResNets available on PyTorch roughly increase with depth. This is complementary to \citet[Figure 1]{GKDziugaite20SearchRobustMeasuresGeneralization} where it is empirically observed on {\em simple layered fully-connected} models that path-norm has difficulty to correlate positively with the generalization error when the depth grows. For increasing training sizes, we did not observe a clear trend for the $L^1$ path-norm, which seems to mildly evolve with the number of epochs rather than with the train size, see \Cref{app:expes} for details.
\vspace{-0.4cm}
\section{Conclusion}
\vspace{-0.3cm}
\textbf{Contribution.} To the best of our knowledge, this work is the first to introduce path-norm related tools for general DAG ReLU networks (with average/$*$-max-pooling, skip connections), and \Cref{thm:GeneralizationBound} is the first generalization bound valid for such networks based on path-norm. This bound recovers or beats the sharpest known ones of the same type. Its ease of computation leads to the first experiments on modern networks that assess the promises of such approaches. A gap between theory and practice is observed for a dense version of ResNet18 trained with standard tools: the bound is $30$ orders of magnitude too large on ImageNet. 

\textbf{Possible leads to close the gap between theory and practice.} 1) Without changing the bound of \Cref{thm:GeneralizationBound}, sparsity seems promising to reduce the path-norm by several orders of magnitude without changing the performance of the network. 2) \Cref{thm:GeneralizationBound} results from the worst situation (that can be met) where all the inputs activate all the paths of the network simultaneously. Bounds involving the expected path-activations could be tighter. The coordinates of $\Phi(\param)$ are elementary bricks that can be summed to get the slopes of $R_\param$ on the different region where $R_\param$ is affine \citep{arora2018understandingReLUnetworks}, $\|\Phi(\param)\|_1$ is the sum of all the bricks in absolue value, resulting in a worst-case uniform bound for all the slopes. Ideally, the bound should rather depend on the expected slopes over the different regions, weighted by the probability of falling into these regions. 3) Weight sharing may leave room for sharpened analysis \citep{Pitas19LimitsWeightSharingGeneralization, Galanti23WeightSharingGene}. 4) A $k$-max-pooling neuron with kernel size $K$ only activates $1/K$ of the paths, but the bound sums the coordinates of $\Phi$ related to these $K$ paths. This may lead to a bound $K$ times too large in general (or even more in the presence of multiple maxpooling layers). 5) Possible bounds involving the $L^q$ path-norm for $q>1$ deserve a particular attention, since numerical evaluations show that they are several orders of magnitude below the $L^1$ norm. 

\textbf{Extensions to other architectures.} Despite its applicability to a wide range of standard modern networks, the generalization bound in \Cref{thm:GeneralizationBound} does not cover networks with other activations than ReLU, identity, and $*$-max-pooling. The same proof technique could be extended to new activations that: 1) are positively homogeneous, so that the weights can be rescaled without changing the associated function; and 2) satisfy a contraction lemma similar to the one established here for ReLU and max neurons (typically requiring the activation to be Lipschitz). A plausible candidate is Leaky ReLU. For smooth approximations of the ReLU, such as the SiLU (for Efficient Nets) and the Hardswish (for MobileNet-V3), parts of the technical lemmas related to contraction may extend since they are Lipschitz, but these activations are not positively homogeneous.
\vspace{-0.3cm}
\subsubsection*{Acknowledgments}
\vspace{-0.2cm}
This work was supported in part by the AllegroAssai ANR-19-CHIA-0009 and NuSCAP ANR-20-CE48-0014 projects of the French Agence Nationale de la Recherche.

The authors thank the Blaise Pascal Center for the computational means. It uses the SIDUS \citep{quemener2013} solution developed by Emmanuel Quemener.

\bibliography{main.bbl}
\bibliographystyle{iclr2024_conference}

\appendix
\newpage

\section*{\bf Supplementary material}

\section{Model's basics}\label{app:ModelsBasics}
\Cref{def:PhiActivation} introduces the path-lifting and the path-activations associated with the general model described in \Cref{def:NN}. We choose to call it path-lifting as in \cite{BonaPellissier22NipsIdentifiability} (and not an embedding as initially named in \cite{Stock22Embedding}, since it is not injective). Formally, a lifting in the sense of category theory would factorize the mapping  $\param\mapsto R_\param$, which is not the case of $\param\mapsto\Phi(\param)$, but is the case with the extended  mapping  $\param\to (\sgn(\param), \Phi(\param))$. We chose the name lifting for $\Phi(\param)$ for the sake of brevity.

\begin{definition}[Paths and depth in a DAG]\label{def:Paths}
Consider a DAG $G=(N,E)$ as in \Cref{def:NN}. A path of $G$ is any sequence of neurons $v_0,\dots, v_d$ such that each $v_{i}\to v_{i+1}$ is an edge in $G$. Such a path is denoted $p = v_0 \to \ldots \to v_d$. This includes paths reduced to a single $v\in\NeuronSet$, denoted $p=v$. The {\em length} of a path is $\length{p}=d$ (the number of edges). We will denote $p_\ell := v_\ell$ the $\ell$-th neuron for a general $\ell\in\{0,\dots,\length{p}\}$ and use the shorthand $\pathend{p}=v_{\length{p}}$ for the last neuron. The {\em depth of the graph} $G$ is the maximum length over all of its paths. If $v_{d+1} \in \suc(\pathend{p})$ then $p \to v_{d+1}$ denotes the path $v_0 \to \ldots \to v_d \to v_{d+1}$. We denote by $\paths^{G}$ (or simply $\paths$) the set of paths ending at an output neuron of $G$.
\end{definition}
\begin{definition}[Sub-graph ending at a given neuron]\label{def:subgraphToV}
    Given a neuron $v$ of a DAG $G$, we denote $G^{\to v}$ the graph deduced from $G$ by keeping only the largest subgraph with the same inputs as $G$ and with $v$ as a single output: every neuron $u$ with no path to reach $v$ through the edges of $G$ is removed, as well as all its incoming and outcoming edges. We will use the shorthand $\pathsto{v}:=\paths^{G^{\to v}}$ to denote the set of paths in $G$ ending at $v$.
\end{definition}
\begin{definition}[Path-lifting and path-activations]\label{def:PhiActivation} 
Consider a ReLU neural network architecture $G$ as in \Cref{def:NN} and parameters $\param\in\R^G$ associated with $G$. For $p\in\paths$, define
\[
\Phi_p(\param) := \left\{\begin{array}{cc}
    \prod\limits_{\ell=1}^{\length{p}} \paramfromto{v_{\ell-1}}{v_{\ell}} & \textrm{if }p_0\in\NeuronsIn,\\
    \bias_{p_0} \prod\limits_{\ell=1}^{\length{p}} \paramfromto{v_{\ell-1}}{v_{\ell}} & \textrm{otherwise}, 
\end{array}\right.
\]
where an empty product is equal to $1$ by convention. The path-lifting $\Phi^{G}(\param)$ of $\param$ is
\[
    \Phi^G(\param) := (\Phi_p(\param))_{p\in\paths^G}.
\]
This is often denoted $\Phi$ when the graph $G$ is clear from the context. We will use the shorthand $\Phito{v} := \Phi^{G^{\to v}}$ to denote the path-lifting associated with $G^{\to v}$ (\Cref{def:subgraphToV}).

Consider an input $x$ of $G$. The activation of an {\em edge} $u\to v$ on $(\param,x)$ is defined to be $a_{u\to v}(\param,x):=1$ when $v$ is an identity neuron; $a_{u\to v}(\param,x):=\mathbbm{1}_{v(\param,x)>0}$ when $v$ is a ReLU neuron; and when $v$ is a $k$-max-pooling neuron, define $a_{u\to v}(\param,x):=1$ if the neuron $u$ is the first in $\ant(v)$ in lexicographic order to satisfy $u(\param,x):=\kpool\left((w(\param,x))_{w\in\ant(v)}\right)$ and $a_{u\to v}(\param,x):=0$ otherwise. The activation of a {\em neuron} $v$ on $(\param,x)$ is defined to be $a_v(\param,x):=1$ if $v$ is an input neuron, an identity neuron, or a $k$-max-pooling neuron, and $a_v(\param,x):=\mathbbm{1}_{v(\param,x)> 0}$ if $v$ is a ReLU neuron. We then define the activation of a {\em path} $p\in\paths$ with respect to input $x$ and parameters $\param$ as: $a_p(\param,x):=a_{p_0}(\param,x) \prod_{\ell=1}^{\length{p}} a_{v_{\ell-1} \to v_{\ell}}(\param,x)$ (with an empty product set to one by convention). %
Consider a new symbol $\biasNeuron$  that is not used for denoting neurons. The path-activations matrix $\activation(\param,x)$ is defined as the matrix in $\R^{\paths\times (\BiasAndNeuronsIn)}$ such that for any path $p\in\paths$ and neuron $u\in\BiasAndNeuronsIn$
\[
(\activation(\param,x))_{p, u}:= \left\{\begin{array}{cc}
        a_p(\param,x) \mathbbm{1}_{p_0=u} & \textrm{if }u\in\NeuronsIn, \\
        a_p(\param,x) \mathbbm{1}_{p_0\notin \NeuronsIn} & \textrm{otherwise when }u=\biasNeuron.
    \end{array}\right.
\]
\end{definition}
The next lemma shows how the path-lifting and the path-activations offer an equivalent way to define the model of \Cref{def:NN}.

\begin{lemma}\label{lem:ForwardWithPhi}
Consider a ReLU network as in \Cref{def:NN}. For every neuron $v$, every input $x$ and every parameters $\param$:
\begin{equation}\label{eq:ForwardWithPhiAppendix}
v(\param,x) = \inner{\Phito{v}(\param)}{\ActTo{v}(\param,x) \left(\begin{array}{c}
     x \\
     1 
\end{array}\right)}.
\end{equation}
\end{lemma}

\begin{proof}[Proof of \Cref{lem:ForwardWithPhi}]
For any neuron $v$, recall that $\pathsto{v}$ is the set of paths ending at neuron $v$ (\Cref{def:subgraphToV}). We want to prove
\begin{align}
    v(\param,x) & = \inner{\Phito{v}(\param)}{\ActTo{v}(\param,x) \left(\begin{array}{c}
         x \\
         1 
    \end{array}\right)} \nonumber\\
    & = \sum_{p\in\pathsto{v}} \Phi_p(\param) a_p(\param,x) x_{p_0}. \label{eq:GoalLemA.1}
\end{align}
where we recall that $p_0$ denotes the first neuron of a path $p$ (\Cref{def:Paths}). 
For paths starting at an input neuron, $x_{p_0}$ is simply the corresponding coordinate of $x$. For other paths, we use the convention $x_u:=1$ for any neuron $u \in \NeuronSet \setminus \NeuronsIn$.

The proof of \Cref{eq:GoalLemA.1} goes by induction on a topological sorting \cite{Cormen} of the neurons. We start with input neurons since by \Cref{def:NN}, these are the ones without antecedents so they are the first to appear in a topological sorting. Consider an input neuron $v$. The only path in $\pathsto{v}$ is $p=v$. By \Cref{def:PhiActivation}, it holds $\Phi_p(\param)=1$ (empty product) and $a_p(\param,x)=a_v(\param,x)=1$. Moreover, we have $v(\param,x)=x_v$ (\Cref{def:NN}). This proves \Cref{eq:GoalLemA.1} for input neurons.

Now, consider $v\notin\NeuronsIn$ and assume that \Cref{eq:GoalLemA.1} holds true for every neuron $u\in\ant(v)$. We first prove by cases that
\begin{equation}\label{eq:preInductionForward}
    v(\param,x) = a_v(\param,x) \bias_v + \sum_{u\in\ant(v)} u(\param,x)a_{u\to v}(\param,x)\paramfromto{u}{v}.
\end{equation}

\textbf{Case of an identity neuron.} If $v$ is an identity neuron, then by \Cref{def:NN}
\[
v(\param,x) = \bias_v + \sum_{u\in\ant(v)} u(\param,x)\paramfromto{u}{v}.
\]
Moreover, by \Cref{def:PhiActivation} we have $a_v(\param,x)=1$ and $a_{u\to v}(\param,x) = 1$, since $v$ is an identity neuron, so that the previous equality can also be written as
\[
v(\param,x) = {a_v(\param,x)} \bias_v + \sum_{u\in\ant(v)} u(\param,x){a_{u\to v}(\param,x)}\paramfromto{u}{v}.
\]
This shows \Cref{eq:preInductionForward} in this case.

\textbf{Case of a ReLU neuron.} If $v$ is a ReLU neuron then similarly 
\begin{align*}
    v(\param,x)  = \relu\left(\bias_v + \sum_{u\in\ant(v)} u(\param,x)\paramfromto{u}{v}\right) 
    & = \mathbbm{1}_{v(\param,x)>0}\left(\bias_v + \sum_{u\in\ant(v)} u(\param,x)\paramfromto{u}{v}\right) \\
    & = \underbrace{\mathbbm{1}_{v(\param,x)> 0}}_{=a_{v}(\param,x)}\bias_v + \sum_{u\in\ant(v)} u(\param,x) \underbrace{\mathbbm{1}_{v(\param,x)> 0}}_{=a_{u\to v}(\param,x)}\paramfromto{u}{v}
\end{align*}
This shows again \Cref{eq:preInductionForward}.

\textbf{Case of a $k$-max-pooling neuron.} When $v$ is a $k$-max-pooling neuron, it holds:
\begin{align*}
    v(\param,x) & = \kpool\left(\left( \bias_v + u(\param,x)\paramfromto{u}{v}\right)_{u\in\ant(v)}\right) \\
    & = \bias_v + \sum_{u\in\ant(v)} \underbrace{\mathbbm{1}_{\textrm{$u$ is the first to realize this $\kpool$}}}_{=a_{u\to v}(\param,x)\textrm{ (\Cref{def:PhiActivation})}}u(\param,x)\paramfromto{u}{v} && \textrm{(\Cref{def:ReluKpool})} \\
    & = \underbrace{a_v(\param,x)}_{=1\textrm{ (\Cref{def:PhiActivation})}} \bias_v + \sum_{u\in\ant(v)} u(\param,x)a_{u\to v}(\param,x)\paramfromto{u}{v}.
\end{align*}
This finishes proving \Cref{eq:preInductionForward} in every case. Using the induction hypothesis on the antecedents of $v$, \Cref{eq:preInductionForward} implies
\[
    v(\param,x) = a_v(\param,x)\bias_v + \sum_{u\in\ant(v)} \left(\sum_{p\in\pathsto{u}} \Phi_p(\param) a_p(\param,x)x_{p_0}\right)a_{u\to v}(\param,x)\paramfromto{u}{v}.
\]
We want to prove that this is equal to
\[
\sum_{p\in\pathsto{v}} \Phi_p(\param) a_p(\param,x) x_{p_0}.
\]
A path $\tilde{p}\in\pathsto{v}$ is either the path $\tilde{p}=v$ starting and ending at $v$, or it can be written in a unique way as $\tilde{p}=p\to v$ where $p\in\pathsto{u}$ is a path ending at an antecedent $u$ of $v$. For the simple path $\tilde{p}=v$, it holds by definition (\Cref{def:PhiActivation}): $a_{\tilde{p}}(\param,x)=a_v(\param,x)$, $\Phi_{\tilde{p}}(\param)=\bias_v$ and by the convention used in this proof we have $x_{\tilde{p}_0}=x_v=1$ since $v$ is not an input neuron. This shows:
\[
a_v(\param,x)\bias_v = \Phi_{\tilde{p}}(\param)a_{\tilde{p}}(\param,x)x_{\tilde{p}_0}.
\]
When $\tilde{p}=p\to v$ as above, it holds by definition: $x_{p_0} = x_{\tilde{p}_0}$, $\Phi_{\tilde{p}}(\param) = \Phi_p(\param)\paramfromto{u}{v}$ and $a_{\tilde{p}}(\param,x)=a_{p}(\param,x)a_{u\to v}(\param,x)$. It then holds:
\[
\Phi_p(\param) a_p(\param,x)x_{p_0}a_{u\to v}(\param,x)\paramfromto{u}{v} = \Phi_{\tilde{p}}(\param)a_{\tilde{p}}(\param,x)x_{\tilde{p}_0}.
\]
This concludes the induction and proves the result.
\end{proof}

A straightforward consequence of \Cref{lem:ForwardWithPhi} is that path-norms can be used to bound from above the Lipschitz constant of $x\mapsto R_\param(x)$.

\begin{definition}(Mixed path-norm)\label{def:MixedPathNorm}
    For $0< q,r\leq \infty$, define:
    \begin{equation}\label{eq:MixedPathNorm}
        \|\Phi(\param)\|_{q,r} := \left\|\left(\|\Phito{v}(\param)\|_q\right)_{v\in\NeuronsOut}\right\|_r.
    \end{equation}
\end{definition}

\begin{lemma}\label{lem:LipsX}
    Consider $0<r\leq\infty$. For every parameters $\param$ and inputs $x,x'$:
    \[
    \|R_\param(x)-R_{\param}(x')\|_r\leq \|\Phi(\param)\|_{1,r}\|x-x'\|_\infty.
    \]
\end{lemma}
For $r=1$, we simply have $\|\Phi(\param)\|_{1,r}=\|\Phi(\param)\|_1$. For this {\em specific value of $r$}, and in the {\em specific case} of  layered fully-connected neural networks without biases, the conclusion of \Cref{lem:LipsX} is already mentioned in \citet[before Section 3.4]{Neyshabur17ImplictRegularizationInDL}, and proven in \citet[Theorem 5]{PhDFurusho20Lipschitz}. \Cref{lem:LipsX} extends it to arbitrary DAG ReLU networks, and to arbitrary $r\in(0,\infty]$. This requires the introduction of {\em mixed} path-norms, which, to the best of our knowledge, have not been considered before in the literature.

\begin{proof}[Proof of \Cref{lem:LipsX}]
Consider parameters $\param$. Consider inputs $x,x'$ with the same path-activations with respect to $\param$: $\activation(\param,x)=\activation(\param,x')$. For $0<r<\infty$:
\begin{align*}
    \|R_\param(x) - R_\param(x')\|_{r}^r & \underset{\textrm{\Cref{lem:ForwardWithPhi}}}{=} \sum_{v\in\NeuronsOut} \left|\inner{\Phito{v}(\param)}{\ActTo{v}(\param,x) \left(\begin{array}{c}
     x \\
     1 
\end{array}\right) - \ActTo{v}(\param,x') \left(\begin{array}{c}
     x' \\
     1 
\end{array}\right)}\right|^r \nonumber\\
& \underset{\textrm{Hölder}}{\leq} \sum_{v\in\NeuronsOut} \|\Phito{v}(\param)\|_1^r \left\|\ActTo{v}(\param,x) \left(\begin{array}{c}
     x \\
     1 
\end{array}\right) - \ActTo{v}(\param,x') \left(\begin{array}{c}
     x' \\
     1 
\end{array}\right)\right\|_\infty^r \nonumber\\
& \underset{\activation(\param,x)=\activation(\param,x')}{\leq} \sum_{v\in\NeuronsOut} \|\Phito{v}(\param)\|_1^r \left\|\ActTo{v}(\param,x) \left(\left(\begin{array}{c}
     x \\
     1 
\end{array}\right) -\left(\begin{array}{c}
     x' \\
     1 
\end{array}\right)\right)\right\|_\infty^r \nonumber\\
& \underset{\|\ActTo{v}(\param,x)y\|_\infty\leq \|y\|_\infty}{\leq} \sum_{v\in\NeuronsOut} \|\Phito{v}(\param)\|_1^r \|x - x'\|_\infty^r \\
& = \|\Phi(\param)\|_{1,r}^r \|x - x'\|_\infty^r. \label{eq:Lips}
\end{align*}
When $0<r<\infty$, we just proved the claim locally on each region where the path-activations $\activation(\param,\cdot)$ are constant. This stays true on the boundary of these regions by continuity. It then expands to the whole domain by triangular inequality because on any segment joining any pair of inputs, there is a finite number of times when the region changes. The proof can be easily adapted to $r=\infty$.
\end{proof}

Another straightforward but important consequence of \Cref{lem:ForwardWithPhi} is that mixed path-norms can be computed in a single forward pass, up to replacing $*$-max-pooling neurons with linear ones.

\begin{theorem}\label{thm:ComputePathNorm}
    Consider an architecture $G=(\NeuronSet,E,(\rho_v)_{v\in\NeuronSet\setminus\NeuronsIn})$ as in \Cref{def:NN}. Consider the architecture $\tilde{G}:=(\NeuronSet,E,(\tilde{\rho}_v)_{v\in\NeuronSet\setminus\NeuronsIn})$ with $\tilde{\rho}_v:=\id$ if $v\in\NeuronSet_{\pool}$, and $\tilde{\rho}_v:=\rho_v$ otherwise. Consider $q\in(0,\infty)$, $r\in(0,\infty]$ and arbitrary parameters $\param\in\R^G=\R^{\tilde{G}}$. For a vector $\alpha$, denote $|\alpha|^q$ the vector deduced from $\alpha$ by applying $x\mapsto |x|^q$ coordinate-wise. Denote by ${\bf{1}}$ the input full of ones. It holds:
    \begin{equation}\label{eq:ComputePathNorm}
        \|\Phi(\param)\|_{q,r} = \| \ |\realization^{\tilde{G}}_{|{\param}|^q}({\bf{1}})|^{1/q}\ \|_r.
    \end{equation}
    Moreover, the formula is false in general if the $*$-max-pooling neurons have not been replaced with identity ones (\ie if the forward pass is done on $G$ rather than $\tilde{G}$).
\end{theorem}
\begin{figure}
    \centering
    \includegraphics[scale=0.4]{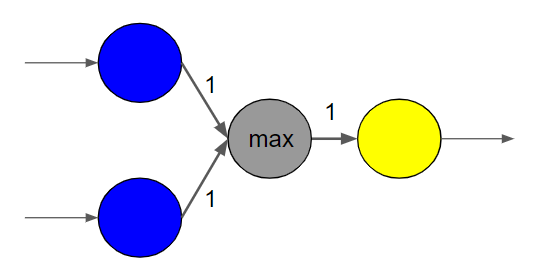}
    \caption{Example of a network where one must replace the max-pooling neuron to compute the path-norm with a single forward pass as in \Cref{eq:ComputePathNorm}.}
    \label{fig:CounterExampleComputePathNorm}
\end{figure}

\begin{proof}[Proof of \Cref{thm:ComputePathNorm}]
\Cref{fig:CounterExampleComputePathNorm} shows that \Cref{eq:ComputePathNorm}
is false if the $*$-max-pooling neurons have not been replaced with identity ones as the forward pass with input ${\bf{1}}$ yields output $1$ while the $L^1$ path-norm is $2$. 

We now establish \Cref{eq:ComputePathNorm}. By continuity in $\param$ of both sides of \Cref{eq:ComputePathNorm}, it is sufficient to prove it when every coordinate of $\param$ in $E$ is nonzero. Note that by \Cref{def:PhiActivation}, the path-lifting $\Phi^{G}$ and $\Phi^{\tilde{G}}$ associated respectively with $G$ and $\tilde{G}$ are the same since $G$ and $\tilde{G}$ have the same underlying DAG. Denote by $\Phi=\Phi^{G}=\Phi^{\tilde{G}}$ the common path-lifting, and by $a^{\tilde{G}}$ the path-activations associated with ${\tilde{G}}$. Denote by $\pathsto{v}$ the set of paths ending at a given neuron $v$ of $\tilde{G}$. According to \Cref{lem:ForwardWithPhi}, it holds for every output neuron $v$ of $\tilde{G}$ with $x={\bf{1}}$, using the convention $x_u:=1$ if $u\notin\NeuronsIn$
\[
    (R^{\tilde{G}}_{|{\param}|^q}({\bf{1}}))_v = \sum_{p\in\pathsto{v}} \Phi_p(|{\param}|^q) a^{\tilde{G}}_p(|{\param}|^q,x)x_{p_0} = \sum_{p\in\pathsto{v}} \Phi_p(|{\param}|^q) a^{\tilde{G}}_p(|{\param}|^q,{\bf{1}}).
\]
Since we restricted to $\param_{E}$ with nonzero coordinates, $|{\param}_{E}|^q$ has positive coordinates. As $\tilde{G}$ has only identity or ReLU neurons, it follows by a simple induction on the neurons that $u(|{\param}|^q,{\bf{1}}) > 0$ for every neuron $u$ of $\tilde{G}$. Thus, every path $p\in\paths^{\tilde{G}}$ is active, \ie, $a^{\tilde{G}}_p(|\param|^q,{\bf{1}})=1$. Since by \Cref{def:PhiActivation}, $\Phi_p(|{\param}|^q) = |\Phi_p({\param})|^q$, we obtain:
\[
(R^{\tilde{G}}_{|{\param}|^q}({\bf{1}}))_v = \sum_{p\in\pathsto{v}}  |\Phi_p(\param)|^q = \|\Phito{v}(\param)\|_q^q.
\]
The claim follows by \Cref{def:MixedPathNorm}.
\end{proof}

\section{Normalized parameters}

In the sequel it will be useful to restrict the analysis to {\em normalized} parameters.
\begin{definition}\label{def:NormalizedParameters}
    Consider $0< q\leq \infty$. Parameters $\param$ are said {\em $q$-normalized} if for every $v \in \NeuronSet \setminus (\NeuronsOut\cup\NeuronsIn)$:
    \[
        \|\Phito{v}(\param)\|_q = \|\left(\begin{array}{c}
     \paramto{v} \\
     {\bias}_v 
\end{array}\right)\|_q \in \{0, 1\},
    \]
    and if this is $0$ then it also holds $\param^{v\orange{\to}}=0$ (\orange{outgoing} edges of $v$).
\end{definition}

Thanks to the rescaling-invariance of ReLU neural network parameterizations, \Cref{alg:NormalizeAlgo} allows to rescale {\em any} parameters $\param$ into a normalized version $\tilde{\param}$ such that $R_{\tilde{\param}} = R_\param$ and $\Phi(\param) = \Phi(\tilde{\param})$ (\Cref{lem:AlgNormalization}). The rescaling-invariance are due to the positive-homogeneity of every activation function $\rho\in\{\id,\relu,\kpool\}$: $\rho(x) = \frac{1}{\lambda}\rho(\lambda x)$ for every $\lambda>0$ and $x\in\R$. 

\textbf{\Cref{alg:NormalizeAlgo}} The first line of the algorithm considers a topological sorting of the neurons \cite[Section 22.4]{Cormen}, \ie an order on the neurons such that if $u\to v$ is an edge then $u$ comes before $v$ in this ordering. Such an order always exists (and it can be computed in linear time). Moreover, note that a classical max-pooling neuron $v$ (corresponding to a $k$-max-pooling neuron with $k=1$, constant incoming weights all equal to one, and biases set to zero) has not anymore its incoming weights equal to one after normalization, in general. This has no incidence on the validity of the generalization bound on classical max-pooling neurons: rescaling is only used in the proof to reduce to another representation of the parameters that realize the same function and that is more handy to work with.
\begin{algorithm}
\caption{Normalization of parameters for norm $q\in(0,\infty]$}\label{alg:NormalizeAlgo}
\begin{algorithmic}[1]
\State Consider a topological sorting $v_1,\dots, v_k$ of the neurons
\For{$v=v_1,\dots,v_k$}
    \If{$v\notin\NeuronsIn\cup\NeuronsOut$}
    \State $\lambda_v \gets \left\|\left(\begin{array}{c}
    \paramto{v} \\
    \bias_v 
\end{array}\right)\right\|_q$
    \If{$\lambda_v=0$}
        \State $\paramfrom{v} \gets 0$
    \Else
        \State $\left(\begin{array}{c}
    \paramto{v} \\
    \bias_v 
\end{array}\right)\gets \frac{1}{\lambda_v} \left(\begin{array}{c}
    \paramto{v} \\
    \bias_v 
\end{array}\right)$ \Comment{normalize incoming weights and bias}
        \State $\paramfrom{v} \gets  \lambda_v \times \paramfrom{v}$ \Comment{rescale outgoing weights to preserve the function $\realization_\param$}
    \EndIf
    \EndIf
\EndFor
\end{algorithmic}
\end{algorithm}

\begin{lemma}\label{lem:AlgNormalization}
    Consider $q\in(0,\infty]$. If $\param$ is the output of \Cref{alg:NormalizeAlgo} for the input $\tilde{\param}$ then $\param$ is $q$-normalized (\Cref{def:NormalizedParameters}), $R_{\param} =R_{\tilde{\param}}$  and $\Phi(\param)=\Phi(\tilde{\param})$.
\end{lemma}

\begin{proof}
\textbf{Step 1.} We first prove that for every $v\in\NeuronSet\setminus(\NeuronsIn\cup\NeuronsOut)$:
\begin{equation}\label{eq:Normalization1}
    \left\|\left(\begin{array}{c}
    \paramto{v} \\
    \bias_v 
\end{array}\right)\right\|_q \in\{0,1\}.
\end{equation}
Consider $v\notin\NeuronsIn\cup\NeuronsOut$. Right after $v$ has been processed by the for loop of \Cref{alg:NormalizeAlgo} (line 2), it is clear that \Cref{eq:Normalization1} holds true for $v$. It remains to see that $\bias_v$ and $\paramto{v}$ are not modified until the end of \Cref{alg:NormalizeAlgo}. A bias can only be modified when this is the turn of the corresponding neuron, so $\bias_v$ is untouched after the iteration corresponding to $v$. And $\paramto{v}$ can only be modified when treating antecedents of $v$, but since antecedents must be before $v$ in any topological sorting, they cannot be processed after $v$. This proves that $\left\|\left(\begin{array}{c}
    \paramto{v} \\
    \bias_v 
\end{array}\right)\right\|_q \in\{0,1\}$.

\textbf{Step 2.} We now prove that if $\left\|\left(\begin{array}{c}
    \paramto{v} \\
    \bias_v 
\end{array}\right)\right\|_q =0$ then it also holds $\paramfrom{v}=0$ ({\em outgoing} edges of $v$). Indeed, $\left\|\left(\begin{array}{c}
    \paramto{v} \\
    \bias_v 
\end{array}\right)\right\|_q=0$ implies $\lambda_v=0$ in \Cref{alg:NormalizeAlgo}, which implies that $\paramfrom{v}=0$ right after rescaling (line 6) at the end of the iteration corresponding to $v$. Because the next iterations can only involve multiplication of the coordinates of $\paramfrom{v}$ by scalars, $\paramfrom{v}=0$ is also satisfied at the end of the algorithm. This proves the claim.

\textbf{Step 3.} In order to prove that $\param$ is $q$-normalized, it only remains to establish that $\|\Phito{v}(\param)\|_q = \|\left(\begin{array}{c}
     \paramto{v} \\
     {\bias}_v 
\end{array}\right)\|_q$. We prove this by induction on a topological sorting of the neurons.

A useful fact for the induction is that for every $v\notin\NeuronsIn$:
\begin{equation}\label{eq:PhiToVInduction}
        \Phito{v}(\param) = \left(\begin{array}{c}
         (\Phito{u}(\param)\paramfromto{u}{v})_{u\in\ant(v)} \\
         \bias_v
    \end{array}\right)
\end{equation}
where we recall that $\Phito{u}(\cdot)=1$ for input neurons $u$. \Cref{eq:PhiToVInduction} holds because $\Phito{v}$ is the path-lifting of $G^{\to v}$ (see \Cref{def:subgraphToV}), and the only paths in $G^{\to v}$ are $p=v$, and the paths going through antecedents of $v$ ($v$ has antecedents since it is not an input neuron).

Let's now start the induction. By definition, the first neuron $v\notin\NeuronsIn$ in a topological sorting has only input neurons as antecedents. Therefore, $\Phito{u}(\param)=1$ for every $u\in\ant(v)$. Using \Cref{eq:PhiToVInduction}, we get
\[
        \|\Phito{v}(\param)\|_q^q = |\bias_v|^q + \sum_{u\in\ant(v)}  |\paramfromto{u}{v}|^q  =\left\|\left(\begin{array}{c}
    \paramto{v} \\
    \bias_v 
\end{array}\right)\right\|_q^q
\]
This shows the claim for $v$. 

Assume the result to be true for $v\notin\NeuronsIn$ and all the neurons before $v$ in the considered topological order (in particular, for every $u\in\ant(u)$). By \Cref{eq:PhiToVInduction}, we have
\[
        \|\Phito{v}(\param)\|_q^q = |\bias_v|^q + \sum_{u\in\ant(v)}  \|\Phito{u}(\param)\|_q^q|\paramfromto{u}{v}|^q.
\]
The induction hypothesis guarantees that $c_u := \|\Phito{u}(\param)\|_q=\|\left(\begin{array}{c}
     \paramto{u} \\
     {\bias}_u 
\end{array}\right)\|_q$ for every $u\in\ant(v)$. By \Cref{eq:Normalization1}, we also have $c_u \in \{0,1\}$ for every $u \in \ant(v)$. 
When $c_u=1$, it clearly holds $\|\Phito{u}(\param)\|_q^q |\paramfromto{u}{v}|^q = |\paramfromto{u}{v}|^q$. Otherwise, $c_u=0$, hence $\paramfrom{u}=0$ as proved above, and we also obtain $\|\Phito{u}(\param)\|_q^q |\paramfromto{u}{v}|^q = |\paramfromto{u}{v}|^q$. We deduce that
    \begin{align*}
        \|\Phito{v}(\param)\|_q^q & = |\bias_v|^q + \sum_{u\in\ant(v)} |\paramfromto{u}{v}|^q =\left\|\left(\begin{array}{c}
    \paramto{v} \\
    \bias_v 
\end{array}\right)\right\|_q^q.
    \end{align*}
This concludes the induction.

\textbf{Step 4.} Each of the activation functions $\rho\in\{\id,\relu,\kpool\}$ is positive homogeneous, so classical arguments for layered fully-connected networsk (see e.g. \cite[Lemma 1 and Theorem 1]{Stock22Embedding}) are easily adapted to show that  $R_\param$ and $\Phi(\param)$ are both unchanged at each iteration of \Cref{alg:NormalizeAlgo} where $\lambda_v \neq 0$. Inspecting the iterations involving a neuron such that $\lambda_v=0$ reveals that: (i) $\Phi_p(\param)=0$ is unchanged for all paths $p$ going through this neuron; (ii) $\Phi_p(\param)$ is also unchanged for other paths since they only involve entries of $\param$ that are kept identical. As a result, $\Phi(\param)$ is also preserved at each such iteration. Finally, for each of the considered activations, $\lambda_v=0$ implies that $v(\param,x)=0$ for any $x$ (\Cref{def:Neurons}), so setting the outgoing weights $\paramfrom{v}=0$ does not change $R_\param$. This completes the proof.
\end{proof}

\section{Path-norms and products of operator norms}\label{app:PhiVsProducts}
For $0<q\leq\infty$, recall that the mixed (quasi-)norm $\|M\|_{q,\infty}$ of a matrix $M$ is:
\[
    \|M\|_{q,\infty} = \max_{\textrm{row of $M$}} \|\textrm{row}\|_q.
\]
{\em For $q=1$}, this is the operator norm of $M$ induced by the $L^\infty$-norm on the input and output spaces. As a consequence, this is the Lipschitz constant of $x\mapsto Mx+b$ with respect to these norms, and by composition, it is well-known (see, \eg, \cite{Combettes20Lipschitz}, where tighter bounds are also provided) that it can be used to derive a bound on the Lipschitz constant of layered fully-connected ReLU networks.

\begin{lemma}[see, \eg, {\cite{Combettes20Lipschitz}}]
    For a layered fully-connected ReLU network with biases of the form
\[
R_\param(x) = M_L\relu(M_{L-1}\dots \relu(M_1x + b_1) + b_{L-1}) + b_L,
\]
it holds for every $x,x'$ and {\em for $q=1$}:
\[
    \|R_{\param}(x) - R_{\param}(x')\|_\infty \leq \left(\prod_{\ell=1}^L \|M_\ell\|_{q,\infty}\right) \|x-x'\|_\infty.
\]
\end{lemma}
\begin{proof}
    As this is well-known we only reproduce the proof for $L=2$, and the general case easily follows by an induction: (highlighting in \orange{orange} $x$ and $x'$):
    \begin{align*}
        \left\|R_{\param}(x) - R_{\param}(x')\right\|_\infty & = \left\|M_2\relu(M_1\orange{x} + b_1) + b_2 - M_2\relu(M_1\orange{x'} + b_1) + b_2\right\|_\infty \\
        & = \left\|M_2\left(\relu(M_1\orange{x} + b_1) - \relu(M_1\orange{x'} + b_1)\right)\right\|_\infty \\
        & \leq \|M_2\|_{1,\infty}\left\|\relu(M_1\orange{x} + b_1) - \relu(M_1\orange{x'} + b_1)\right\|_\infty \\
        & \leq \|M_2\|_{1,\infty}\left\|M_1\orange{x} + b_1 - M_1\orange{x'} + b_1\right\|_\infty && \textrm{$\relu$ is 1-Lipschitz}\\
        & \leq \|M_2\|_{1,\infty}\left\|M_1\left(\orange{x} -\orange{x'}\right)\right\|_\infty \\
        & \leq \|M_2\|_{1,\infty}\|M_1\|_{1,\infty}\left\|\orange{x} -\orange{x'}\right\|_\infty.
    \end{align*}
\end{proof}
Another bound on the Lipchitz constant with respect to the $L^\infty$-norm on both the input and the output is the mixed path-norm $\|\Phi(\param)\|_{1,\infty}$ (\Cref{lem:LipsX}), which is even valid for an arbitrary DAG ReLU network. Is there any relation between the mixed path-norm $\|\Phi(\param)\|_{1,\infty}$ and the product of operator norm $\prod_{\ell=1}^L \|M_\ell\|_{1,\infty}$ for layered fully-connected networks (LFCNs)? 

\textbf{Comparing the Lipschitz constants for a scalar-valued LFCN.} It is known that for a {\em scalar-valued} (i.e., with $\dout=1$) LFCN without biases, corresponding to $R_\param(x) = M_L\relu(M_{L-1}\dots \relu(M_1x))$, the bound $\|\Phi(\param)\|_{q} \leq \prod_{\ell=1}^L \|M_\ell\|_{q,\infty}$ holds \cite[Theorem 5]{Neyshabur15NormBasedControls} for every $0<q\leq \infty$, with equality if the parameters have been rescaled properly. In the specific case of a scalar-valued DAG ReLU network, it holds $\|\Phi(\param)\|_{q} = \|\Phi(\param)\|_{q,\infty}$ (\Cref{def:MixedPathNorm}). Thus, the result from \citet{Neyshabur15NormBasedControls} shows that the mixed path-norm $\|\Phi(\param)\|_{1,\infty}$ is a tighter Lipschitz bound than the product of operator norms $\prod_{\ell=1}^L \|M_\ell\|_{1,\infty}$ in the specific case of {\em scalar-valued} LFCN.

We now extend the comparison made in \cite{Neyshabur15NormBasedControls} from a {\em scalar-valued LFCN} without biases to an {\em arbitrary DAG ReLU network}. We will again derive that the mixed path-norm is a lower bound on (an extension of) the product of layers' norm for a general {\em DAG ReLU network}. Compared to \cite{Neyshabur15NormBasedControls}, this requires both our extension of the notion of path-norm to mixed path-norms, and extending the notion of product of layers' norm to deal with a DAG. 

{\em From now on, we always consider $q\in(0,\infty)$ and $r\in(0,\infty]$ if not specified otherwise.}

\textbf{From scalar-valued to vector-valued networks: from standard to mixed $L^q$-norms.} \Cref{lem:LipsX} shows that going from a {\em scalar-valued} network (LFCN or, more generally a DAG) to a {\em vector-valued} network requires going from standard $L^q$-norms $\|\Phi(\param)\|_q$, as considered in \citet{Neyshabur15NormBasedControls}, to mixed $L^q$-norms $\|\Phi(\param)\|_{q,r}$, which is, to the best of our knowledge, first introduced in \Cref{def:MixedPathNorm}.

\textbf{From LFCN to DAG: extending the product of layers' norm.} 
There is no notion a of a ``layer'' $M_\ell$ in a general DAG, hence the product $\prod_{\ell=1}^L \|M_\ell\|_{q,\infty}$ makes no sense anymore in this general context. We now introduce a quantity that generalizes this product for a general DAG. 
\begin{definition}\label{def:DAGProductOpNorms}
     Consider $q\in(0,\infty)$. For practical purposes, we denote 
     $\biascompleted_v:=\bias_v$ when $v\in\NeuronSet\setminus\NeuronsIn$, and $\biascompleted_v:=1$ if $v\in\NeuronsIn$. For every path $p\in\paths$, consider
\begin{equation}\label{eq:defPi}
        \pathProduct{p}{q}{\param} := \left(\sum_{\ell=0}^{\length{p}}\left|\biascompleted_{p_\ell}\right|^q \prod_{k=\ell+1}^{\length{p}} \|\paramto{p_k}\|_q^q\right)^{1/q},
\end{equation}

with the convention that an empty product is equal to one. For $q\in(0,\infty)$ and $r\in(0,\infty]$, define:
\begin{equation}
\Pi_{q,r}(\param) := \left\|\left(\max_{p\in\pathsto{v}} \pathProduct{p}{q}{\param}\right)_{v\in\NeuronsOut}\right\|_r.\label{eq:defPiGlobal}
\end{equation}
\end{definition}
Let us check that $\Pi_{q,\infty}$ generalizes the product of layers' norm to an arbitrary DAG.
\begin{lemma}
    Consider $q\in(0,\infty)$. In a LFCN $R_\param(x) = M_L\relu(M_{L-1}\dots \relu(M_1x))$ with $L$ layers and without biases, it holds:
    \[
    \Pi_{q,\infty}(\param) = \prod_{\ell=1}^{L} \|M_\ell\|_{q,\infty}.
\]
\end{lemma}
\begin{proof}
    In a LFCN, all the neurons of two consecutive layers are connected, so $\prod_{\ell=1}^L \|M_\ell\|_{q,\infty}$ is also the maximum over all paths starting from an input neuron (and ending at an output neuron, by definition of $\paths$) of the product of $L^q$-norms: 
    \[
    \prod_{\ell=1}^L \|M_\ell\|_{q,\infty} = \max_{p\in\paths, p_0\in\NeuronsIn} \prod_{\ell=1}^{\length{p}} \|\paramto{p_\ell}\|_q. 
    \]
    Since there are no biases (corresponding to $\gamma_u=b_u=0$ for every $u\notin\NeuronsIn$) we have
\[
    \pathProduct{p}{q}{\param}= \left\{\begin{array}{cc}
        \prod_{\ell=1}^{\length{p}} \|\paramto{p_\ell}\|_q & \textrm{if $p_0\in\NeuronsIn$}, \\
        0 & \textrm{otherwise}.
    \end{array}\right.
\]
This shows the claim:
\begin{align*}
    \Pi_{q,\infty}(\param) & = \max_{v \in \NeuronsOut}  \max_{p \in \pathsto{v}} \pathProduct{p}{q}{\param} = \max_{p\in\paths} \pathProduct{p}{q}{\param} \\
    & = 
\max_{p\in\paths, p_0 \in \NeuronsIn} \prod_{\ell=1}^{\length{p}} \|\paramto{p_\ell}\|_q = \prod_{\ell=1}^{\length{p}} \|M_\ell\|_{q,\infty}.\qedhere
\end{align*}
\end{proof}

The next theorem proves that the mixed path-norm is a lower bound on this extended product of layers' norm. In particular, this shows that the Lipschitz bound given by $\|\Phi(\param)\|_{1,\infty}$ (which is computable in a single forward-pass by \Cref{thm:ComputePathNorm}) is always at least as good as the one given by $\prod_{\ell=1}^L \|M_\ell\|_{1,\infty}$ for a LFCN without biases. This extends as claimed the result from \cite{Neyshabur15NormBasedControls} to the vector-valued case.
\begin{theorem}\label{thm:Phi=MinRescalings}
    Consider $q\in(0,\infty)$ and $r\in(0,\infty]$. For every DAG ReLU network (\Cref{def:NN}) and every parameters $\param$:
\[        \|\Phi(\param)\|_{q,r} \leq \Pi_{q,r}(\param).
\]

    If $\param$ is $q$-normalized (\Cref{def:NormalizedParameters}) then:
    \[
    \|\Phi(\param)\|_{q,r} = \Pi_{q,r}(\param) = \left\|\left(\left\|\left(\begin{array}{c}
    \paramto{v} \\
    \bias_v 
\end{array}\right)\right\|_q\right)_{v\in\NeuronsOut}\right\|_r.
    \]
\end{theorem}

\begin{figure}[ht]
    \centering
    \includegraphics[width=\textwidth]{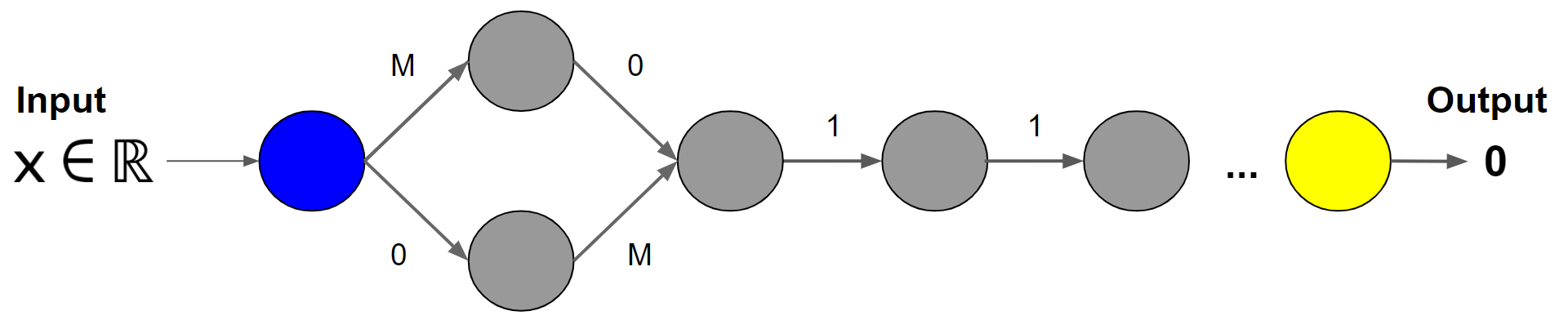}
    \caption{A network which path-norm is zero while the product of operator norms scales as $M^2$.}
    \label{fig:PhiNormVsProdNorm}
\end{figure}

\begin{proof}[Proof of \Cref{thm:Phi=MinRescalings}]
First, when $\param$ is $q$-normalized, by \Cref{def:NormalizedParameters} and \Cref{eq:MixedPathNorm}:
\[
\|\Phi(\param)\|_{q,r} = \left\|\left(\left\|\left(\begin{array}{c}
    \paramto{v} \\
    \bias_v 
\end{array}\right)\right\|_q\right)_{v\in\NeuronsOut}\right\|_r.
\]
Let us prove by induction on a topological sorting of the neurons that for every $v\in\NeuronSet$:
\begin{equation}\label{eq:Phito<=Pi}
    \|\Phito{v}(\param)\|_q\leq \max_{p\in\pathsto{v}} \pathProduct{p}{q}{\param}
\end{equation}
with equality if $\param$ is $q$-normalized. As a direct consequence of~\Cref{eq:MixedPathNorm,eq:defPiGlobal}, this will prove that $\|\Phi(\param)\|_{q,r} \leq \Pi_{q,r}(\param)$, with equality when $\param$ is $q$-normalized, yielding all the claimed results.

We start with $v\in\NeuronsIn$ since input neurons are the first to appear in a topological sorting. In this case, the only path in $\pathsto{v}$ is $p=v$, for which it holds $\pathProduct{p}{q}{\param}=|\gamma_{p_0}|=|\gamma_{v}|=1$ by \Cref{def:DAGProductOpNorms}. Moreover, we also have $\Phito{v}(\param)=1$ (\Cref{def:PhiActivation}). This proves \Cref{eq:Phito<=Pi} and the case of equality for input neurons even if $\param$ is not normalized.

Now, consider $v\notin\NeuronsIn$ and assume \Cref{eq:Phito<=Pi}, and the case of equality for $q$-normalized parameters, to be true for every antecedent of $v$. In this case, we have $\Phito{v}(\param) = \left(\begin{array}{c}
         (\Phito{u}(\param)\paramfromto{u}{v})_{u\in\ant(v)} \\
         \bias_v 
    \end{array}\right)$
    so it holds
    \[
    \|\Phito{v}(\param)\|_q^q = |\bias_v|^q + \sum_{u\in\ant(v)} \|\Phito{u}(\param)\|_q^q|\paramfromto{u}{v}|^q \leq |\bias_v|^q + \|\paramto{v}\|_q^q\max_{u\in\ant(v)}\|\Phito{u}(\param)\|_q^q.
    \]
Using the induction hypothesis on every $u\in\ant(v)$ yields:
\begin{align*}
   \|\Phito{v}(\param)\|_q^q & \leq |\bias_v|^q + \|\paramto{v}\|_q^q\max_{u\in\ant(v)}\max_{p\in\pathsto{u}} (\pathProduct{p}{q}{\param})^{q}
\end{align*}
Consider $u\in\ant(v)$ and $p\in\pathsto{u}$, and denote by $\tilde{p}=p\to v$ the path $p$ concatenated with the edge $u\to v$. By \Cref{def:DAGProductOpNorms}, we have (highlighting in \orange{orange} the important changes)
\begin{align*}
   & |\bias_v|^q + \|\paramto{v}\|_q^q(\pathProduct{p}{q}{\param})^q  \\
   & = \underbrace{|\bias_v|^q}_{=|\orange{\gamma}_v|^q\textrm{since $v\notin\NeuronsIn$}} + \|\paramto{v}\|_q^q \sum_{\ell=0}^{\length{p}}\left|\biascompleted_{p_\ell}\right|^q \prod_{k=\ell+1}^{\length{p}} \|\paramto{p_k}\|_q^q \\
   & = |\orange{\gamma}_{\orange{\pathend{\tilde{p}}}}|^q + \|\paramto{\orange{\pathend{\tilde{p}}}}\|_q^q \sum_{\ell=0}^{\orange{\length{\tilde{p}}-1}}\left|\biascompleted_{\orange{\tilde{p}}_\ell}\right|^q \prod_{k=\ell+1}^{\orange{\length{\tilde{p}}-1}} \|\paramto{\orange{\tilde{p}}_k}\|_q^q
   & = (\pathProduct{\orange{\tilde{p}}}{q}{\param})^q
\end{align*}
We deduce that
\[
\|\Phito{v}(\param)\|_q^q \leq \max_{u\in\ant(v)}\max_{p\in\pathsto{u}} 
(\pathProduct{p\to u}{q}{\param})^q = \max_{\tilde{p}\in\pathsto{v}} (\pathProduct{\tilde{p}}{q}{\param})^q.
\]
This proves \Cref{eq:Phito<=Pi} for $v$. To conclude the induction, it remains to treat the equality case for $v$ assuming that $\param$ is $q$-normalized. Since $v \notin \NeuronsIn$, $\ant(v) \neq \emptyset$, and since each neuron $u \in \ant(v)$ cannot be an output neuron, the fact that $\param$ is $q$-normalized implies by \Cref{def:NormalizedParameters} that $\|\Phito{u}(\param)\|_q\in\{0,1\}$, with $\paramfrom{u}=0$ as soon as $\|\Phito{u}(\param)\|_q=0$. This implies
\begin{align*}
    \|\Phito{u}(\param)\|_q^q|\paramfromto{u}{v}|^q &= |\paramfromto{u}{v}|^q
\end{align*}
As a result 
\begin{align*}
\sum_{u\in\ant(v)} \|\Phito{u}(\param)\|_q^q|\paramfromto{u}{v}|^q = \sum_{u\in\ant(v)} |\paramfromto{u}{v}|^q = \|\paramto{v}\|_q^q
= \|\paramto{v}\|_q^q \max_{u\in\ant(v)} \|\Phito{u}(\param)\|_q^q.
\end{align*}
By the induction hypothesis, we also have
\[
     \max_{u\in\ant(v)} \|\Phito{u}(\param)\|_q^q = \max_{u\in\ant(v)} \max_{p\in\pathsto{u}} (\pathProduct{p}{q}{\param})^q.
\]
Putting everything together yields
\begin{align*}
   \|\Phito{v}(\param)\|_q^q & = |\bias_v|^q + \sum_{u\in\ant(v)} \|\Phito{u}(\param)\|_q^q|\paramfromto{u}{v}|^q \\
   & = |\bias_v|^q + \|\paramto{v}\|_q^q \max_{u\in\ant(v)} \max_{p\in\pathsto{u}} (\pathProduct{p}{q}{\param})^q \\
   & = \max_{p\in\pathsto{v}} \pathProduct{p}{q}{\param}^q.
\end{align*}
This shows the case of equality and concludes the induction.
\end{proof}

\section{Relevant (and apparently new) contraction lemmas}\label{app:Contraction}
The main result is \Cref{lem:ContractionLemma}.
\begin{lemma}\label{lem:ContractionLemma}
    Consider finite sets $I,W,Z$, and for each $z\in Z$, consider a set $T^z\subset(\R^W)^I$. We denote $t=(t_i)_{i \in I}
    \in T^z$ with $t_i = (t_{i,w})_{w \in W}\in \R^{W}$. Consider functions $f_{i,z}:\R^W\to\R$ and a finite family $\eps=(\eps_j)_{j\in J}$ of independent identically distributed Rademacher variables, with the index set $J$ that will be clear from the context. Finally, consider a convex and non-decreasing function $\gPeeling:\R\to\R$. Assume that at least one of the following setting holds.
    
    \textbf{Setting 1: scalar input case.}  $|W|=1$ and for every $i\in I$ and $z\in Z$, $f_{i,z}$ is $1$-Lipschitz with $f_{i,z}(0)=0$.

    \textbf{Setting 2: $*$-max-pooling case.} For every $i\in I$ and $z\in Z$, there is $k_{i,z}\in\Ns$ such that for every $t\in T^z$, $f_{i,z}(t)=t_{(k_{i,z})}$ is the $k_{i,z}$-th largest coordinate of $t$.

    Then we have:
    \begin{equation}\label{eq:ContractionLemma}
        \begin{aligned}
            \E \max_{z\in Z} \sup_{t\in T^z} \gPeeling\left(\sum_{i\in I} \eps_{i,z} f_{i,z}(t_i)\right) \leq \E \max_{z\in Z} \sup_{t\in T^z} \gPeeling\left(\sum_{i\in I, w\in W}
            \eps_{i,w,z} t_{i,w}\right).
        \end{aligned}
    \end{equation}
\end{lemma}

The scalar input case is a simple application of Theorem 4.12 in \citet{LedouxTalagrand91ProbaBanachSpaces}. Indeed, for $z\in Z$ and $t\in T^z$, define $s(z,t)\in \R^{I\times Z}$ to be the matrix with coordinates $(i,v)\in I\times Z$ given by $[s(z,t)]_{i,v}=t_i$ if $v=z$, $0$ otherwise. Define also $f_{i,v}=f_i$. Since $f_{i,v}(0)=0$, it holds: 
\[
\sum_{i\in I} \eps_{i,z} f_i(t_i) = \sum_{i\in I, v\in Z} \eps_{i,v} f_{i,v}([s(z,t)]_{i,v}).
\]
If $S:=\{s(z,t), z\in Z, t\in T^z\}$ and $J:=I\times Z$ then the result claimed in the scalar case reads
\[
\E \sup_{s\in S} \gPeeling\left(\sum_{j\in J} \eps_{j} f_{j}(s_j)\right) \leq \E \sup_{s\in S} \gPeeling\left(\sum_{j\in J}
            \eps_{j} s_j\right).
\]
The latter is true by Theorem 4.12 of \citet{LedouxTalagrand91ProbaBanachSpaces}. However, we present an additional proof below, which we employ to establish the new scenario involving $*$-max-pooling. This alternative proof closely follows the structure of the proof outlined in Theorem 4.12 of \citet{LedouxTalagrand91ProbaBanachSpaces}: the beginning of the proof is the same for the scalar case and the $*$-max-pooling case, and then the arguments become specific to each case.

Note that Theorem 4.12 of \citet{LedouxTalagrand91ProbaBanachSpaces} does not apply for the $*$-max-pooling case because the $t_i$'s are now vectors. The most related result we could find is a vector-valued contraction inequality \citep{Maurer16VectorContraction} that is known in the specific case where $|Z|=1$, $\gPeeling$ is the identity, and for arbitrary $1$-Lipschitz functions $f_{i,z}$ such that $f_{i,z}(0)=0$ (with a different proof, and with a factor $\sqrt{2}$ on the right-hand side). Here, the vector-valued case we are interested in is $f_{i,z}=k_{i,z}\textrm{-}\mathtt{pool}$ and $\gPeeling=\exp$, which is covered by \Cref{lem:ContractionLemma}. We could not find it stated elsewhere.

In the proof of \Cref{lem:ContractionLemma}, we reduce to the simpler case where $|Z|=1$ and $|I|=1$ that corresponds to the next lemma. Again, the scalar input case is given by \citet[Equation (4.20)]{LedouxTalagrand91ProbaBanachSpaces} while the $*$-max-pooling case is apparently new.

\begin{lemma}\label{lem:InductionStepContraction} Consider a finite set $W$, a set $T$ of elements $t=(t_1,t_2)\in\R^W\times \R$ and a function $f:\R^W\to \R$. Consider also a convex non-decreasing function $F:\R\to\R$ and a family of iid Rademacher variables $(\eps_j)_{j\in J}$ where $J$ will be clear from the context. Assume that we are in one of the two following situations.

\textbf{Scalar input case.} $f$ is $1$-Lipschitz, satisfies $f(0)=0$ and has a scalar input ($|W|=1$).

\textbf{$*$-max-pooling case.} There is $k\in\Ns$ such that $f$ computes the $k$-th largest coordinate of its input.

Denoting $t_1=(t_{1,w})_{w\in W}$, it holds:
\[
    \E \sup_{t\in T} F\left(\eps_1 f(t_1) + t_2\right) \leq \E \sup_{t\in T} F\left(\sum_{w} \eps_{1,w} t_{1,w} + t_2\right).
\]
\end{lemma}

The proof of \Cref{lem:InductionStepContraction} is postponed. We now prove \Cref{lem:ContractionLemma}.
\begin{proof}[Proof of \Cref{lem:ContractionLemma}]
    First, because of the Lipschitz assumptions on the $f_i$'s and the convexity of $\gPeeling$, everything is measurable and the expectations are well defined.

    We prove the result by reducing to the simpler case of \Cref{lem:InductionStepContraction}. This is inspired by the reduction done in the proof of \citet[Theorem 4.12]{LedouxTalagrand91ProbaBanachSpaces} in the special case of scalar $t_i$'s ($|W|=1$).

    \textbf{Reduce to the case $|Z|=1$ by conditioning and iteration.} For $z\in Z$, define 
    \begin{align*}
        A_z & := \sup_{t\in T^z} \gPeeling\left(\sum_{i\in I} \eps_{i,z} f_{i,z}(t_i)\right), \\
        B_z & := \sup_{t\in T^z} \gPeeling\left(\sum_{i\in I, w\in W} \eps_{i,w,z} t_{i,w}\right).
    \end{align*}
    \Cref{lem:ReduceZ=1} applies since these random variables are independent. Thus, it is enough to prove that for every $c\in[-\infty,\infty)$:
    \[
        \E \max(A_z, c)  \leq \E \max(B_z, c).
    \]
    Define $F(x)=\max(\gPeeling(x),c)$. This can be rewritten as (inverting the supremum and the maximum)
    \begin{equation}\label{eq:ReducedV=1}
        \E \sup_{t\in T^z} F\left(\sum_{i\in I} \eps_{i,z} f_{i,z}(t_i)\right) \leq \E \sup_{t\in T^z} F\left(\sum_{i\in I, w\in W} \eps_{i,w,z} t_{i,w}\right).
    \end{equation}
    We just reduced to the case where there is a single $z$ to consider, up to the price of replacing $\gPeeling$ by $F$. Since $\gPeeling$ and $x\mapsto \max(x,c)$ are non-decreasing and convex, so is $F$ by composition. Alternatively, note that we could also have reduced to the case $|Z|=1$ by defining $S:=\{s(z,t), z\in Z, t\in T^z\}$ just as it is done right after the statement of \Cref{lem:ContractionLemma}. In order to apply \Cref{lem:InductionStepContraction}, it remains to reduce to the case $|I|=1$.
    
    \textbf{Reduce to the case $|I|=1$ by conditioning and iteration.} \Cref{lem:ReduceI=1} shows that in order to prove \Cref{eq:ReducedV=1}, it is enough to prove that for every $i\in I$ and every subset $R\subset \R^W\times R$, denoting $r=(r_1,r_2)\in\R^W\times\R$, it holds
    \[
        \E \sup_{r\in R} F\left(\eps_{i,z} f_{i,z}(r_1) + r_2\right) \leq \E \sup_{r\in R} F\left(\sum_{w\in W} \eps_{i,w,z} r_{1,w} + r_2\right).
    \]    
    We just reduced to the case $|I|=1$ since one can now consider the indices $i$ one by one. The latter inequality is now a direct consequence of \Cref{lem:InductionStepContraction}. This proves the result.
\end{proof}

\begin{lemma}\label{lem:ReduceZ=1}
    Consider a finite set $Z$ and independent families of independent real random variables $(A_z)_{z\in Z}$ and $(B_z)_{z\in Z}$. If for every $z\in Z$ and every constant $c\in[-\infty, \infty)$, it holds $\E \max(A_z,c) \leq \E \max(B_z,c)$, then
    \[
    \E \max_{z\in Z} A_z \leq \E \max_{z\in Z} B_z.
    \]
\end{lemma}
\begin{proof}[Proof of \Cref{lem:ReduceZ=1}]
    The proof is by conditioning and iteration. To prove the result, it is enough to prove that if
    \[
    \E \max_{z\in Z} A_z \leq \E \max\left(\max_{z\in Z_1} A_z, \max_{z\in Z_2} B_z\right)
    \]
    for some partition $Z_1,Z_2$ of $Z$, with $Z_2$ possibly empty for the initialization of the induction, then for every $z_0\in Z_1$:
    \[
    \E \max_{z\in Z} A_z \leq \E \max\left(\max_{z\in Z_1\setminus\{z_0\}} A_z, \max_{z\in Z_2\cup\{z_0\}} B_z\right),
    \]
    with the convention that the maximum over an empty set is $-\infty$. Indeed, the claim would then come directly by induction on the size of $Z_2$. 
    
    Now, consider an arbitrary partition $Z_1,Z_2$ of $Z$, with $Z_2$ possibly empty, and consider $z_0\in Z_1$. It is then enough to prove that
    \begin{equation}\label{eq:GoalV=1}
        \E \max\left(\max_{z\in Z_1} A_z, \max_{z\in Z_2} B_z\right) \leq \E \max\left(\max_{z\in Z_1\setminus\{z_0\}} A_z, \max_{z\in Z_2\cup\{z_0\}} B_z\right).
    \end{equation}
    Define the random variable $C= \max\left(\max_{z\in Z_1\setminus\{z_0\}} A_z, \max_{z\in Z_2} B_z\right)$ which may be equal to $-\infty$ when the maximum is over empty sets, and which is independent of $A_{z_0}$ and $B_{z_0}$. It holds:
    \[
    \max\left(\max_{z\in Z_1} A_z, \max_{z\in Z_2} B_z\right) = \max\left(A_{z_0}, C\right)
    \]
    and
    \[
    \max\left(\max_{z\in Z_1\setminus\{z_0\}} A_z, \max_{z\in Z_2\cup\{z_0\}} B_z\right) = \max\left(B_{z_0}, C\right).
    \]
    \Cref{eq:GoalV=1} is then equivalent to
    \[
        \E \max(A_{z_0}, C) \leq \E \max(B_{z_0},C)
    \]
    with $C$ independent of $A_{z_0}$ and $B_{z_0}$. For a constant $c\in[-\infty, \infty)$, denote $A(c) = \E \max(A_{z_0},c)$ and $B(c) = \E \max(B_{z_0},c)$. We have:
    \begin{align*}
        \E \max(A_{z_0}, C) & = \E\left(\E\left(\max\left(A_{z_0}, C\right)|A_{z_0}\right)\right) && \text{law of total expectation} \\
        & = \E A(C) && \text{independence of $C$ and $A_{z_0}$}.        
    \end{align*}
    and similarly $\E \max(B_{z_0},C) = \E B(C)$. It is then enough to prove that $A(C)\leq B(C)$ almost surely. Since $C\in[-\infty, \infty)$, this is true by assumption. This proves the claims.
\end{proof}

\begin{lemma}\label{lem:ReduceI=1}
    Consider finite sets $I,W$ and independent families of independent real random variables $(\eps_i)_{i\in I}$ and $(\eps_{i,w})_{i\in I, w\in W}$. Consider functions $f_i:\R^W\to \R$ and $F:\R\to\R$ that are continuous. Assume that for every $i\in I$ and every subset $R\subset\R^W\times \R$, denoting $r=(r_1,r_2)\in R$ with $r_1=(r_{1,w})_w\in\R^W$ and $r_2\in\R$ the components of $r$, it holds 
    \[
        \E \sup_{r\in R} F(\eps_i f_i(r_1) + r_2)\leq \E \sup_{r\in R} F(\sum_{w \in W} \eps_{i,w} r_{1,w} + r_2).
    \]
    Consider an arbitrary $T\subset (\R^W)^I$ and for $t=(t_i)_{i\in I}\in T$, denote $t_{i,w}$ the $w$-th coordinate of $t_i\in\R^W$. It holds:
    \[
    \E \sup_{t\in T} F(\sum_{i\in I} \eps_i f_i(t_i))\leq \E \sup_{t\in T} F(\sum_{i\in I, w \in W} \eps_{i,w} t_{i,w}).
    \]
\end{lemma}
\begin{proof}[Proof of \Cref{lem:ReduceI=1}]
    The continuity assumption on $F$ and the $f_i$'s is only used to make all the considered suprema measurable. The proof goes by conditioning and iteration. For any $J\subset I$, denote $\eps_J$ the family that contains both $(\eps_j)_{j\in J}$ and $(\eps_{j,w})_{j\in J, w\in W}$. Define
    \begin{align*}
        h_J(t, \eps_J) & := \sum_{j\in J} \eps_{j} f_j(t_j), \\
        H_J(t, \eps_J) & := \sum_{j\in J, w \in W} \eps_{j,w} t_{j,w},
    \end{align*}
    with the convention that an empty sum is zero. To make notations lighter, if $J=\{j\}$ then we may write $h_j$ and $H_j$ instead of $h_J$ and $H_J$. We also omit to write the dependence on $\eps_J$ as soon as possible. What we want to prove is thus equivalent to
    \[
    \E \sup_{t\in T} F(h_I(t)) \leq \E \sup_{t\in T} F(H_I(t)).
    \]
    It is enough to prove that for every partition $I_1,I_2$ of $I$, with $I_2$ possibly empty, if
    \[
        \E \sup_{t\in T} F(h_I(t)) \leq \E \sup_{t\in T} F(h_{I_1}(t) + H_{I_2}(t)),
    \]
    then for every $j\in I_1$,
    \[
        \E \sup_{t\in T} F(h_I(t)) \leq \E \sup_{t\in T} F(h_{I_1\setminus\{j\}}(t) + H_{I_2\cup\{j\}}(t)).
    \]
    Indeed, the result would then come by induction on the size of $I_2$. Fix an arbitrary partition $I_1,I_2$ of $I$ with $I_2$ possibly empty, and $j\in I_1$. It is then enough to prove that    
    \begin{equation}\label{eq:ContractionInductionI}
         \E \sup_{t\in T} F(h_{I_1}(t) + H_{I_2}(t)) \leq \E \sup_{t\in T} F(h_{I_1\setminus\{j\}}(t) + H_{I_2\cup\{j\}}(t)).
    \end{equation}
    Denote $\eps_{-j} := \eps_{I\setminus \{j\}}$ and $\phi(t,\eps_{-j}) := h_{I_1\setminus\{j\}}(t, \eps_{I_1\setminus\{j\}}) + H_{I_2}(t, \eps_{I_2}))$. It holds:
    \[
        h_{I_1}(t) + H_{I_2}(t) = h_j(t,\eps_j) + \phi(t,\eps_{-j})
    \]
    and, writing $\eps_{j,\cdot}=(\eps_{j,w})_{w\in W}$:
    \[
    h_{I_1\setminus\{j\}}(t) + H_{I_2\cup\{j\}}(t) = H_j(t,\eps_{j,\cdot}) + \phi(t,\eps_{-j}).
    \]
    
    Consider the measurable functions
    \[
    g(\eps_j,\eps_{-j}):=\sup_{t\in T} F(h_j(t,\eps_j) + \phi(t,\eps_{-j}))
    \]
    and 
    \[
    G(\eps_{j,\cdot},\eps_{-j}):=\sup_{t\in T} F(H_j(t,\eps_{j,\cdot}) + \phi(t,\eps_{-j})).
    \]
    Denote $\Delta$ the ambiant space of $\eps_{-j}$ and consider a constant $\delta\in \Delta$. Define $\hat{g}(\delta) = \E g(\eps_j,\delta)$ and $\hat{G}(\delta) = \E G(\eps_{j,\cdot}, \delta)$. It holds
    \begin{align*}
        \E \sup_{t\in T} F(h_{I_1}(t) + H_{I_2}(t)) & = \E g(\eps_j,\eps_{-j}) && \text{by definition of $g$} \\
        & = \E\left(\E\left(g(\eps_j,\eps_{-j})|\eps_{-j}\right)\right) && \text{law of total expectation} \\
        & = \E \hat{g}(\eps_{-j}) && \text{independence of $\eps_j$ and $\eps_{-j}$}
    \end{align*}
    and similarly $\E \sup_{t\in T} F(h_{I_1\setminus\{j\}}(t) + H_{I_2\cup\{j\}}(t)) = \E \hat{G}(\eps_{-j})$. Thus, \Cref{eq:ContractionInductionI} is equivalent to $\E \hat{g}(\eps_{-j})\leq \E \hat{G}(\eps_{-j})$. For every $\delta\in \Delta$, we can define $R(\delta)=\{(t_j, \phi(t,\delta))\in\R^W\times\R, t\in T\}$ and it holds
    \[
    \hat{g}(\delta) = \E \sup_{r\in R} F(\eps_j f_j(r_1) + r_2)
    \]
    and
    \[
    \hat{G}(\delta) = \E \sup_{r\in R} F(\sum_{w\in } \eps_{j,w} r_{1,w} + r_2).
    \]
    Thus, $\hat{g}(\delta)\leq \hat{G}(\delta)$ for every $\delta\in\Delta$ by assumption. This shows the claim.
\end{proof}
\begin{proof}[Proof of \Cref{lem:InductionStepContraction}] Recall that we want to prove
\begin{equation}\label{eq:GoalContractionProof}
    \E \sup_{t\in T} F\left(\eps_1 f(t_1) + t_2\right) \leq \E \sup_{t\in T} F\left(\sum_{w\in W} \eps_{1,w} t_{1,w} + t_2\right).
\end{equation}
    
    \textbf{Scalar input case.} In this case, $|W|=1$ \ie the inputs $t_1$ are scalar and the result is well-known, see \citet[Equation (4.20)]{LedouxTalagrand91ProbaBanachSpaces}.
    
    \textbf{$k$-max-pooling case.} In this case, $f$ computes the $k$-th largest coordinate of its input. Computing explicitly the expectation where the only random thing is $\eps_1\in\{-1,1\}$, the left-hand side of \Cref{eq:GoalContractionProof} is equal to
    \[
    \frac{1}{2} \sup_{t\in T} F\left(f(t_1) + t_2\right) + \frac{1}{2} \sup_{s\in T} F\left(- f(s_1) + s_2\right).
    \]
    Consider $s,t\in T$. Recall that $s_1,t_1\in\R^W$. Denote $s_{1,(k)}$ the $k$-th largest component of vector $s_1$. The set $\{w\in W: s_{1,w}\leq s_{1,(k)}\}$ has at least $|W|-k+1$ elements, and $\{w\in W: t_{1,(k)}\leq t_{1,w}\}$ has at least $k$ elements, so their intersection is not empty. Consider {\em any}\footnote{The choice of a specific $w$ has no importance, unlike when defining the activations of $k$-max-pooling neurons.} $w(s,t)$ in this intersection. We are now going to use that both $f(t_1) = t_{1,(k)} \leq t_{1,w(s,t)}$ and $-f(s_1) = - s_{1,(k)} \leq - s_{1,w(s,t)}$. Even if we are not going to use it, note that this implies $f(t) - f(s) \leq t_{1,w(s,t)} - s_{1,w(s,t)}$: we are exactly using an argument that establishes that $f$ is $1$-Lipschitz. Since $f(t_1) = t_{1,(k)} \leq t_{1,w(s,t)}$ and $F$ is non-decreasing, it holds:
    \begin{align*}
        F\left(f(t_1) + t_2\right) \leq & F\left(t_{1,w(s,t)} + t_2\right) \\
        \underset{\eps\textrm{ centered}}{=} & F\left(t_{1,w(s,t)} + \E\left(\sum_{w\neq w(s,t)} \eps_{1,w} t_{1,w}\right) + t_2\right) \\
        \underset{\textrm{Jensen}}{\leq} & \E F\left(t_{1,w(s,t)} + \sum_{w\neq w(s,t)} \eps_{1,w} t_{1,w} + t_2\right).
    \end{align*}
    Moreover, $-f(s_1) = - s_{1,(k)} \leq - s_{1,w(s,t)}$ so that in a similar way:
    \begin{align*}
        F\left(-f(s_1) + s_2\right) \leq & F\left(- s_{1,w(s,t)} + s_2\right) \\
        \leq & F\left(- s_{1,w(s,t)} + \E\left(\sum_{w\neq w(s,t)} \eps_{1,w} s_{1,w}\right) + s_2\right) \\
        \leq & \E F\left(-s_{1,w(s,t)} + \sum_{w\neq w(s,t)} \eps_{1,w} s_{1,w} + s_2\right).
    \end{align*}
    At the end, we get
    \begin{align*}
        & \frac{1}{2} F\left(f(t_1) + t_2\right) + \frac{1}{2} F\left(- f(s_1) + s_2\right) \\
        & \leq \frac{1}{2} \E F\left(t_{1,w(s,t)} + \sum_{w\neq w(s,t)} \eps_{1,w} t_{1,w} + t_2\right) \\
        & + \frac{1}{2} \E F\left(-s_{1,w(s,t)} + \sum_{w\neq w(s,t)} \eps_{1,w} s_{1,w} + s_2\right) \\
        & \leq \frac{1}{2} \E \sup_{r\in T} F\left(r_{1,w(s,t)} + \sum_{w\neq w(s,t)} \eps_{1,w} r_{1,w} + r_2\right) \\
        & + \frac{1}{2} \E \sup_{r\in T} F\left(-r_{1,w(s,t)} + \sum_{w\neq w(s,t)} \eps_{1,w} r_{1,w} + r_2\right) \\
        & = \E \sup_{r\in T} F\left(\eps_{1,w(s,t)} r_{1,w(s,t)} + \sum_{w\neq w(s,t)} \eps_{1,w} r_{1,w} + r_2\right) \\
        & = \E \sup_{r\in T} F\left(\sum_{w} \eps_{1,w} r_{1,w} + r_2\right).
    \end{align*}
    The latter is independent of $s,t$. Taking the supremum over all $s,t\in T$ yields \Cref{eq:GoalContractionProof} and thus the claim.
\end{proof}

\section{Peeling argument}\label{app:Peeling}
First, we state a simple lemma that will be used several times.
\begin{lemma}\label{lem:GetRidOfAbsValues}
Consider a vector $\eps\in\R^n$ with iid Rademacher coordinates, meaning that $\P(\eps_i=1)=\P(\eps_i=-1)=1/2$. Consider a function $\gPeeling:\R\to\Rp$. Consider a set $X\subset\R^n$. It holds:
\[
\E_\eps \sup_{x\in X} \gPeeling\left(\left|\sum_{i=1}^n \eps_i x_i\right|\right) \leq 2 \E_\eps \sup_{x\in X} \gPeeling\left(\sum_{i=1}^n \eps_i x_i\right).
\]
\end{lemma}

\begin{proof}[Proof of \Cref{lem:GetRidOfAbsValues}]
    Since $\gPeeling\geq 0$, it holds $\gPeeling(|x|) \leq \gPeeling(x) + \gPeeling(-x)$. Thus
    \[
    \E_\eps \sup_{x\in X} \gPeeling\left(\left|\sum_{i=1}^n \eps_i x_i\right|\right) \leq \E_\eps \sup_{x\in X} \gPeeling\left(\sum_{i=1}^n \eps_i x_i\right) + \E_\eps \sup_{x\in X} \gPeeling\left(\sum_{i=1}^n (-\eps_i) x_i\right).
    \]
    Since $\eps$ is symmetric, that is $-\eps$ has the same distribution as $\eps$, we deduce that the latter is just $ 2 \E_\eps \sup_{x\in X} \gPeeling\left(\sum_{i=1}^n \eps_i x_i\right)$. This proves the claim.
\end{proof}

\textbf{Notations} We now fix for all the next results of this section $n$ vectors $x_1,\dots,x_n\in\R^\din$, for some $\din\in\Ns$. We denote $x_{i,u}$ the coordinate $u$ of $x_i$.

For any neural network architecture, recall that $v(\param,x)$ is the output of neuron $v$ for parameters $\param$ and input $x$, and $\ant^d(v)$ is the set of neurons $u$ for which there exists a path from $u$ to $v$ of distance $d$. For a set of neurons $V$, denote $R_V(\param,x)=(v(\param,x))_{v\in V}$. 

\textbf{Introduction to peeling} This section shows that some expected sum over output neurons $v$ can be reduced to an expected maximum over $\ant(v)$, and iteratively over an expected maximum over $\ant^{d}(v)$ for increasing $d$'s. Eventually, the maximum is only over input neurons as soon as $d$ is large enough. We start with the next lemma which is the initialization of the induction over $d$: it peels off the output neurons $v$ to reduce to their antecedents $\ant(v)$.

\begin{lemma}\label{lem:PeelingNOut}
    Consider a neural network architecture as in \Cref{def:NN} with $\NeuronsIn\cap\NeuronsOut=\emptyset$. Consider an associated set $\Param$ of parameters $\param$ such that $\sum\limits_{v\in\NeuronsOut} \|\paramto{v}\|_1 + |\bias_v|\leq r$. Consider a family of independent Rademacher variables $(\eps_j)_{j\in J}$ with $J$ that will be clear from the context. Consider a non-decreasing function $\gPeeling:\R\to\Rp$. Consider a new neuron $\biasNeuron$  and set by convention $x_{\biasNeuron}=1$ for every input $x$. It holds
\begin{multline*}
 \E_{\eps}\gPeeling\left( \sup\limits_{\param\in\Param} \sum_{\substack{i=1,\dots,n,\\v\in\NeuronsOut}} \eps_{i,v} v(\param,x_i)\right) \\
 \leq \E_{\eps}\gPeeling\left( r\max_{v\in\NeuronsOut} \max_{u\in(\ant(v)\cap \NeuronsIn)\cup\{\biasNeuron\}} \left|\sum_{i=1}^n \eps_{i,v} x_{i,u}\right|\right) \\
 + \E_{\eps}\gPeeling\left(r\max_{v\in\NeuronsOut} \max_{u\in\ant(v)\setminus\NeuronsIn} \sup\limits_{\param} \left|\sum_{i=1}^n \eps_{i,v} u(\param,x_i)\right|\right)
\end{multline*}
where  in the last term if $ant(v)\setminus\NeuronsIn=\emptyset$, by convention $\max_{u\in\ant(v)\setminus\NeuronsIn} = -\infty$, and $\gPeeling(-\infty):=0$.
\end{lemma}

\begin{proof}[Proof of \Cref{lem:PeelingNOut}]
    Recall that for a set of neurons $V$, we denote $R_V(\param,x)=(v(\param,x))_{v\in V}$. Recall that by \Cref{def:NN}, output neurons $v$ have $\rho_v=\id$ so for every $v\in\NeuronsOut$, $\param\in\Param$ and every input $x$:
    \[
    v(\param,x) = \inner{\left(\begin{array}{c}
                     \paramto{v}  \\
                     \bias_v
                \end{array}\right)}{\left(\begin{array}{c}
                     R_{\ant(v)}(\param, x) \\
                     1
                \end{array}\right)}.
    \]
    Denote by $\biasNeuron$ a new neuron that computes the constant function equal to one ($\biasNeuron(\param,x)=1$), we get:
    \begin{align*}
                & \E_{\eps}\gPeeling\left( \sup\limits_{\param} \sum\limits_{\substack{i=1,\dots,n \\ v\in\NeuronsOut}} \eps_{i,v} v(\param,x_i)\right) \\
                & = \E_{\eps}\gPeeling\left( \sup\limits_{\param} \sum\limits_{v\in\NeuronsOut} \inner{\left(\begin{array}{c}
                     \paramto{v}  \\
                     \bias_v
                \end{array}\right)}{\sum_{i=1}^n \eps_{i,v} \left(\begin{array}{c}
                     R_{\ant(v)}(\param, x_i) \\
                     1
                \end{array}\right)}\right) \\
                & \underset{\textrm{Hölder}}{\leq} \E_{\eps}\gPeeling\left( \sup\limits_{\param} \underbrace{\left(\sum\limits_{v\in\NeuronsOut} \|\paramto{v}\|_1 + |\bias_v|\right)}_{\leq r\textrm{ by assumption}
                } \max_{v\in\NeuronsOut}\left(\left|\sum_{i=1}^n \eps_{i,v}\right|,\max_{u\in\ant(v)} \left|\sum_{i=1}^n \eps_{i,v} u(\param,x_i)\right|\right)\right) \\
                & \leq \E_{\eps}\gPeeling\left( r\max_{v\in\NeuronsOut} \max_{u\in\ant(v)\cup\{\biasNeuron\}} \sup\limits_{\param} \left|\sum_{i=1}^n \eps_{i,v} u(\param,x_i)\right|\right).
            \end{align*}
        Everything is non-negative so the maximum over $u\in\ant(v)\cup\{\biasNeuron\}$ is smaller than the sum of the maxima over $u\in(\ant(v)\cap \NeuronsIn)\cup\{\biasNeuron\}$ and $u\in\ant(v)\setminus\NeuronsIn$. Note that when $u$ is an input neuron, it simply holds $u(\param,x_i)=x_{i,u}$. This proves the result.
\end{proof}
We now show how to peel neurons to reduce the maximum over $\ant^d(v)$ to $\ant^{d+1}(v)$. Later, we will repeat that until the maximum is only on input neurons. Compared to the previous lemma, note the presence of an index $m=1,\dots,M$ in the maxima. This is because after $d$ steps of peeling (when the maximum over $u$ has been reduced to $u\in\ant^{d}(v)$), we will have $M=K^{d-1}$ where $K$ is the kernel size. Indeed, the number of copies indexed by $m$ gets multiplied by $K$ after each peeling step.

\begin{lemma}\label{lem:PeelingInduction}
Consider a neural network architecture with an associated set $\Param$ of parameters $\param$ such that every neuron $v\notin\NeuronsOut\cup\NeuronsIn$ satisfies $\|\paramto{v}\|_1 + |\bias_v|\leq 1$. Assume that $\bias_v=0$ for every $v\in\NeuronSet_{\pool}$. Consider a family of independent Rademacher variables $(\eps_j)_{j\in J}$ with $J$ that will be clear from the context. Consider arbitrary $M,d\in\N$ and a convex non-decreasing function $\gPeeling:\R\to\Rp$. Take a symbol $\biasNeuron$  which does not correspond to a neuron ($\biasNeuron\notin\NeuronSet$) and set by convention $x_{\biasNeuron}=1$ for every input $x$. Define $P:=|\{k\in\Ns, \exists u\in\NeuronSet_{\kpool}\}|$ as the number of different types of $*$-max-pooling neurons in $G$, and $K:=\max_{u\in\NeuronSet_{\pool}} |\ant(u)|$ the maximal kernel size of the network ($K:=1$ if $P=0$). It holds:
\begin{multline*}
\E_{\eps}\gPeeling\left( \max_{\substack{v\in\NeuronsOut, \\ m=1,\dots,M}} \max_{u\in\ant^d(v)\setminus \NeuronsIn} \sup\limits_{\param} \left|\sum_{i=1}^n \eps_{i,v,m} u(\param, x_i)\right|\right) \\
         \leq (3+2P)\E_{\eps}\gPeeling\left( \max_{\substack{v\in\NeuronsOut, \\ m=1,\dots,KM}} \max_{u\in(\ant^{d+1}(v)\cap \NeuronsIn)\cup\{\biasNeuron\}} \left|\sum_{i=1}^n \eps_{i,v,m} x_{i,u}\right| \right) \\
         + (3+2P)\E_{\eps}\gPeeling\left( \max_{\substack{v\in\NeuronsOut, \\ m=1,\dots,KM}} \max_{u\in\ant^{d+1}(v)\setminus \NeuronsIn} \sup\limits_{\param} \left|\sum_{i=1}^n \eps_{i,v,m} u(\param,x_i)\right| \right)
\end{multline*}
with similar convention as in \Cref{lem:PeelingNOut} for empty maxima.
\end{lemma}

\begin{proof}
    \textbf{Step 1: split the neurons depending on their activation function.} In the term that we want to bound from above, the neurons $u\in\ant^d(v)\setminus \NeuronsIn$ are not input neurons so they compute something of the form $\rho_u(\dots)$ where $\rho_u$ is the activation associated with $u$, $1$-Lipschitz, and satisfies $\rho_u(0)=0$. The first step of the proof is to get rid of $\rho_u$ using a contraction lemma similar to Theorem 4.12 in \citet{LedouxTalagrand91ProbaBanachSpaces}. However, here, the function $\rho_u$ depends on the neuron $u$, what we are taking a maximum over so that classical contraction lemmas do not apply directly. To resolve this first obstacle, we split the neurons according to their activation function. Below, we highlight in \orange{orange} what is important and/or the changes from one line to another. Denote $\NeuronSet_\rho$ the neurons that have $\rho$ as their associated activation function, and the term with a maximum over all $u\in\NeuronSet_\rho$ is denoted:
    \[
        e(\rho) := \E_{\eps}\gPeeling\left( \max_{\substack{v\in\NeuronsOut, \\ m=1,\dots,M}} \max_{u\in(\ant^d(v)\cap \orange{\NeuronSet_{\rho}})\setminus\NeuronsIn} \sup\limits_{\param} \left|\sum_{i=1}^n \eps_{i,v,m} u(\param,x_i)\right|\right),
    \]
    with the convention $e(\rho)=0$ if $\NeuronSet_\rho$ is empty. This yields a first bound
    \[
                 \E_{\eps}\gPeeling\left( \max_{\substack{v\in\NeuronsOut, \\ m=1,\dots,M}} \max_{u\in\ant^d(v)\setminus\NeuronsIn} \sup\limits_{\param} \left|\sum_{i=1}^n \eps_{i,v,m} u(\param,x_i)\right|\right) \\
                 \leq e(\relu) + e(\id) + \sum_{k} e(\kpool)
    \]
    where the sum of the right-hand side is on all the $k\in\Ns$ such that there is at least one neuron in $\NeuronSet_{\kpool}$. Define $E(\rho)$ to be the same thing as $e(\rho)$ but without the absolute values:
    \[
        E(\rho) := \E_{\eps}\gPeeling\left( \max_{\substack{v\in\NeuronsOut, \\ m=1,\dots,M}} \max_{u\in(\ant^d(v)\cap \orange{\NeuronSet_{\rho}})\setminus\NeuronsIn} \sup\limits_{\param} \sum_{i=1}^n \eps_{i,v,m} u(\param,x_i)\right).
    \]
    \Cref{lem:GetRidOfAbsValues} gets rid of the absolute values by paying a factor 2:
    \[
        e(\rho)\leq 2 E(\rho).
    \]
    We now want to bound each $E(\rho)$.
    
    \textbf{Step 2: get rid of the $*$-max-pooling and ReLU activation functions.} Since the maximal kernel size is $K$, any $*$-max-pooling neuron $u$ must have at most $K$ antecedents. When a $u\in\NeuronSet_{\pool}$ has less than $K$ antecedents, we artificially add neurons $w$ to $\ant(u)$ to make it of cardinal $K$, and we set by convention $\paramfromto{w}{u}=0$. We also fix an arbitrary order on the antecedents of $u$ and write $\ant(u)_w$ for the antecedent number $w$, with $R_{\ant(u)_w}$ the function associated with this neuron. For a ReLU or $*$-max-pooling neuron $u$, define the pre-activation of $u$ to be
    \begin{align*}
        \pre_u(\param,x):=\left\{\begin{array}{cc}
            \inner{\left(\begin{array}{c}
                     \paramto{u}  \\
                     \bias_u 
                \end{array}\right)}{\left(\begin{array}{c}
                     R_{\ant(u)}(\param, x) \\
                     1
                     \end{array}\right)} & \textrm{if }u\in\NeuronSet_{\relu}, \\
            \left(\bias_u + \paramfromto{\ant(u)_w}{u}R_{\ant(u)_w}(\param,x)\right)_{w=1,\dots,k} & \textrm{otherwise when }u\in\NeuronSet_{\pool}.
        \end{array}\right.
    \end{align*}
    where we recall that $\bias_u=0$ for $u\in\NeuronSet_{\pool}$ by assumption. We nevertheless keep $\bias_u$ in the computation until the point where this assumption is apparently actually needed to continue.
    Note that the pre-activation has been defined to satisfy $u(\param,x)=\rho_u(\pre_u(\param,x))$. When $\rho$ is the ReLU or $\kpool$, we can thus rewrite $E(\rho)$ in terms of the pre-activations:
    \[
        E(\rho) = \E_{\eps}\gPeeling\left( \max_{\substack{v\in\NeuronsOut, \\ m=1,\dots,M}} \max_{u\in(\ant^d(v)\cap\NeuronSet_{\rho})\setminus\NeuronsIn} \sup\limits_{\param} \sum_{i=1}^n \eps_{i,v,m} \orange{\rho(\pre_u(\param,x_i))}\right).
    \]
    Consider the finite set $Z=\{(v,m), v\in\NeuronsOut, m=1,\dots,M\}$ and for every $z=(v,m)\in Z$, define $T^{z}=\{(\pre_u(\param,x_i))_{i=1,\dots,n}: u\in(\ant^d(v)\cap \NeuronSet_{\rho})\setminus\NeuronsIn, \param\in\Param\}$. An element of $T^z$ will be denoted $t=(t_i)_{i=1}^n\in T^z\subset\R^n$ if $\rho=\relu$, and $t=(t_i)_{i=1}^n \in T^z\subset(\R^{k})^n$ with $t_i=(t_{i,w})_{w=1}^k\in\R^k$ if $u\in\NeuronSet_{\pool}$. We can again rewrite $E(\rho)$ as
    \[
        E(\rho) = \E_{\eps}\gPeeling\left(\max_{\orange{z\in Z}} \sup_{\orange{t\in T^z}} \sum_{i=1}^n \eps_{i,\orange{z}} \orange{\rho(t_i)}\right).
    \]
    We now want to get rid of the activation function $\rho$ with a contraction lemma. There is a second difficulty that prevents us from directly applying classical contraction lemmas such as Theorem 4.12 of \citet{LedouxTalagrand91ProbaBanachSpaces}. It is the presence of a maximum over multiple copies indexed by $z\in Z$ of a supremum that depends on iid families $(\eps_{i,z})_{i=1\dots n}$. Indeed, Theorem 4.12 of \citet{LedouxTalagrand91ProbaBanachSpaces} only deals with a single copy ($|Z|=1$). This motivates the contraction lemma established for the occasion in \Cref{lem:ContractionLemma}. Once the activation functions removed, we can conclude separately for $\rho=\relu,\id$ and $\rho=\kpool$.
    
    \textbf{Step 3a: deal with $\rho=\kpool$ via rescaling.} In the case $\rho=\kpool$, \Cref{lem:ContractionLemma} shows that
    \begin{align*}
        & \E_{\eps}\gPeeling\left( \max_{z\in Z} \sup_{t\in T^z} \sum_{i=1}^n \eps_{i,z} \orange{\kpool}(t_i)\right) \\
        & \leq \E_{\eps}\gPeeling\left( \max_{z\in Z} \sup_{t\in T^z} \sum_{\substack{i=1,\dots,n,\\ w=1,\dots,K}} \eps_{i,z,w} t_{i,w}\right).
    \end{align*}
    The right-hand side is equal to
    \begin{equation}\label{eq:PeelingMaxRescaling}
        \E_{\eps}\gPeeling\left( \max_{\substack{v\in\NeuronsOut, \\ m=1,\dots,M}} \sup_{\substack{u\in(\ant^d(v)\cap \NeuronSet_{\pool})\setminus\NeuronsIn,\\ \param\in\Param}} \sum_{\substack{i=1,\dots,n,\\ w=1,\dots,K}} \eps_{i,v,m,w}(\bias_u + \paramfromto{\ant(u)_w}{u}R_{\ant(u)_w}(\param,x_i))\right).
    \end{equation}
    We now deal with this using the assumption on the norm of incoming weights.
    Recalling that $\bias_u=0$ for every $u\in\NeuronSet_{\pool}$:
    
    \begin{multline*}
         \sum_{\substack{i=1,\dots,n,\\ w=1,\dots,K}} \eps_{i,v,m,w}(\underbrace{\bias_u}_{=0} + \paramfromto{\ant(u)_w}{u}R_{\ant(u)_w}(\param,x_i)) \\
         = \sum_{w=1,\dots,K} \paramfromto{\ant(u)_w}{u} \left(\sum_{i=1,\dots,n} \eps_{i,v,m,w}R_{\ant(u)_w}(\param,x_i)\right) \\
         \underset{\textrm{Hölder}}{\leq} \underbrace{\|\paramto{u}\|_1}_{
        \leq 1\textrm{ by assumption}
        } \max_{w=1,\dots,K} \left| \sum_{i=1,\dots,n} \eps_{i,v,m,w}R_{\ant(u)_w}(\param,x_i)\right| \\
         \underset{\textrm{decoupling $w$ and $\ant(u)_w$}}{\leq} \max_{\orange{w}\in\ant(u)} \max_{\orange{w'}=1,\dots,K}\left|\sum_{i=1}^n \eps_{i,v,m,\orange{w'}} \orange{w}(\param,x_i)\right|.
    \end{multline*}
    Note for the curious reader that this inequality is the current obstacle when the biases are nonzero since we would end up with $K|\bias_u| + \|\paramto{u}\|_1$ and we could not anymore use that this is $\leq 1$.
    
    We deduce that \Cref{eq:PeelingMaxRescaling} is bounded from above by
    \[
        \E_{\eps}\gPeeling\left(\max_{\substack{v\in\NeuronsOut, \\ m=1,\dots,M}} \sup_{u\in(\ant^d(v)\cap \NeuronSet_{\pool})\setminus\NeuronsIn, \param\in\Param} \max_{w\in\ant(u)} \max_{w'=1,\dots,K}\left|\sum_{i=1}^n \eps_{i,v,m,w'} w(\param,x_i)\right|\right).
    \]
    Instead of having $\eps_{i,v,m,w'}$ with $m=1,\dots,M$ and $w'=1,\dots,K$, we re-index it as $\eps_{i,v,m}$ with $m=1,\dots,KM$. Note also that $u\in(\ant^d(v)\cap \NeuronSet_{\pool})\setminus\NeuronsIn$ and $w\in\ant(u)$ implies $w\in\ant^{d+1}(v)$, so considering a maximum over $w\in\ant^{d+1}(v)$ can only yield something larger. Moreover, we can add a new neuron $\biasNeuron$  that computes the constant function equal to one ($\biasNeuron(\param,x)=1$) and add $\biasNeuron$ to the maximum over $w$. Implementing all these changes, \Cref{eq:PeelingMaxRescaling} is bounded by
    \[
        H:=\E_{\eps}\gPeeling\left(\max_{\substack{v\in\NeuronsOut, \\ {\color{orange} m=1,\dots,KM}}} \sup_{\orange{w\in\ant^{d+1}(v)\cup\{\biasNeuron\}}} \sup_{\param\in\Param} \left|\sum_{i=1}^n \eps_{i,v,{\orange{m}}} w(\param,x_i)\right|\right).
    \]
    We now derive similar inequalities when $\rho=\id$ and $\rho=\relu$.

    \textbf{Step 3b: deal with $\rho=\id,\relu$ via rescaling.} In the case $\rho=\relu$, \Cref{lem:ContractionLemma} shows that
    \begin{align*}
        & \E_{\eps}\gPeeling\left(\max_{z\in Z} \sup_{t\in T^z} \sum_{i=1}^n \eps_{i,z} \orange{\relu}(t_i)\right) \\
        & \leq \E_{\eps}\gPeeling\left( \max_{z\in Z} \sup_{t\in T^z} \sum_{i=1,\dots,n} \eps_{i,z} t_{i}\right).
    \end{align*}
    The difference with the $*$-max-pooling case is that each $t_{i}$ is scalar so this does not introduce an additional index $w$ to the Rademacher variables. The right-hand side can be rewritten as
    \begin{equation*}
        \E_{\eps}\gPeeling\left(\max_{\substack{v\in\NeuronsOut, \\ m=1,\dots,M}} \sup_{u\in(\ant^d(v)\cap \orange{\NeuronSet_{\relu}})\setminus\NeuronsIn, \param\in\Param} \sum_{i=1,\dots,n} \eps_{i,v,m} \inner{\left(\begin{array}{c}
                     \paramto{u}  \\
                     \bias_u 
                \end{array}\right)}{\left(\begin{array}{c}
                     R_{\ant(u)}(\param, x_i) \\
                     1
                     \end{array}\right)}\right)
    \end{equation*}
    We can only increase the latter by considering a maximum over all $u\in\ant^d(v)$, not only the ones in $\NeuronSet_{\relu}$. We also add absolutes values. This is then bounded by
    \begin{equation}
        \orange{F} := \E_{\eps}\gPeeling\left(\max_{\substack{v\in\NeuronsOut, \\ m=1,\dots,M}} \sup_{u\in\ant^d(v)\setminus\NeuronsIn, \param\in\Param} \orange{\biggl|}\sum_{i=1,\dots,n} \eps_{i,v,m} \inner{\left(\begin{array}{c}
                     \paramto{u}  \\
                     \bias_u 
                \end{array}\right)}{\left(\begin{array}{c}
                     R_{\ant(u)}(\param, x_i) \\
                     1
                     \end{array}\right)}\orange{\biggr|}\right).
    \end{equation}
    This means that $E(\relu)\leq F$. Let us also observe that $e(\id)\leq F$. Indeed, recall that by definition
    \begin{align*}
        e(\id) =  \E_{\eps}\gPeeling\left(\max_{\substack{v\in\NeuronsOut, \\ m=1,\dots,M}} \max_{u\in(\ant^d(v)\cap \orange{\NeuronSet_{\id}})\setminus\NeuronsIn} \sup\limits_{\param} \left|\sum_{i=1}^n \eps_{i,v,m} u(\param,x_i)\right|\right).
    \end{align*}
    We can only increase the latter by considering a maximum over all $u\in\ant^d(v)$. Moreover, for an identity neuron $u$, it holds $u(\param,x)=\inner{\left(\begin{array}{c}
                     \paramto{u}  \\
                     \bias_u 
                \end{array}\right)}{\left(\begin{array}{c}
                     R_{\ant(u)}(\param, x) \\
                     1
                     \end{array}\right)}$. %
    This shows that $e(\id)\leq F$. It remains to bound $F$ using that the assumption on the norm of the parameters. Introduce a new neuron $\biasNeuron$  that computes the constant function equal to one: $\biasNeuron(\param,x)=1$. Note that
    \begin{align*}
        & \sum_{i=1,\dots,n} \eps_{i,v,m} \inner{\left(\begin{array}{c}
                     \paramto{u}  \\
                     \bias_u 
                \end{array}\right)}{\left(\begin{array}{c}
                     R_{\ant(u)}(\param, x_i) \\
                     1
                     \end{array}\right)} \\
        & = \inner{\left(\begin{array}{c}
                     \paramto{u}  \\
                     \bias_u 
                \end{array}\right)}{\sum_{i=1,\dots,n} \eps_{i,v,m} \left(\begin{array}{c}
                     R_{\ant(u)}(\param, x_i) \\
                     1
                     \end{array}\right)} \\
        & \underset{\textrm{Hölder}}{\leq} \underbrace{\left(\|\paramto{u}\|_1 + |\bias_u|\right)}_{
        \leq 1\textrm{ by assumption}
        } \max_{w\in\ant(u)\cup\{\biasNeuron\}} \left|\sum_{i=1}^n \eps_{i,v,m} w(\param,x_i)\right|.
    \end{align*} 
    This shows that
    \begin{align*}
        F \leq \E_{\eps}\gPeeling\left(\max_{\substack{v\in\NeuronsOut, \\ m=1,\dots,M}} \sup_{u\in\ant^d(v)\setminus\NeuronsIn, \param\in\Param} \max_{w\in\ant(u)\cup\{\biasNeuron\}} \left|\sum_{i=1}^n \eps_{i,v,m} w(\param,x_i)\right|\right).
    \end{align*}
    Obviously, introducing additional copies of $\eps$ to make the third index going from $m=1$ to $KM$ can only make it larger. Moreover, $u\in\ant^d(v)\setminus\NeuronsIn$ and $w\in\ant(u)$ implies $w\in\ant^{d+1}(u)$, so we can instead consider a maximum over $w\in\ant^{d+1}(v)$. This gives the upper-bound
    \begin{align*}
        F & \leq \E_{\eps}\gPeeling\left(\max_{\substack{v\in\NeuronsOut, \\ {\color{orange} m=1,\dots,KM}}} \max_{\orange{w\in\ant^{d+1}(v)\cup\{\biasNeuron\}}} \sup_{\param\in\Param} \left|\sum_{i=1}^n \eps_{i,v,m} w(\param,x_i)\right|\right) \\
        & = H.
    \end{align*}
    
    \textbf{Step 4: putting everything together.}
    At the end, recalling that there are at most $P$ different $k\in\Ns$ associated with an existing $k$-max-pooling neuron, we get the final bound
    \begin{align*}
        & \E_{\eps}\gPeeling\left(\max_{\substack{v\in\NeuronsOut, \\ m=1,\dots,M}} \max_{u\in\ant^d(v)\setminus\NeuronsIn} \sup\limits_{\param} \left|\sum_{i=1}^n \eps_{i,v,m} u(\param,x_i)\right|\right) \\
        & \leq e(\id) + e(\relu) +  \sum_k e(\kpool) \\
        & \leq e(\id) + 2E(\relu) + 2\sum_k E(\kpool) \\
        & \leq F + 2F + 2\sum_k E(\kpool) \\
        & \leq H + 2H + 2\sum_k H \\
        & \leq H + 2H + 2PH = (3+2P)H.
    \end{align*}
    The term $(3+2P)H$ can again be bounded by splitting the maximum over $w\in\ant^{d+1}(v)\cup\{\biasNeuron\}$ between the $w$'s that are input neurons, and those that are not, since everything is non-negative. This yields the claim.
\end{proof}

\begin{remark}[Improved dependencies on the kernel size]
    Note that in the proof of \Cref{lem:PeelingInduction}, the multiplication of $M$ by $K$ can be avoided if there are no $*$-max-pooling neurons in $\ant^d(v)$. Because of skip connections, even if there is a single $*$-max-pooling neuron in the architecture, it can be in $\ant^d(v)$ for many $d$'s. A more advanced version of the argument is to peel only the ReLU and identity neurons, by leaving the $*$-max-pooling neurons as they are, until we reach a set of $*$-max-pooling neurons large enough that we decide to peel simultaneously. This would prevent the multiplication by $K$ every time $d$ is increased.
\end{remark}

We can now state the main peeling theorem, which directly result from \Cref{lem:PeelingNOut} and \Cref{lem:PeelingInduction} by induction on $d$. Note that these lemmas contain assumptions on the size of the incoming weights of the different neurons. These assumptions are met using \Cref{alg:NormalizeAlgo}, which rescales the parameters without changing the associated function nor the path-norm.

\begin{theorem}\label{thm:Peeling}
    Consider a neural network architecture as in \Cref{def:NN} with $\NeuronsOut\cap\NeuronsIn=\emptyset$. Assume that $\bias_v=0$ for every $v\in\NeuronSet_{\pool}$. Define $P:=|\{k\in\Ns, \exists u\in\NeuronSet_{\kpool}\}|$ the number of different types of $*$-max-pooling neurons in $G$, and $K:=\max_{u\in\NeuronSet_{\pool}}$ the maximum kernel size ($K:=1$ by convention if $P=0$). Denote by $\biasNeuron$ a new input neuron and define $x_{\biasNeuron}=1$ for any input $x$. For any set of parameters $\Param$ associated with the network, such that $\|\Phi(\param)\|_1\leq r$ for every $\param\in\Param$, it holds for every convex non-decreasing function $\gPeeling:\R\to\Rp$
 \[
        \begin{aligned}
        & \E_{\eps} \gPeeling\left(\sup_{\param\in\Param} \sum_{\substack{i=1,\dots,n,\\ v\in\NeuronsOut}} \eps_{i,v} v(\param,x_i)\right) \\
        & \leq 
        \frac{(3+2P)^D}{2+2P}\E_{\eps}\gPeeling\left(r \max_{\substack{v\in\NeuronsOut, \\ m=1,\dots,K^{D-1}}} \max_{u\in\NeuronsIn\cup\{\biasNeuron\}} \left|\sum_{i=1}^n \eps_{i,v,m} x_{i,u}\right| \right).
        \end{aligned}
 \]
\end{theorem}
\begin{proof}[Proof of \Cref{thm:Peeling}] Without loss of generality, we can replace $\Param$ by its image under \Cref{alg:NormalizeAlgo} with $q=1$, as \Cref{alg:NormalizeAlgo} does not change the associated function $R_\param$ nor the path-norm $\|\Phi(\param)\|_1$ (\Cref{lem:AlgNormalization}) so that we still have $\|\Phi(\param)\|_1\leq r$ and the supremum over $\param\in\Param$ on the left-hand side can be taken over rescaled parameters. By \Cref{lem:AlgNormalization}, the parameters are $1$-normalized (\Cref{def:NormalizedParameters}), so we will be able to use \Cref{lem:PeelingNOut} and \Cref{lem:PeelingInduction}. Indeed, $1$-normalized parameters $\param$ satisfy $\sum_{v\in\NeuronsOut} \|\paramto{v}\|_1 + |\bias_v| = \|\Phi(\param)\|_1\leq r$ so \Cref{lem:PeelingNOut} applies. We also have $\|\paramto{v}\|_1 + |\bias_v|\leq 1$ for every  $v\notin\NeuronsIn\cup\NeuronsOut$, by definition of $1$-normalization, so \Cref{lem:PeelingInduction} also applies.

By induction on $d\geq 1$, we prove that (highlighting in \orange{orange} what is important)
\begin{align*}
    & \E_{\eps}\gPeeling\left(\sup\limits_{\param\in\Param} \sum_{\substack{i=1,\dots,n,\\v\in\NeuronsOut}} \eps_{i,v} v(\param,x_i)\right) \\
    & \leq \sum_{\ell=1}^d (3+2P)^{\ell-1} \E_{\eps}\gPeeling\left(r\max_{\substack{v\in\NeuronsOut, \\ m=1,\dots,K^{\ell-1}}} \max_{u\in(\ant^\ell(v)\orange{\cap \NeuronsIn})\cup\{\biasNeuron\}} \left|\sum_{i=1}^n \eps_{i,v,m} \orange{x_{i,u}}\right| \right) \\
    & + (3+2P)^{d-1} \E_{\eps}\gPeeling\left(r\max_{\substack{v\in\NeuronsOut, \\ m=1,\dots,K^{d-1}}} \max_{u\in\ant^d(v)\orange{\setminus\NeuronsIn}} \sup\limits_{\param\in\Param} \left|\sum_{i=1}^n \eps_{i,v,m} \orange{u(\param,x_i)}\right|\right),
\end{align*}
with the same convention as in \Cref{lem:PeelingNOut} for maxima over empty sets. This is true for $d=1$ by \Cref{lem:PeelingNOut}. The induction step is verified using \Cref{lem:PeelingInduction}. This concludes the induction. Applying the result for $d=D$, and since $\ant^D(v)\setminus\NeuronsIn=\emptyset$, we get:
\begin{align*}
    & \E_{\eps}\gPeeling\left(\sup\limits_{\param\in\Param} \sum_{\substack{i=1,\dots,n,\\v\in\NeuronsOut}} \eps_{i,v} v(\param,x_i)\right) \\
    & \leq \sum_{d=1}^D (3+2P)^{d-1} \E_{\eps}\gPeeling\left(r\max_{\substack{v\in\NeuronsOut, \\ m=1,\dots,K^{d-1}}} \max_{u\in(\ant^d(v)\cap \NeuronsIn)\cup\{\biasNeuron\}} \left|\sum_{i=1}^n \eps_{i,v,m} x_{i,u}\right| \right).
\end{align*}
We can only increase the right-hand side by considering maximum over all $u\in\NeuronsIn\cup\{\biasNeuron\}$ and by adding independent copies indexed from $m=1$ to $m=K^{D-1}$. Moreover, $\sum_{d=1}^D (3+2P)^{d-1} = ((3+2P)^D-1)/(2+2P)$. This shows the final bound:
\begin{align*}
    & \E_{\eps}\gPeeling\left( \sup\limits_{\param\in\Param} \sum_{\substack{i=1,\dots,n,\\v\in\NeuronsOut}} \eps_{i,v} v(\param,x_i)\right) \\
    & \leq \frac{(3+2P)^D}{2+2P} \E_{\eps}\gPeeling\left( r\max_{\substack{v\in\NeuronsOut, \\ m=1,\dots,K^{D-1}}} \max_{u\in\NeuronsIn\cup\{\biasNeuron\}} \left|\sum_{i=1}^n \eps_{i,v,m} x_{i,u}\right| \right).
\end{align*}
\end{proof}

\section{Details to derive the generalization bound (\Cref{thm:GeneralizationBound})}\label{app:DetailsProofGeneralizationBound}
\begin{proof}[Proof of \Cref{thm:GeneralizationBound}]
The proof is given directly in the general case with nonzero biases on every $v\notin\NeuronSet_{\pool}$ and uses an additional input neuron $\biasNeuron$. We highlight along the proof where improved results can be obtained assuming zero biases, yielding \Cref{thm:GeneralizationBound} as a consequence. Define the random matrices $E=(\eps_{i,v})_{i,v}\in\R^{n\times \dout}$ and $R(\param,\rv{X})=(v(\param,\rv{X}_i))_{i,v}\in\R^{n\times \dout}$ so that $\inner{E}{R(\param,\rv{X})} = \sum_{i,v} \eps_{i,v} (R_{\param}(\rv{X}_i))_{v}$. It holds:
\begin{align*}
            \E_{\rv{Z}} \textrm{ $\ell$-generalization error of $\hatparam(\rv{Z})$} & \leq \frac{2}{n} \E_{\rv{Z}, \eps}\left( \sup_{\param} \sum_{i=1}^n \eps_i\ell\left(R_{\param}(\rv{X}_i), \rv{Y}_i\right)\right)\\
            & \leq \frac{2\sqrt{2}L}{n} \E_{\rv{Z}, \eps}\left( \sup\limits_{\param} \inner{E}{R(\param,\rv{X})}\right).
\end{align*}
The first inequality is the symmetrization property given by \citet[Theorem 26.3]{ShalevShwartz14UnderstandingML}, and the second inequality is the vector-valued contraction property given by \citet{Maurer16VectorContraction}. These are the relevant versions of very classical arguments that are widely used to reduce the problem to the Rademacher complexity of the model \citep[Propositions 4.2 and 4.3]{BachBook}\citep[Equations (4.17) and (4.18)]{Wainwright19HDS}\citep[Proof of Theorem 8]{Bartlett02RademacherComplexity}\citep[Theorem 26.3]{ShalevShwartz14UnderstandingML}\citep[Equation (4.20)]{LedouxTalagrand91ProbaBanachSpaces}. In particular, this step has nothing specific with neural networks. Note that the assumption on the loss is used for the second inequality.

We now condition on $\rv{Z}=(\rv{X},\rv{Y})$ and denote $\E_{\eps}$ the conditional expectation. For any random variable $\lambda(\rv{Z})>0$ measurable in $\rv{Z}$, it holds
\begin{align*}
    \E_{\eps}\left(\sup\limits_{\param} \inner{E}{R(\param,\rv{X})}\right) = & \frac{1}{\lambda(\rv{Z})}\log \exp\left(\lambda(\rv{Z}) \E_{\eps}\left( \sup\limits_{\param} \inner{E}{R(\param,\rv{X})}\right)\right) \\
    \underset{\textrm{$\lambda$ measurable in $\rv{Z}$}}{=} & \frac{1}{\lambda(\rv{Z})}\log \exp\left(\E_{\eps}\left(\lambda(\rv{Z}) \sup\limits_{\param} \inner{E}{R(\param,\rv{X})}\right)\right) \\
    \underset{\textrm{Jensen}}{\leq} & \frac{1}{\lambda(\rv{Z})}\log \E_{\eps}\exp\left(\lambda(\rv{Z}) \sup\limits_{\param} \inner{E}{R(\param,\rv{X})}\right).
\end{align*}
For $z=((x_i,y_i))_{i=1}^n\in(\R^\din\times\R^\dout)^n$, denote 
\[
e(z) = \E_{\eps}\exp\left(\lambda(z) \sup\limits_{\param} \inner{E}{R(\param,x)}\right).
\]
Since $\rv{Z}$ is independent of $\eps$, it holds
\[
\E_{\eps}\exp\left(\lambda(\rv{Z}) \sup\limits_{\param} \inner{E}{R(\param,\rv{X})}\right) = e(\rv{Z}).
\]
Denote $r=\sup_{\param\in\Param} \|\Phi(\param)\|_1$. For $z$ as above, simply denote $\lambda:=\lambda(z)$. Since $\NeuronsIn\cap\NeuronsOut=\emptyset$ and the biases of $*$-max-pooling neurons are null, the peeling argument given by \Cref{thm:Peeling} for $\gPeeling:t\in\R\mapsto \exp(\lambda t)$ guarantees:
\[
    e(z) \leq  
    \frac{(3+2P)^D}{2+2P}\E_{\eps}\exp\left(\lambda r\max_{\substack{v\in\NeuronsOut, \\ m=1,\dots,K^{D-1}}} \max_{u\in\NeuronsIn\cup\{\biasNeuron\}} \left|\sum_{i=1}^n \eps_{i,v,m} x_{i,u}\right| \right),
\]
where $x_{i,u}$ is coordinate $u$ of vector $x_i\in\R^\din$, and where $\biasNeuron$  is an added neuron for which we set by convention $x_{\biasNeuron}=1$ for any input $x$. It is easy to check that the same bound holds true with a maximum only over $u\in\NeuronsIn$ (not considering $\biasNeuron$ in the maximum) when all biases are constrained to be null. In such a setting, all the $\max_{u\in \NeuronsIn\cup\{\biasNeuron\}}$ below can be replaced by $\max_{u\in \NeuronsIn}$.  Denote
\[
\sigma(x) := \max_{u\in \NeuronsIn\cup\{\biasNeuron\}} \left(\sum_{i=1}^n x_{i,u}^2\right)^{1/2} \geq \sqrt{n}.
\]
Using \Cref{lem:RadInputBound}, it holds
\begin{align*}
    \E_{\eps}\exp\left(\lambda r \max_{\substack{v\in \NeuronsOut, \\ u\in\NeuronsIn\cup\{\biasNeuron\}, \\ m=1,\dots,K^{D-1}}} \left|\sum_{i=1}^n \eps_{i,v,m} (\rv{X}_i)_{u}\right| \right) \leq 2K^{D-1}(\din+1)\dout \exp\left(\frac{(r \lambda(z)\sigma(x))^2}{2}\right).
\end{align*}
When biases are constrained to be null, $\din+1$ is replaced by $\din$. Putting everything together, we get:
\begin{align*}
    \E_{\eps}\left(\sup\limits_{\param} \inner{E}{R(\param,\rv{X})}\right) = e(\rv{Z}) & \leq \left(\frac{1}{\lambda}\log(C_1) + \lambda(\rv{Z}) C_2(\rv{X})\right)
\end{align*}
with
\[
C_1 = 2K^{D-1}(\din+1)\dout \times \frac{(3+2P)^D}{2+2P} = \frac{3+2P}{1+P}((3+2P)K)^{D-1}(\din+1)\dout
\]
(again with $\din+1$ replaced by $\din$ when all biases are null) and
\[
C_2(\rv{X}) = \frac{1}{2}(r\sigma({\rv{X}}))^2.
\]
Choosing $\lambda(\rv{Z})=\sqrt{\frac{\log(C_1)}{C_2(\rv{X})}}$ yields:
\begin{align*}
    & \E_{\eps}\left(\sup\limits_{\param} \inner{E}{R(\param,\rv{X})}\right) \leq 2\sqrt{\log(C_1)C_2(\rv{X})} \\
    & \leq \underbrace{\sqrt{2}\sigma(\rv{X}) r}_{=2\sqrt{C_2(\rv{X})}} \underbrace{\left(\log\left(\frac{3+2P}{1+P}(\din+1)\dout\right) + D\log\left((3+2P)K\right)\right)^{1/2}}_{\geq \sqrt{\log(C_1)}}
\end{align*}
with $\din+1$ replaced by $\din$ when all biases are null.
Taking the expectation on both sides over $\rv{Z}$, and multiplying this by 
$\frac{2\sqrt{2}L}{n}$ yields \Cref{thm:GeneralizationBound}.
\end{proof}

The next lemma is classical \citep[Section 7.1]{Golowich18GeneralizationBound} and is here only for completeness.

\begin{lemma}\label{lem:RadInputBound}
    For any $d,k\in\Ns$ and $\lambda>0$, it holds
  \[
  \E_{\eps}\exp\left(\lambda \max_{\substack{m=1,\dots,k, \\ u=1,\dots, d}}\left|\sum_{i=1}^n \eps_{i,m} (\rv{X}_i)_{u}\right| \right) \leq 2kd\max_{u=1,\dots,d}\exp\left(\frac{\lambda^2}{2} \sum_{i=1}^n (\rv{X}_i)_{u}^2\right).
 \]
\end{lemma}

\begin{proof}
It holds
\[
\E_{\eps}\exp\left(\lambda \max_{\substack{m=1,\dots,k, \\ u=1,\dots, d}}\left|\sum_{i=1}^n \eps_{i,m} (\rv{X}_i)_{u}\right| \right) \leq \sum_{\substack{m=1,\dots,k, \\ u=1,\dots, d}}\E_{\eps}\exp\left(\lambda \left|\sum_{i=1}^n \eps_{i,m} (\rv{X}_i)_{u}\right| \right).
\]
For given $u$ and $m$:
    \begin{multline*}
        \E_{\eps}\exp\left(\lambda \left|\sum_{i=1}^n \eps_{i,m} (\rv{X}_i)_{u}\right| \right)  \underset{\textrm{\Cref{lem:GetRidOfAbsValues}}} \leq  2 \E_{\eps}\exp\left(\lambda \sum_{i=1}^n \eps_{i,m} (\rv{X}_i)_{u}\right) \\ = 2 \prod_{i=1}^n \frac{\exp\left(\lambda (\rv{X}_i)_{u}\right) + \exp\left(- \lambda (\rv{X}_i)_{u}\right)}{2} 
        \leq  2 \exp\left(\frac{\lambda^2}{2} \sum_{i=1}^n (\rv{X}_i)_{u}^2\right)
    \end{multline*}
using $\exp(x)+\exp(-x)\leq 2\exp(x^2/2)$ in the last inequality.
\end{proof}

\section{The cross-entropy loss is Lipschitz}\label{app:CrossEntropy}
\Cref{thm:GeneralizationBound} \emph{applies to the cross-entropy loss with $L=\sqrt{2}$}. To see this, first recall that with $C$ classes, the cross-entropy loss is defined as 
\[
\ell:(x,y)\in\R^{C}\times \{0,1\}^{C}\mapsto -\sum_{c=1}^\dout y_c \log\left(\frac{\exp(x_c)}{\sum_{d} \exp(x_d)}\right).
\]
Consider $y\in\{0,1\}^C$ with exactly one nonzero coordinate and an exponent $p\in[1,\infty]$ with conjugate exponent $p'$ ($1/p+1/p'=1$). For every $x,x'\in\R^C$:
\[
       \ell(x,y) - \ell(x',y) \leq 2^{1/p'} \|x-x'\|_p.
\]
Consider a class $c\in\{1,\dots,C\}$ and take $y\in\{0,1\}^C$ to be a one-hot encoding of $c$ (meaning that $y_{c'} = \mathbbm{1}_{c'=c}$). Consider an exponent $p\in[1,\infty]$ with conjugate exponent $p'$ ($1/p + 1/p'=1$). The function $f:x\mapsto \ell(x,y) =  -\sum_c y_c \log\left(\frac{\exp(x_c)}{\sum_{c'=1}^C \exp(x_{c'})}\right) = - \log\left(\frac{\exp(x_c)}{\sum_{c'=1}^C \exp(x_{c'})}\right)$ is continuously differentiable so that for every $x,x'\in\R^C$:
\[
f(x) - f(x') = \int_{0}^1 \inner{\nabla f(tx + (1-t)x')}{x-x'}dt \leq \sup_{t\in[0,1]} \|\nabla f(tx + (1-t)x')\|_p \|x-x'\|_{p'}.
\]
In order to differentiate $f$, let's start to differentiate $g(x) = \frac{\exp(x_c)}{\sum_{c'=1}^C \exp(x_{c'})}$. Denote $\partial_i$ the partial derivative with respect to coordinate $i$. For $i\neq c$:
\begin{align*}
    \partial_{c} g(x) & = \frac{\exp(x_c) \left(\sum_{c'} \exp(x_{c'})\right) - \exp(x_c) \left(\exp(x_c)\right)}{\left(\sum_{c'} \exp(x_{c'})\right)^2} \\
    & = g(x) \frac{\sum_{c'\neq c} \exp(x_{c'})}{\sum_{c'} \exp(x_{c'})}. \\
    \partial_{i} g(x) & = \frac{0 \left(\sum_{c'} \exp(x_{c'})\right) - \exp(x_c) \left(\exp(x_i)\right)}{\left(\sum_{c'} \exp(x_{c'})\right)^2} \\
    & = g(x) \frac{- \exp(x_i)}{\sum_{c'} \exp(x_{c'})}.
\end{align*}
Since $f(x)=(-\log\circ h)(x)$:
\begin{align*}
    \partial_i f(x) & = - \frac{\partial_i g(x)}{g(x)} \\
    & = \frac{1}{\sum_{c'=1}^C \exp(x_{c'})} \times \left\{\begin{array}{cc}
        - \sum_{{c'}\neq c} e^{x_{c'}} & \textrm{ if }i=c,\\
        e^{x_i} & \textrm{ otherwise}.
    \end{array}\right.
\end{align*}
Thus
\begin{align*}
    \|\nabla f(x)\|_p^p & = \sum_{i=1}^C |\partial_i f(x)|^p \\
    & = \frac{\left(\sum_{{c'}\neq c} \exp({x_{c'}})\right)^p + \sum_{c'\neq c} \exp(x_{c'})^p}{\left(\sum_{c'=1}^C \exp(x_{c'})\right)^p} \\
    & \leq 2 \frac{\left(\sum_{{c'}\neq c} \exp({x_{c'}})\right)^p}{\left(\sum_{c'=1}^C \exp(x_{c'})\right)^p} \\
    & \leq 2 \frac{\left(\sum_{{c'}\neq c} \exp({x_{c'}})\right)^p}{\left(\sum_{c'\neq c} \exp(x_{c'})\right)^p} \\
    & = 2.
\end{align*}
where we used in the first inequality that $\|v\|_p^p \leq \|v\|_1^p$ for any vector $v$. This shows that for every $x,x'\in\R^C$:
\[
    \ell(x,y) - \ell(x',y) \leq 2^{1/p} \|x-x'\|_{p'}.
\]

\section{The top-1 accuracy loss is not Lipschitz}\label{app:Top1Acc}

\Cref{thm:GeneralizationBound} \emph{does not apply to the top-1 accuracy loss} $\ell(\hat{y}, y)=\mathbbm{1}_{\argmax \hat{y} = \argmax y}$ as \Cref{hyp:LossLips} cannot be satisfied by $\ell$. Indeed, it is easy to construct situations where $\hat{y}_1 = R_{\param}(x_1)$ is arbitrarily close to $\hat{y}_2=R_{\param}(x_2)$ with $x_2$ correctly classified, while $x_1$ is not (just take $x_2$ on the boundary decision of the network and $x_1$ on the wrong side of the boundary), so that the left-hand side is equal to $1$ and the right-hand side is arbitrarily small. Thus, there is no finite $L>0$ that could satisfy \Cref{hyp:LossLips}.

\section{The margin-loss is Lipschitz}\label{app:MarginLoss}
For $\hat{y}\in\R^\dout$ and a one-hot encoding $y\in\R^\dout$ of the class $c$ of $x$ (meaning that $y_{c'}=\mathbbm{1}_{c'=c}$ for every $c'$), the margin $M(\hat{y},y)$ is defined by
\[
    M(\hat{y}, y) := [\hat{y}]_c - \max_{c'\neq c} [\hat{y}]_{c'}.
\]
For $\gamma>0$, recall that the $\gamma$-margin-loss is defined by
\begin{equation}\label{eq:defMarginLoss}
    \ell(\hat{y}, y) = \left\{\begin{array}{cc}
    0 & \textrm{if  }\gamma < M(\hat{y}, y), \\
    1 - \frac{M(\hat{y}, y)}{\gamma}  & \textrm{if  }0\leq M(\hat{y}, y)\leq \gamma, \\
    1 & \textrm{if  }M(\hat{y}, y)<0.
\end{array}\right.
\end{equation}

For any class $c$ and one-hot encoding $y$ of $c$, it is known that $\hat{y}\in\R^{\dout}\mapsto M(\hat{y}, y)$ is $2$-Lipschitz with respect to the $L^2$-norm on $\hat{y}$ \citep[Lemma A.3]{Bartlett17SpectralGeneralizationBound}. Moreover, the function 
\[
r\in\R\mapsto \left\{\begin{array}{cc}
    0 & \textrm{if  }r < -\gamma, \\
    1 +\frac{r}{\gamma}  & \textrm{if  }-\gamma\leq  r \leq 0, \\
    1 & \textrm{if  }r>0.
\end{array}\right.
\]
is $\frac{1}{\gamma}$-Lipschitz. By composition, this shows that $\hat{y}\in\R^\dout\mapsto \ell_\gamma(\hat{y},y)$ is $\frac{2}{\gamma}$-Lipschitz with respect to the $L^2$-norm.

\begin{proof}[Proof of \Cref{thm:MarginLoss}]
    Since the labels $\rv{Y}$ are one-hot encodings, we equivalently consider $\rv{Y}$ either in $\R^{\dout}$ or in $\{1,\dots,\dout\}$. It holds \citep[Lemma A.4]{Bartlett17SpectralGeneralizationBound}
\[
    \P\left(\argmax_{c} [R_\param(\rv{X})]_c \neq \rv{Y}\right) \leq \E\left(\ell_\gamma(R_\param(\rv{X}), \rv{Y})\right)
\]
for any $\gamma>0$ and associated $\gamma$-margin-loss $\ell_\gamma$. Thus, considering the generalization error for $\ell_\gamma$:
\begin{align*}
    \P\left(\argmax_{c} [R_\param(\rv{X})]_c \neq \rv{Y}\right) & \leq \underbrace{\frac{1}{n}\sum_{i=1}^n \ell_\gamma\left(R_{\hatparam(\rv{Z})}(\rv{X}_i), \rv{Y}_i\right)}_{= \textrm{ training error of $\hatparam(\rv{Z})$}} + \E_{\rv{Z}} \textrm{ $\ell_\gamma$-generalization error of $\hatparam(\rv{Z})$}.
\end{align*}
By definition of $\ell_\gamma$, the training error of $\hatparam(\rv{Z})$ is at most $\frac{1}{n}\sum_{i=1}^n \mathbbm{1}_{[R_{\hatparam(\rv{Z})}(\rv{X}_i)]_{\rv{Y}_i} \leq \gamma + \max_{c \neq \rv{Y}_i}[R_{\hatparam(\rv{Z})}(\rv{X}_i)]_{c}}$. Moreover, \Cref{thm:GeneralizationBound} can be used to bound the generalization error associated with $\ell_\gamma$ with $L=2/\gamma$. This proves the claim.
\end{proof}

\section{Details on the experiments of \Cref{sec:exp}}\label{app:expes}
\textbf{Details for \Cref{tab:ValueBoundResNets}.} All the experiments are done on ImageNet-1k using $99\%$ of the 1,281,167 images of the training set for training, the other $1\%$ is used for validation. Thus, $n=1268355 = \floor{0.99 \times 1281167}$ in our experiments, $\din=224\times 224 \times 3=150528$, $\dout=1000$. We also estimated $B=2.640000104904175$ by taking the maximum of the $L^\infty$ norms of the training images normalized for inference\footnote{The constant $\sigma$ in \Cref{thm:GeneralizationBound} corresponds to data $\rv{Z}_i$ drawn from the distribution for which we want to evaluate the test error. This is then the data normalized for inference. Thus, the training loss appearing in \Cref{thm:GeneralizationBound} is also evaluated on the training data $\rv{Z}_i$ normalized for inference. In the experiments, we ignore this fact and still evaluate the training loss on the data augmented for training. Moreover, note that it is not possible to recover the training images augmented for training from the images normalized for inference, because cropping is done at random for training. Thus, the real life estimator is not a function of the images $\rv{Z}_i$ normalized for inference, and \Cref{thm:GeneralizationBound} does not apply stricto sensu. This fact is ignored here.}. The PyTorch code for normalization at inference is standard:
\begin{lstlisting}
inference_normalization = transforms.Compose([
        transforms.Resize(256),
        transforms.CenterCrop(224),
        transforms.ToTensor(),
        transforms.Normalize(mean=[0.485, 0.456, 0.406], std=[0.229, 0.224,  0.225]),
    ])
\end{lstlisting}

We consider ResNets. They have a single max-pooling layer of kernel size $3\times 3$ so that $K=9$. The depth is $D=3+\textrm{\# basic blocks}\times \textrm{\# conv per basic block}$, where $3$ accounts for the conv1 layer, the average-pooling layer, the fc layer, and the rest accounts for all the convolutional layers in the basic blocks. Note that the average-pooling layer can be incorporated into the next fc layer as it only contains identity neurons, see the discussion after \Cref{thm:GeneralizationBound}, so we can actually consider $D=2+\textrm{\# basic blocks}\times \textrm{\# conv per basic block}$. \Cref{tab:ResNetsBlocks} details the relevant values related to basic blocks. 
\begin{table}[htbp]
    \centering
    \caption{Number of basic blocks, of convolutional layer per basic blocks and associated $D$ for ResNets \citep[Table 1]{He16ResNets}.}
    \label{tab:ResNetsBlocks}
    \begin{tabular}{|c|c|c|c|c|c|}
        \hline
        ResNet & 18 & 34 & 50 & 101 & 152 \\ 
        \hline
        \# basic blocks & 8 & \multicolumn{2}{c|}{16} & 33 & 50 \\
        \hline
        \# conv per basic block & \multicolumn{2}{c|}{2} & \multicolumn{3}{c|}{3} \\
        \hline
        $D$ & 18 & 34 & 50 & 101 & 152 \\
        \hline
    \end{tabular}
\end{table}

\textbf{Pretrained ResNets.} The PyTorch pretrained weights that have been selected are the ones with the best performance: \lstinline{ResNetX_Weights.IMAGENET1K_V1} for ResNets 18 and 34, and \lstinline{ResNetX_Weights.IMAGENET1K_V2} otherwise. 

\textbf{Choice of $\gamma>0$ for \Cref{thm:MarginLoss}.} In \Cref{eq:GeneBoundMargin}, note that there is a trade-off when choosing $\gamma>0$. Indeed, the first term of the right-hand side is non-decreasing with $\gamma$ while the second one is non-increasing. The first term is simply the proportion of datapoints that are not correctly classified with a margin at least equal to $\gamma$. Defining the margin of input $i$ on parameters $\param$ to be $R_{\param}(\rv{X}_i)_{\rv{Y}_i} - \argmax_{c\neq \rv{Y}_i} R_\param(\rv{X}_i)_{c}$, this means that the first term is (approximately) equal to $q$ if $\gamma=\gamma(q)$ is the $q$-quantile of the distribution of the margins over the training set. 

Note that since the second term in \Cref{eq:GeneBoundMargin} is of order $1/\sqrt{n}$, it would be desirable to choose the $1/\sqrt{n}$-quantile (up to a constant) for $\gamma$. However, this is not possible in practice as soon as the training top-1 accuracy is too large compared to $1/\sqrt{n}$ (eg. on ImageNet). Indeed, if the training top-1 error is equal to $e\in[0,1]$, then at least a proportion $e$ of the data margins should be negative\footnote{A data margin is negative if and only if it is misclassified.} so that any $q$-quantile with $q<e$ is negative and cannot be considered for \Cref{thm:MarginLoss}

The distribution of the margins on the training set of ImageNet can be found in \Cref{fig:MarginsPretrained}. The maximum training margin is roughly of size $30$, which is insufficient to compensate the size of the $L^1$ path-norm of pretrained ResNets reported in \Cref{tab:IncreasingDepth}. For $\gamma>30$, the first term of the right-hand side of \Cref{thm:MarginLoss} is greater than one, so that the bound is not informative. This shows that there is no possible choice for $\gamma>0$ that makes the bound informative on these pretrained ResNets. \Cref{tab:QuantilesMargin} reports a quantile for these pretrained ResNets.

\begin{table}[htbp]
    \centering
    \caption{The $q$-quantile $\gamma(q)$ for $q=\frac{1}{3}e + \frac{2}{3}$, with $e$ being the top-1 error, on ImageNet, of pretrained ResNets available on PyTorch.}
    \label{tab:QuantilesMargin}
    \begin{tabular}{|c|c|c|c|c|c|}
         \hline
         ResNet & 18 & 34 & 50 & 101 & 152  \\
         \hline
         $\gamma(q)$ & $5.0$ & $5.6$ & $4.2$ & $5.6$ & $5.8$ \\ 
         \hline
    \end{tabular}
\end{table}

\begin{figure}
\centering
\begin{subfigure}{0.49\textwidth}
\centering
\includegraphics[width = \textwidth]{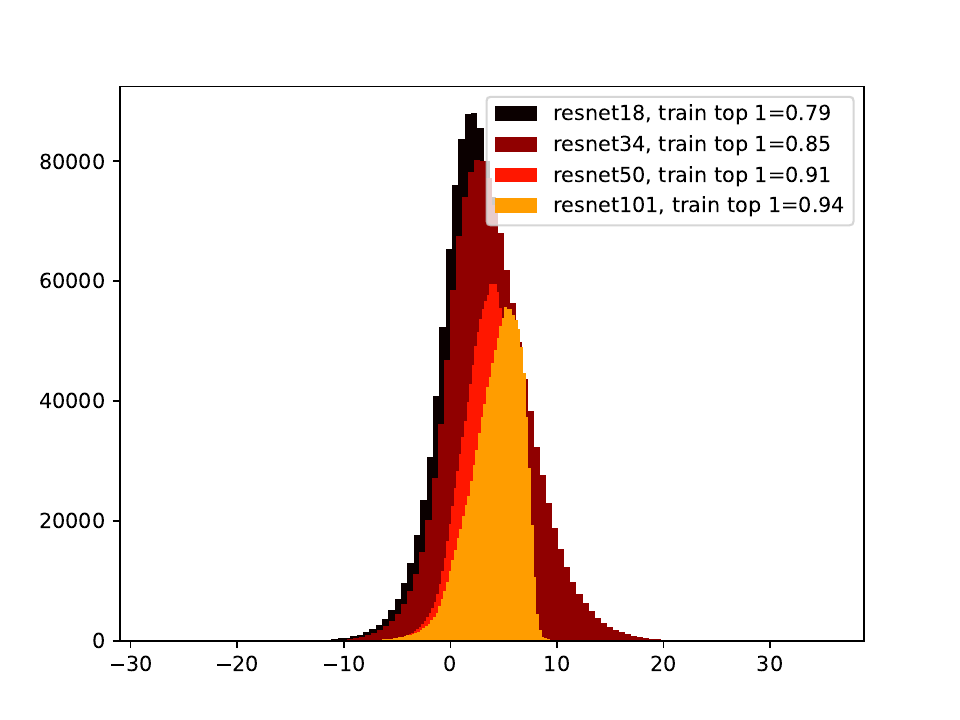}
\end{subfigure}
\begin{subfigure}{0.49\textwidth}
\centering
\includegraphics[width = \textwidth]{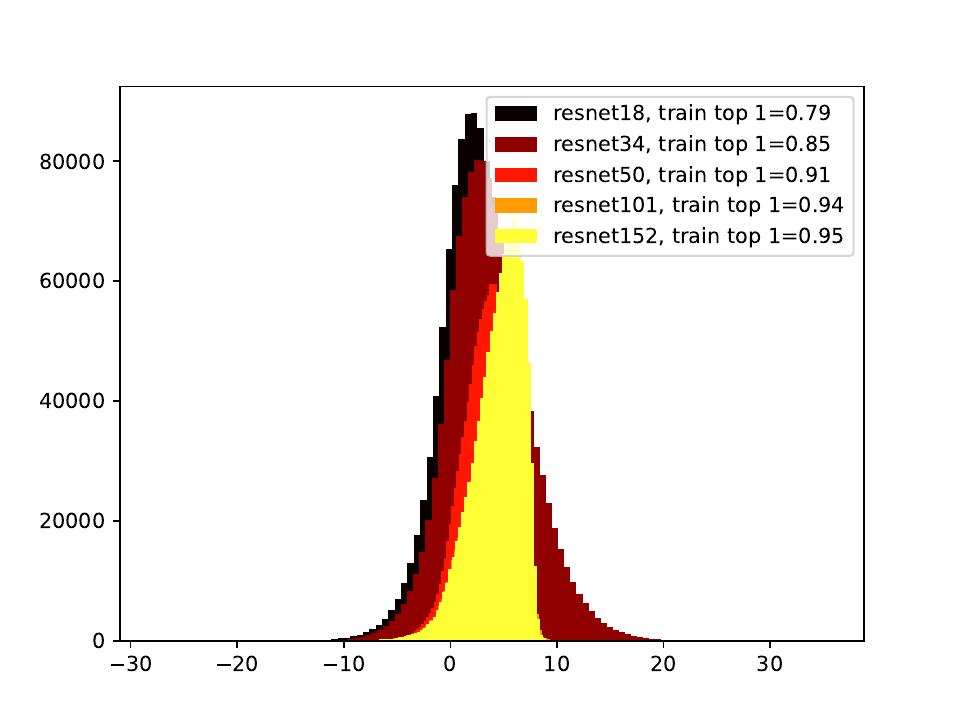}
\end{subfigure}
\caption{Distribution of the margins on the training set of ImageNet, with the pretrained ResNets available on PyTorch.}
\label{fig:MarginsPretrained}
\end{figure}

\textbf{Details for sparse networks.} ResNet18 is trained on $99\%$ of ImageNet with a single GPU using SGD for $90$ epochs, learning rate $0.1$, weight-decay $0.0001$, batch size $1024$, and a multi-step scheduler where the learning rate is divided by $10$ at epochs $30$, $60$ and $80$. The epoch out of the $90$ ones with maximum validation top-1 accuracy is considered as the final epoch. Pruning is done iteratively accordingly to \citet{Frankle21MissingTheMark}. We prune $20\%$ of the remaining weights of each convolutional layer, and $10\%$ of the final fully connected layer, at each pruning iteration, save the mask and rewind the weights to their values after the first $5$ epochs of the dense network, and train for $85$ remaining epochs, before pruning again etc. Results for a single run are shown in \Cref{fig:SparseNetworks}.

\textbf{Details for increasing the train size.} Instead of training on $99\%$ of ImageNet ($n=1268355$), we trained a ResNet18 on $n/2^k$ samples drawn at random, for $1\leq k\leq 5$. For each given $k$, the results are averaged over 3 seeds. The hyperparameters are the same as for sparse networks (except that we do not perform any pruning here): 90 epochs etc. Results are in \Cref{fig:IncreasingTrainSize}.
\begin{figure}
\centering
\begin{subfigure}{0.49\textwidth}
\centering
\includegraphics[width = \textwidth]{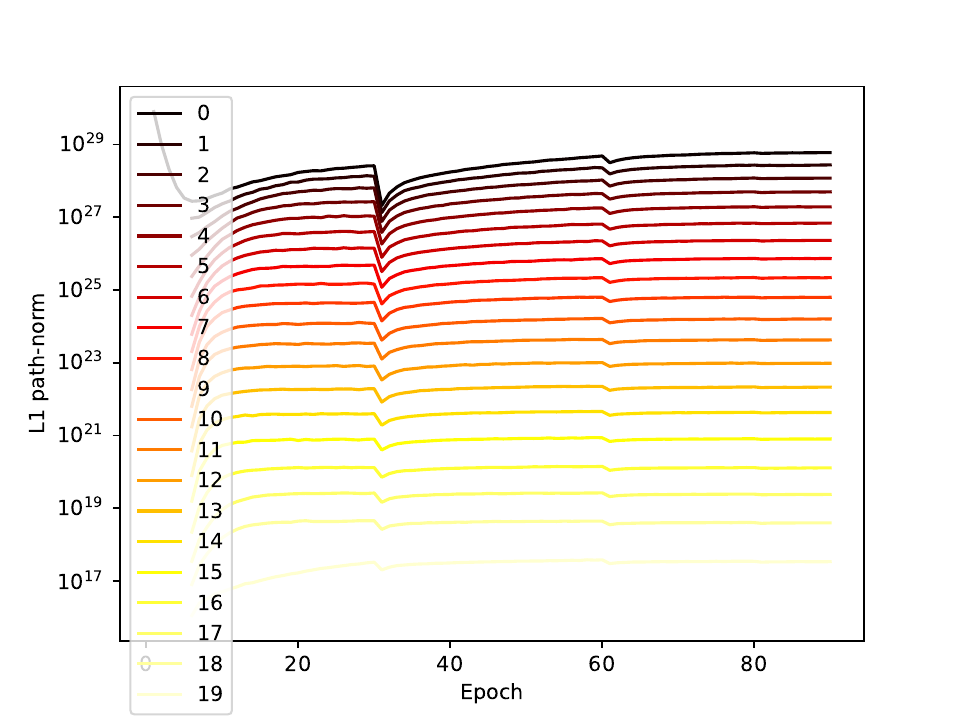}
\end{subfigure}
\begin{subfigure}{0.49\textwidth}
\centering
\includegraphics[width = \textwidth]{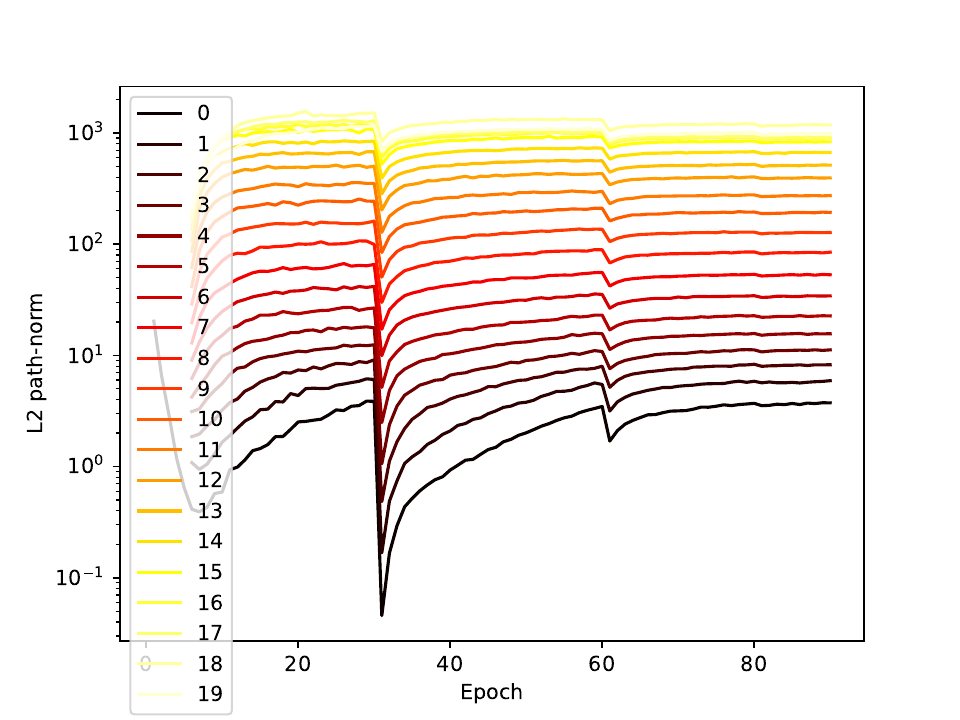}
\end{subfigure}
\begin{subfigure}{0.49\textwidth}
\centering
\includegraphics[width = \textwidth]{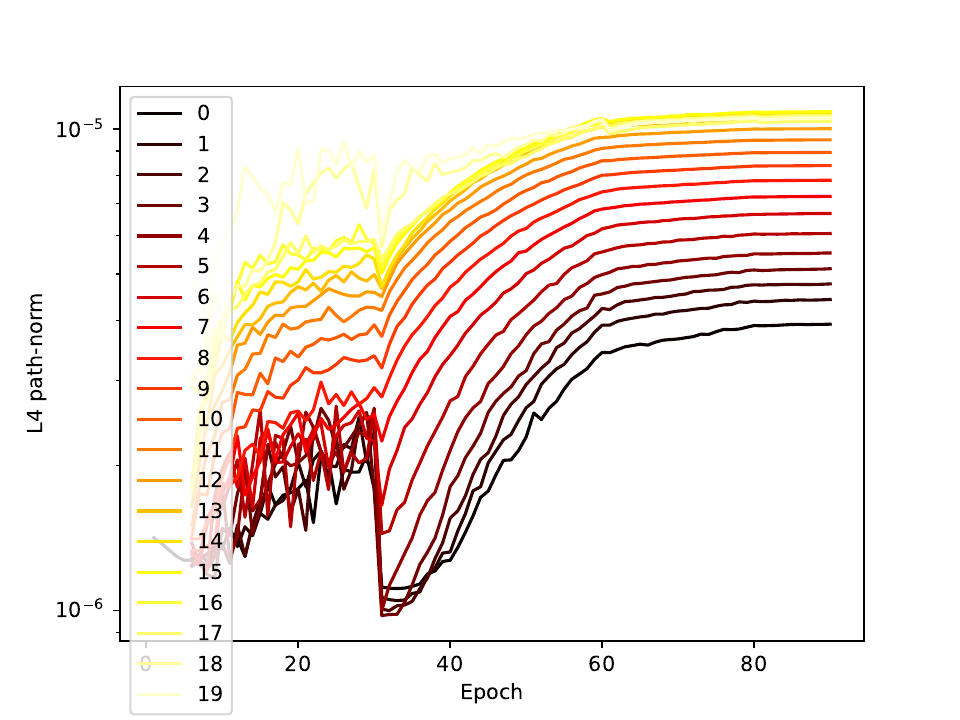}
\end{subfigure}
\begin{subfigure}{0.49\textwidth}
\centering
\includegraphics[width = \textwidth]{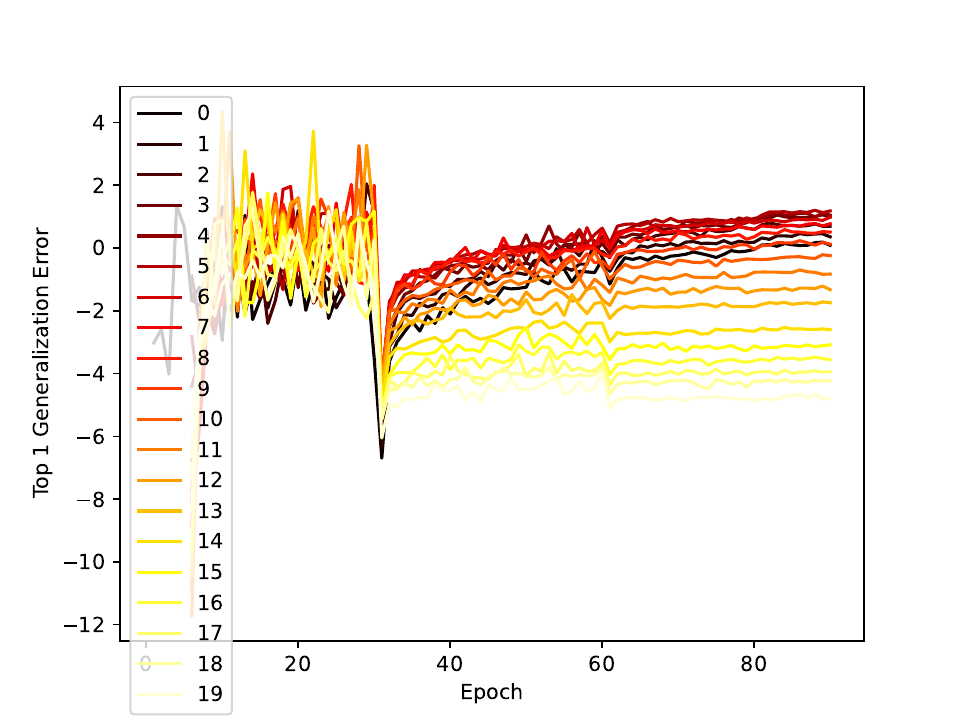}
\end{subfigure}
\begin{subfigure}{0.49\textwidth}
\centering
\includegraphics[width = \textwidth]{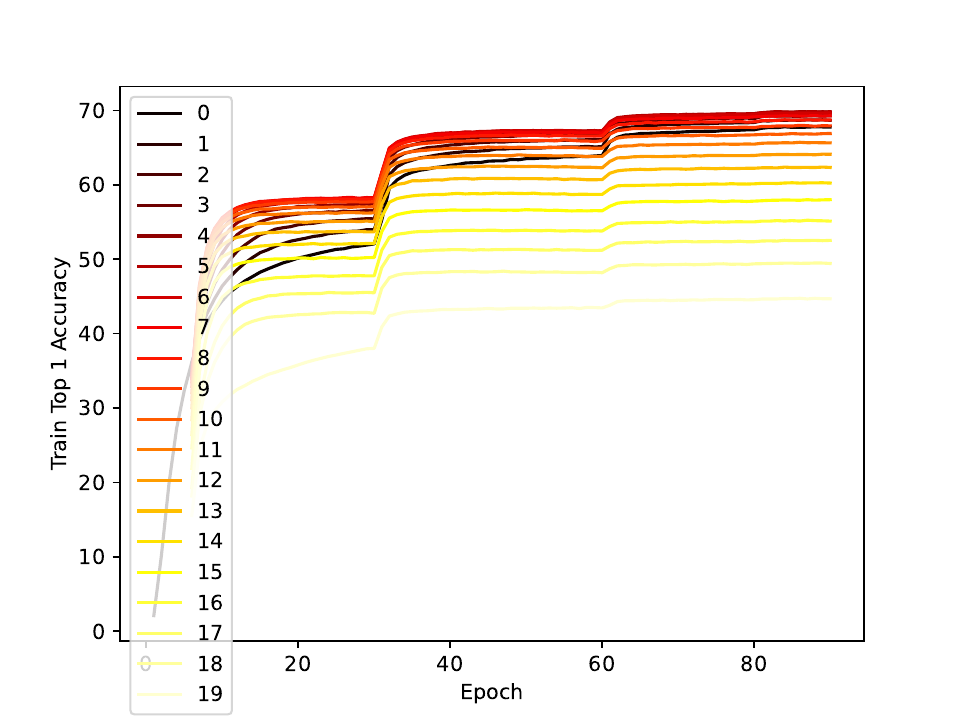}
\end{subfigure}
\begin{subfigure}{0.49\textwidth}
\centering
\includegraphics[width = \textwidth]{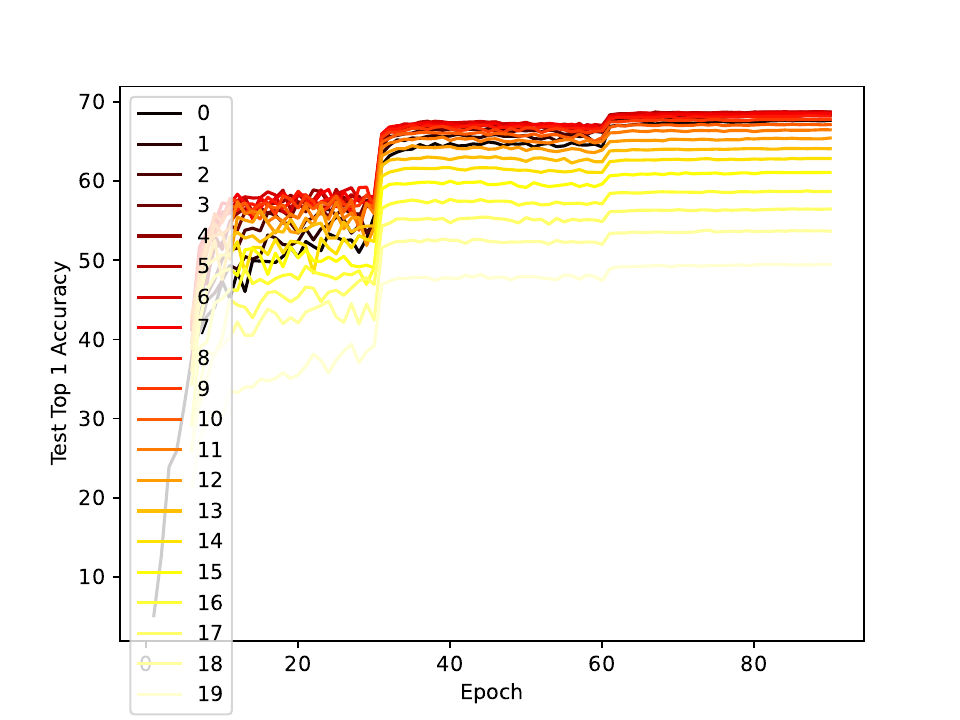}
\end{subfigure}
\caption{$L^q$ path-norm ($q=1,2,4$), test top-1 accuracy, training top-1 accuracy, and the top-1 generalization error (difference between test top-1 and train top-1) during the training of a ResNet18 on ImageNet. The pruning iteration is indicated in legend, with $0$ corresponding to the dense network. The color also indicates the degree of sparsity: from dense (black) to extremely sparse (yellow).}
\label{fig:SparseNetworks}
\end{figure}
\begin{figure}
\centering
\begin{subfigure}{0.88\textwidth}
\centering
\includegraphics[width = \textwidth]{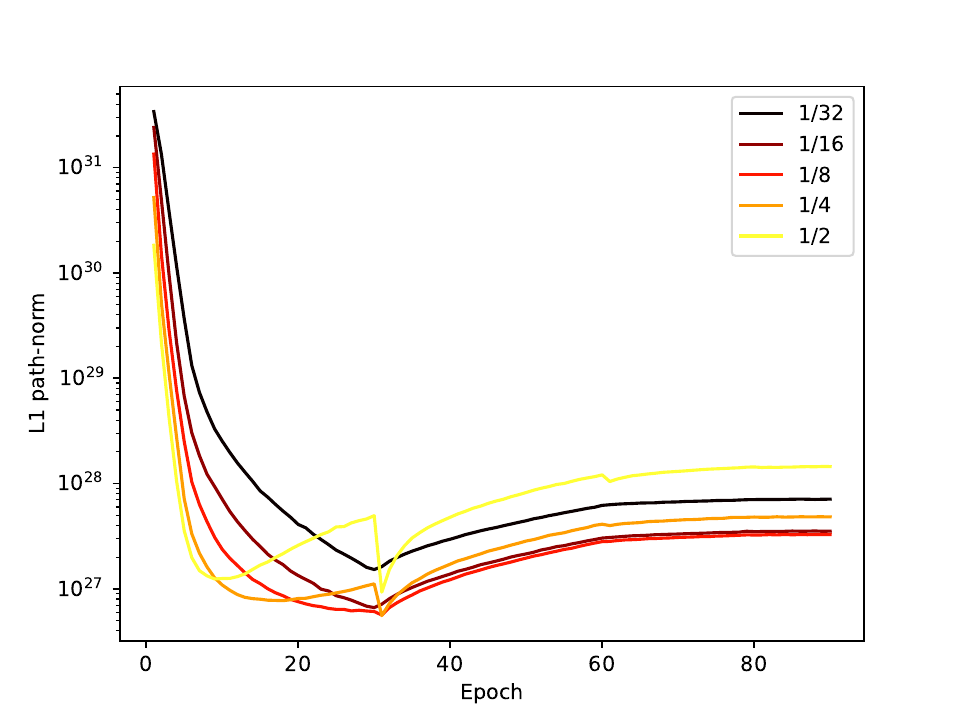}
\end{subfigure}
\begin{subfigure}{0.49\textwidth}
\centering
\includegraphics[width = \textwidth]{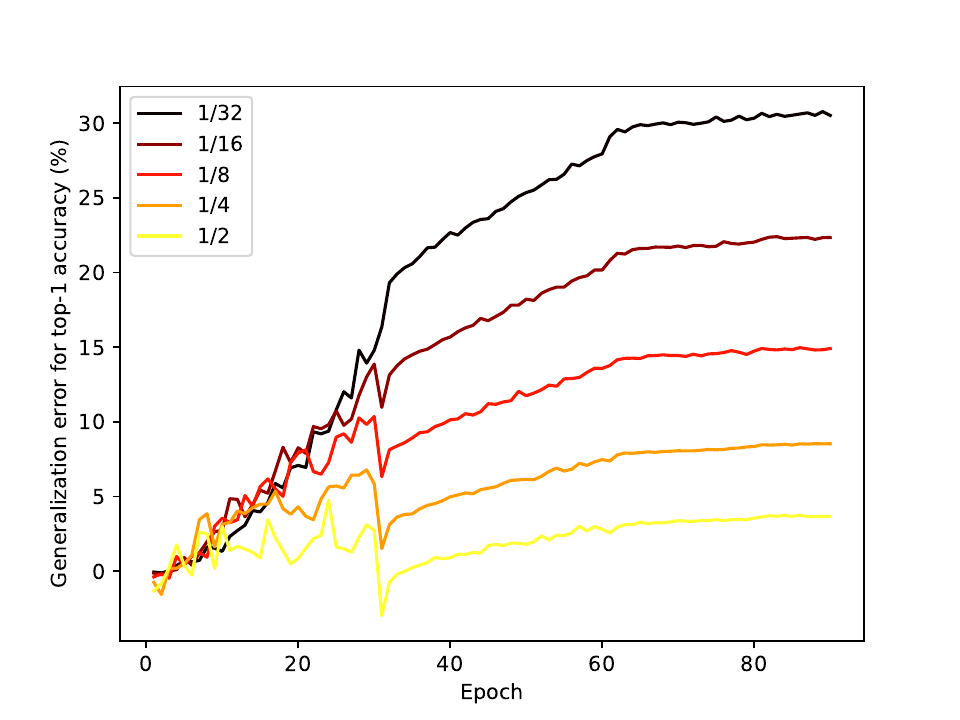}
\end{subfigure}
\begin{subfigure}{0.49\textwidth}
\centering
\includegraphics[width = \textwidth]{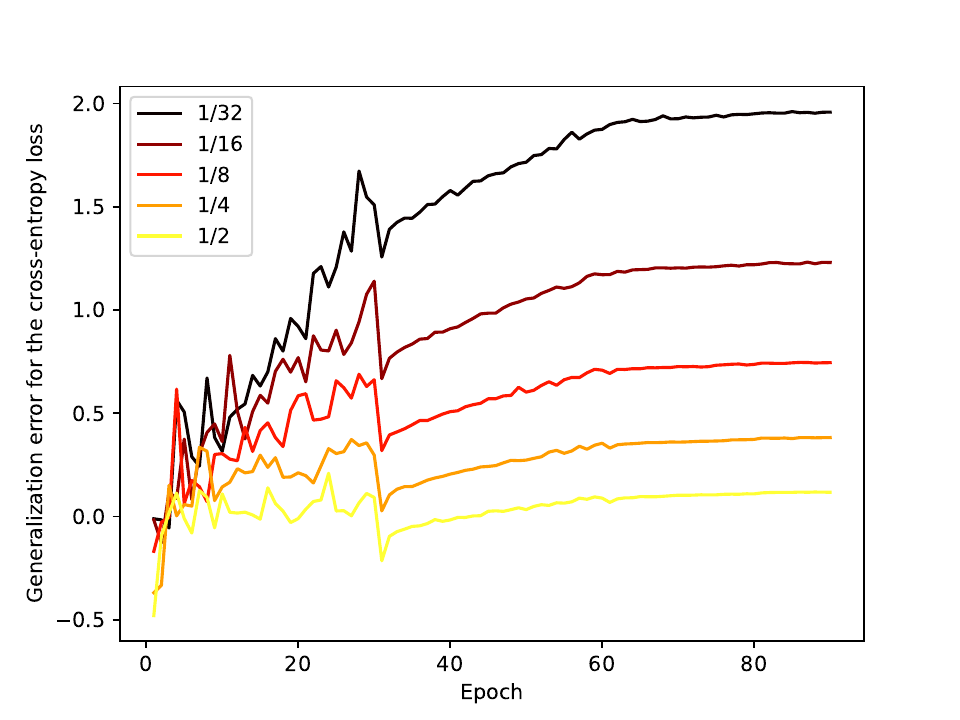}
\end{subfigure}
\caption{$L^1$ path-norm, and empirical generalization errors for both the top-1 accuracy and the cross-entropy during the training of a ResNet18 on a subset of the training images of ImageNet. The legend indicates the size of the subset considered, \eg $1/m$ corresponds to $1/m$ of $99\%$ of ImageNet, leaving the other $1\%$ out for validation. The color also indicates the size of the subset: from small (black) to large (yellow).}
\label{fig:IncreasingTrainSize}
\end{figure}
\section{Main related works}\label{sec:RelatedWorks}
Given the extent of the literature on generalization bounds, we apologize in advance for papers the reader may find missing below. 

\textbf{Previous definitions of the path-lifting and the path-activations} In the work of \citet[Section 5.1]{Kawaguchi17GeneralizationInDL} the path-lifting and the path-activations are evoked in the case of a ReLU DAG with max-pooling neurons {\em and no biases}, but with no explicit definitions. \Cref{def:PhiActivation} gives a formal definition for these objects, and extend it to the case where there are biases, which requires extending the path-lifting to paths starting from neurons that are not input neurons. Moreover, \Cref{def:PhiActivation} extends it to arbitrary $k$-max-pooling neurons (classical max-pooling neurons correspond to $k=1$). 

Note also that the formula \Cref{eq:ForwardWithPhiAppendix} is stated in the specific case of \citet[Section 5.1]{Kawaguchi17GeneralizationInDL} (as an explicit sum rather than an inner product), without proof since the objects are not explicitly defined in \citet{Kawaguchi17GeneralizationInDL}.

A formal definition of the path-lifting is given in the specific case of layered fully-connected ReLU neural networks with biases in the work of \citet[Definition 6]{Stock22Embedding}. Moreover, it is proved that \Cref{eq:ForwardWithPhiAppendix} holds {\em in this specific case} in \citet[Corollary 3]{Stock22Embedding}. \Cref{def:PhiActivation} and \Cref{eq:ForwardWithPhiAppendix} generalize the latter to an arbitrary DAG with $*$-max-pooling or identity neurons (allowing in particular for skip connections, max-pooling and average-pooling).

The rest of the works we are aware of only define and consider the norm of the path-lifting, but not the lifting itself. The most general setting being the one of \citet{Neyshabur15NormBasedControls} with a general DAG, {\em but without max or identity neurons, nor biases}. Not defining the path-lifting and the path-activations makes notations arguably heavier since \Cref{eq:ForwardWithPhiAppendix} is then always written with an explicit sum over all paths, with explicit product of weights along each path, and so on.

\textbf{Previous generalization bounds based on path-norm} See \Cref{tab:GeneralizationBounds} for a comparison. \Cref{app:MoreRelatedWorks} tackles some other bounds that do not appear in \Cref{tab:GeneralizationBounds}.

\textbf{Empirical evaluation of path-norm} The formula given in \Cref{thm:ComputePathNorm} is the first one to fully encompass modern networks with biases, average/$*$-max-pooling, and skip connections, such as ResNets. An equivalent formula is stated for  \emph{layered fully-connected} ReLU networks {\em without biases} (and {\em no pooling/skip connections}) in \citet[Appendix C.6.5]{GKDziugaite20SearchRobustMeasuresGeneralization} and \citet[Equations (43) and (44)]{Jiang20FantasticGeneralizationMeasures} but without proof (and only for $L^2$-path-norm instead of general mixed $L^{q,r}$ path-norm). Actually, {\em this equivalent formula turns out to be false when there are $*$-max-pooling neurons} as one must replace $*$-max-pooling neurons with identity ones, see \Cref{thm:ComputePathNorm}. Care must also be taken with average-pooling neurons that must be rescaled by considering them as identity neurons.

We could not find reported numerical values of the path-norm except for toy examples \citep{PhDGKDZ18, PhDFurusho20Lipschitz, Zheng19BasisPathNorm}. Details are in \Cref{app:MoreRelatedWorks}.

\Cref{app:MoreRelatedWorks} also discusses 1) inherent limitations of \Cref{thm:GeneralizationBound} which are common to many generalization bounds, and 2) the applicability of \Cref{thm:GeneralizationBound} compared to PAC-Bayes bounds.

\section{More Related Works}\label{app:MoreRelatedWorks}

\textbf{More generalization bounds using path-norm} In \citet{E22Barron} is established an additional bound to the one appearing in \Cref{tab:GeneralizationBounds}, for a class of functions and a complexity measure that are related to the \emph{infinite-depth} limits of residual networks and the path-norm. However, it is unclear how this result implies anything for actual neural networks with the actual path-norm \footnote{\citet{E22Barron} starts from the fact that the infinite-depth limits of residual networks can be  characterized with partial differential equations. Thus \citet{E22Barron} establishes a bound for functions characterized by similar, but different, partial differential equations, using what seems to be an analogue of path-norm for these new functions. However, even if the characterizations of these functions are closed, as it is said in \citet{E22Barron}, "it is unclear how the two spaces are related".}. 

The bound in \citet{Zheng19BasisPathNorm} only holds for layered fully-connected ReLU neural networks ({\em no max or identity neurons}) with {\em no biases} and {\em it grows exponentially} with the depth of the network. It is not included in \Cref{tab:GeneralizationBounds} because it requires an additional assumption: the coordinates of the path-lifting must not only be bounded from above, but also \emph{from below}. The reason for this assumption is not discussed in \citet{Zheng19BasisPathNorm}, and it is unclear whether this is at all desirable, since a small path-norm can only be better for generalization in light of \Cref{thm:GeneralizationBound}.

Theorem 4.3 in \citet{Golowich18GeneralizationBound} is a bound that holds for layered fully-connected ReLU neural networks {\em with no biases (no max and no identity neurons)} and it {\em depends on the product of operator norms of the layers}. It has the merit of having no dependence on the size of the architecture (depth, width, number of neurons etc.). However, it  requires an additional assumption: each layer must have an operator norm bounded from below, so that it only applies to a restricted set of layered fully-connected networks. Moreover, it is unclear whether such an assumption is desirable: there are networks with arbitrary small operator norms that realize the zero function, and the latter has a generalization error equal to zero.

Theorem 8 in \citet{Kawaguchi17GeneralizationInDL} gives a generalization bound for {\em scalar-valued} ($\dout=1$) models with an output of the form $\inner{\Phi(\param)}{\activation(\param',x)x}$ for some specific parameters $\param,\param'$ that have no reason to be equal. This is orthogonal to the case of neural networks where one must have $\param=\param'$, and it is therefore not included in \Cref{tab:GeneralizationBounds}. Theorem 5 in \citet{Kawaguchi17GeneralizationInDL} can be seen as a possible first step to derive a bound based on path-norm in the specific case of the mean squared error loss. However, as discussed in more details below, Theorem 5 in \citet{Kawaguchi17GeneralizationInDL} is a rewriting of the generalization error with several terms that are as complex to bound as the original generalization error, resulting in a bound being as hard as the generalization error to evaluate/estimate.

\textbf{More details about Theorem 5 in \citet{Kawaguchi17GeneralizationInDL}} We start by re-deriving Theorem 5 in \citet{Kawaguchi17GeneralizationInDL}. In the specific case of mean squared error, using that
\[
    \|R_{\param}(x)-y\|_2^2 = \|R_{\param}(x)\|_2^2 + \|y\|_2^2 - 2\inner{R_\param(x)}{y},
\]
it is possible to rewrite the generalization error as follows:
\begin{align*}
    \textrm{generalization error of $\hatparam(\rv{Z})$} = & \E\left(\|R_{\hatparam(\rv{Z})}(\tilde{\rv{X}}) - \tilde{\rv{Y}}\|_2^2 | \rv{Z}\right) - \frac{1}{n}\sum_{i=1}^n \|R_{\hatparam(\rv{Z})}(\rv{X}_i) - \rv{Y}_i\|_2^2\\
    = & \E\left(\|R_{\hatparam(\rv{Z})}(\tilde{\rv{X}})\|_2^2 | \rv{Z}\right) - \frac{1}{n} \sum_{i=1}^n \|R_{\hatparam(\rv{Z})}(\rv{X}_i)\|_2^2\\
    & + \E\left(\|\tilde{\rv{Y}}\|_2^2 | \rv{Z}\right) - \sum_{i=1}^n \|\rv{Y}_i\|_2^2\\
    & - 2 \E\left(\inner{R_{\hatparam(\rv{Z})}(\tilde{\rv{X}})}{\tilde{\rv{Y}}} | \rv{Z}\right) - \frac{1}{n}\sum_{i=1}^n \inner{R_{\hatparam(\rv{Z})}(\rv{X}_i)}{\rv{Y}_i}.
\end{align*}
It is then possible to make the $L^2$ path-norm appear. For instance, for one-dimensional output networks, it can be proven (see \Cref{lem:ForwardWithPhi}) that $R_\param(x)=\inner{\Phi(\param)}{z(x,\param)}$ with $\Phi(\param)$ the path-lifting of parameters $\param$, and $z$ that typically depends on the path-activations and the input, so that the first term above can be rewritten
\begin{align*}
    \Phi(\hatparam(\rv{Z}))^T \left(\E\left(z(\tilde{\rv{X}}, \hatparam(\rv{Z}))z(\tilde{\rv{X}}, \hatparam(\rv{Z}))^T | \rv{Z}\right) - \frac{1}{n} \sum_{i=1}^n z(\rv{X}_i, \hatparam(\rv{Z}))z(\rv{X}_i, \hatparam(\rv{Z}))^T\right)\Phi(\hatparam(\rv{Z})).
\end{align*}
Let us call a "generalization error like quantity" any term of the form
\[
    \E \left(f_{\hatparam(\rv{Z})}(\tilde{\rv{Z}}) | \rv{Z}\right) - \frac{1}{n}\sum_{i=1}^n f_{\hatparam(\rv{Z})}(\rv{Z}_i),
\]
that is, any term that can be represented as a difference between the estimator learned from training data $\rv{Z}$ evaluated on test data $\tilde{\rv{Z}}$, and the evaluation on the training data. We see that the derivation above replaces the classical generalization error with two others quantities similar in definition to the generalization error. This derivation, which is specific to mean squared error, leads to Theorem 5 in \citet{Kawaguchi17GeneralizationInDL}. Very importantly, note that this derivation trades a single quantity similar to generalization error for two new such quantities. There is no discussion in \citet{Kawaguchi17GeneralizationInDL} on how to bound these two new terms. Without any further new idea, there is no other way than the ones developed in the literature so far: reduce the problem to bounding a Rademacher complexity (as it is done in \Cref{thm:GeneralizationBound}), or use the PAC-Bayes framework, and so on.

\textbf{More on numerical evaluation of path-norm}
In \citet[Section 2.9.1]{PhDGKDZ18} is reported numerical evaluations after 5 epochs of SGD on a one hidden layer network trained on a binary variant of MNIST. \citet[Figure 9 and Section 3.3.1]{PhDFurusho20Lipschitz} deals with 1d regression with 5 layers and 100 width. Experiments in \citet{Zheng19BasisPathNorm} are on MNIST. Note that it is not clear whether the path-norm used in \citet{Zheng19BasisPathNorm} corresponds to the one defined in \Cref{def:PhiActivation}. Indeed, the references given in \citet{Zheng19BasisPathNorm} for the definition of the path-norm are both \citet{Neyshabur15NormBasedControls} and \citet{Neyshabur17ExploringGeneralization}, but these two papers have two different definitions of the path-norm. The one in \citet{Neyshabur15NormBasedControls} corresponds to the norm of the path-lifting as defined in \Cref{def:PhiActivation} (but in simpler settings: no pooling etc.), while the one in \citet{Neyshabur17ExploringGeneralization} corresponds to the latter divided by the margin of the estimator.

For completeness, let us also mention that it is reported in \citet{GKDziugaite20SearchRobustMeasuresGeneralization, Jiang20FantasticGeneralizationMeasures} whether the path-norm correlates with the empirical generalization error or not, but there is no report of numerical values of the path-norm. In \citet{Neyshabur17ExploringGeneralization} are reported the quotient of path-norms with margins, but not the path-norms alone.

\textbf{Inherent limitations of uniform convergence bounds} \Cref{thm:GeneralizationBound} has some inherent limitations due to its nature. It is \emph{data-dependent} as it depends on the input distribution. However, it does not depend on the label distribution, making it uninformative as soon as $\Param$ is so much expressive that it can fit random labels. Networks that can fit random labels have already been found empirically \citep{Zhang21RethinkingGeneralization}, and it is open whether this stays true with a constraint on the path-norm.

\Cref{thm:GeneralizationBound} is \emph{based on a uniform convergence bound}\footnote{A uniform convergence bound on a model class $\mF$ is a bound on $\E_{\rv{Z}} \sup_{f(Z)\in\mF} \textrm{ generalization error }f(Z)$. This worst-case type of bound can lead to potential limitations when $\mF$ is too expressive.} as any other bound also based on a control of a Rademacher complexity. \citet{Nagarajan19UniformCvgceUnable} empirically argue that even the tightest uniform convergence bound holding with high probability must be loose on some synthetic datasets. If this was confirmed theoretically, this would still allow uniform bounds to be tight when considering other datasets than the one in \citet{Nagarajan19UniformCvgceUnable}, such as real-world datasets, or when the estimator considered in \citet{Nagarajan19UniformCvgceUnable} is not in $\Param$ (for instance because of constraints on the slopes via the path-norm).

Finally, \Cref{thm:GeneralizationBound} can provide theoretical guarantees on the generalization error of the output of a learning algorithm, but only \emph{a posteriori}, after training. In order to have a priori guarantees, one should have to derive a priori knowledge on the path-norm at the end of the learning algorithm.

\textbf{Comparison to PAC-Bayes bounds} Another interesting direction to understand generalization of neural networks is the PAC-Bayes framework \citep{Guedj19SurveyPAC,Alquier21SurveyPAC}. Unfortunately, PAC-Bayes bounds cannot be exactly computed on networks that are trained in a usual way. Indeed, these bounds typically involve a KL-divergence, or something related, for which there is no closed form except for very specific distributions (iid Gaussian/Cauchy weights...) that do not correspond to the distributions of the weights after a usual training\footnote{The randomness of the weights after training comes from the random initialization and the randomness in the algorithm (\eg random batch in SGD).}\footnote{For instance, independence is not empirically observed, see \citet[Section 5.2]{Frankle20EarlyPhaseTraining}}. We are aware of two research directions that try to get over this issue. The first way is to change the learning algorithm by enforcing the weights to be iid normally distributed, and then optimize the parameters of these normal distributions, see for instance the pioneer work of \citet{GKDziugaite17ComputingNonVacuousPACBayesBounds}. The merit of this new learning algorithm is that it has explicitly been designed with the goal of having a small generalization error. Practical performance are worse than with usual training, but this leads to networks with an associated non-vacuous generalization bound. To the best of our knowledge, this is the only way to get a non-vacuous bound\footnote{Except, of course, for methods that are based on the evaluation of the performance on held-out data.}, and unfortunately, this does not apply to usual training. The second way to get over the intractable evaluation of the KL-divergence is to 1) try to approximate the bound within reasonable time, and 2) try to quantify the error made with the approximation \citep{VallePerez20SurveyGeneralization}. Unfortunately, to the best of our knowledge, approximation is often based on a distribution assumption of the weights that is not met in practice (\eg iid Gaussian weights), approximation is costly, and the error is unclear when applied to networks trained usually. For instance, the bound in \citet[Section 5]{VallePerez20SurveyGeneralization} 1) requires at least $O(n^2)$ operations to be evaluated, with $n$ being the number of training examples, thus being prohibitive for large $n$ \citep[Section 7]{VallePerez20SurveyGeneralization}, and 2) it is unclear what error is being made when using a Gaussian process as an approximation of the neural network learned by SGD.
\end{document}